%% file: sn-article.tex
\documentclass[pdflatex,sn-nature]{sn-jnl}
\theoremstyle{thmstyleone}%
\newtheorem{theorem}{Theorem}
\newtheorem{proposition}[theorem]{Proposition}%
\newtheorem{assumption}[theorem]{Assumption}%
\newtheorem{lemma}[theorem]{Lemma}%

\theoremstyle{thmstyletwo}%

\theoremstyle{thmstylethree}%

\input{setup/preamble}

\input{setup/math}

\begin{document}

\input{General/front_page}



\input{General/abstract}

\input{General/keywords}



\maketitle

\section{Introduction}
\input{General/Introduction.tex}


\section{Background and related work}
\input{Sections/background}

\section{GEORCE-FM}
\input{Sections/georce_fm}

\section{Adaptive GEORCE-FM}
\input{Sections/adaptive}

\section{Extension to Finsler manifolds}
\input{Sections/extensions}

\section{Experiments}
\rowcolors{2}{gray!10}{white}
\input{Sections/experiments}

\section{Conclusion}
\input{General/conclusion}

\backmatter

\bmhead{Supplementary information}

The supplementary information in the paper consists of the appendix for proofs, details on experiments and manifolds as well as additional experiments.

\bmhead{Acknowledgements}
\input{General/appendix/ackowledgement}

\section*{Declarations}

\begin{itemize}
\item Funding: Research grant 42062 from VILLUM FONDEN;  funding from the European Research Council (ERC) under the European Union’s Horizon research and innovation programme (grant agreement 101125993); funding from the Novo Nordisk Foundation through the Center for Basic Machine Learning Research in Life Science (NNF20OC0062606).
\item Conflict of interest/Competing interests (check journal-specific guidelines for which heading to use): NA
\item Ethics approval and consent to participate: NA
\item Consent for publication: Currently under review.
\item Data availability: All data used are described in the appendix.
\item Materials availability: NA
\item Code availability: The code can be found at \url{https://github.com/FrederikMR/georce_fm}
\item Author contribution: FMR derived the algorithm and implemented the method as well as benchmarks. SH and SM supervised the project with comments and corrections. SM and FMR proved Lemma 1. All authors reviewed the manuscript.
\end{itemize}

\bibliography{sn-article}


\clearpage
\section*{Appendix}
\begin{appendix}
    \section{Proofs and derivations} \label{ap:proofs}
    \subsection{Equivalence between Fr\'echet means and energy formulation for Riemannian manifolds} \label{ap:frechet_equivalence}
    \input{General/appendix/proofs/frechet_equivalence}
    \subsection{Equivalence between Fr\'echet means and energy formulation for Finslerian manifolds} \label{ap:finsler_energy_frechet}
    \input{General/appendix/proofs/finsler_energy_frechet}
    \subsection{Assumptions} \label{ap:assumptions}
\input{General/appendix/proofs/assumptions}
    \subsection{Necessary conditions} \label{ap:riemann_cond}
    \input{General/appendix/proofs/necessary_cond}
    \subsection{Update scheme} \label{ap:update_scheme}
    \input{General/appendix/proofs/update_scheme}
    \subsection{Global convergence} \label{ap:global_convergence}
    \input{General/appendix/proofs/global_convergence}
    \subsection{Local convergence} \label{ap:local_convergence}
    \input{General/appendix/proofs/local_convergence}
    \subsection{Adaptive convergence} \label{ap:adaptive_convergence}
    \input{General/appendix/proofs/adaptive_convergence}
    \subsection{Fr\'echet mean for Finsler manifolds} \label{ap:finsler_frechet}
    \input{General/appendix/proofs/finsler}
    \rowcolors{2}{gray!10}{white}
    \section{Algorithms} \label{ap:algorithms}
    \input{General/appendix/algorithms}
    \section{Manifolds} \label{ap:manifold_description}
    \input{General/appendix/manifolds}
    \section{Data and methods} \label{ap:data_methods}
    \input{General/appendix/data_methods}
    \section{Experiments} \label{ap:experiments}
    \input{General/appendix/experiments/intro}
    \subsection{Hardware} \label{ap:hardware}
    \input{General/appendix/experiments/hardware}
    \subsection{Additional experiments} \label{ap:additional_experiments}
    \input{General/appendix/experiments/additional_experiments}
\end{appendix}


\end{document}

%% file: setup/preamble.tex
\usepackage{graphicx}%
\usepackage{multirow}%
\usepackage{amsmath,amssymb,amsfonts}%
\usepackage{amsthm}%
\usepackage{mathrsfs}%
\usepackage[title]{appendix}%
\usepackage{textcomp}%
\usepackage{manyfoot}%
\usepackage{booktabs}%
\usepackage{algorithm}%
\usepackage{algorithmicx}%
\usepackage{algpseudocode}%
\usepackage{listings}%
\usepackage{wrapfig}
\usepackage{bookmark}
\usepackage{xfrac}
\usepackage[table]{xcolor}

\usepackage{natbib}

\usepackage{tikz}

%% file: setup/math.tex
\newcommand{\norm}[1]{\left\lVert#1\right\rVert}
\DeclareMathOperator*{\argmin}{arg\,min}

\newcommand\restr[2]{{
  \left.\kern-\nulldelimiterspace 
  #1 
  \vphantom{\big|} 
  \right|_{#2} 
  }}

\newcommand{\dist}[0]{\mathrm{dist}} 
\newcommand{\dif}[0]{\mathrm{d}} 

%% file: General/front_page.tex
\title[Simultaneous Optimization of Geodesics and Fr\'echet Means]{Simultaneous Optimization of Geodesics and Fr\'echet Means}


\author*[1]{\fnm{Frederik} \sur{Möbius Rygaard}}\email{fmry@dtu.dk}
\author*[1]{\fnm{Søren} \sur{Hauberg}}\email{sohau@dtu.dk}
\author*[1]{\fnm{Steen} \sur{Markvorsen}}\email{stema@dtu.dk}

\affil*[1]{\orgdiv{DTU Compute}, \orgname{Technical University of Denmark (DTU)}, \orgaddress{\street{Anker Engelundsvej 1}, \city{Kongens Lyngby}, \postcode{2800}, \country{Denmark}}}

%% file: General/abstract.tex
%
%
%
\abstract{A central part of geometric statistics is to compute the Fréchet mean. This is a well-known intrinsic mean on a Riemannian manifold that minimizes the sum of squared Riemannian distances from the mean point to all other data points. The Fréchet mean is simple to define and generalizes the Euclidean mean, but for most manifolds even minimizing the Riemannian distance involves solving an optimization problem. Therefore, numerical computations of the Fréchet mean require solving an embedded optimization problem in each iteration. We introduce the \textit{GEORCE-FM} algorithm to simultaneously compute the Fréchet mean and Riemannian distances in each iteration in a local chart, making it faster than previous methods. We extend the algorithm to Finsler manifolds and introduce an adaptive extension such that \textit{GEORCE-FM} scales to a large number of data points. Theoretically, we show that \textit{GEORCE-FM} has global convergence and local quadratic convergence and prove that the adaptive extension converges in expectation to the Fr\'echet mean. We further empirically demonstrate that \textit{GEORCE-FM} outperforms existing baseline methods to estimate the Fr\'echet mean in terms of both accuracy and runtime.}

%% file: General/keywords.tex
\keywords{Riemannian Manifolds, Finsler Manifolds, Fréchet Mean, Control Problem , Adaptive Optimization \\
\textbf{Mathematics Subject Classification} 53C22, 49Q99, 65D15}

%% file: General/Introduction.tex
Computing means on Riemannian manifolds is hard. In Euclidean space, computing means is taken for granted, due to its simple closed-form expression, which is not the case for general Riemannian manifolds. At the same time, Riemannian manifolds play an increasing role in data analysis with application to, e.g., generative models \citep{debortoli2022riemannianscorebasedgenerativemodelling, jo2024generativemodelingmanifoldsmixture, mathieu2020riemanniancontinuousnormalizingflows, moser_flow}, robotics \citep{beikmohammadi2023reactive, simeonov2021neuraldescriptorfieldsse3equivariant, feiten_robotics}, and protein modeling \citep{Watson2022.12.09.519842, deftelsen_proetin, Shapovalov_protein}.

At the center of Riemannian statistics lies the concept of the \emph{Fr\'echet mean} \citep{frechet1948} as the building block for more elaborate statistical models \citep{fletcher_geo_reg, fletcher_pga, pennec2006statriemann, pcurves, kernel_density_rm, nonlinear_rm_regression}. The Fr\'echet mean naturally generalizes the Euclidean mean, which, in the discrete case, minimizes the sum of squared distances to all data points \citep{frechet1948},
\begin{equation} \label{eq:frechet_mean}
    \mu = \argmin_{y \in \mathcal{M}}\sum_{i=1}^{N} \dist^{2}(y, a_{i}),
\end{equation}
where $\{a_{i}\}\subset \mathcal{M}$ are the given data points on a Riemannian manifold, $(\mathcal{M},g)$, with metric tensor $g$ and induced distance function, $\dist: \mathcal{M} \times \mathcal{M} \rightarrow \mathbb{R}_{+}$. Only for a few selected families of Riemannian manifolds is the distance known in closed form, and the Fr\'echet mean can be computed in such cases using, for instance, Riemannian gradient descent \citep{pennec2006statriemann, udriste2013convex}. For general Riemannian manifolds, the shortest curves, known as \textit{geodesics}, have to be computed by either solving a system of (typically non-linear) ordinary differential equations (\textsc{ode}) \citep[Page 62]{do1992riemannian} or by minimizing the energy functional with respect to curves $\gamma$ \citep[Page 194]{do1992riemannian}
\begin{equation} \label{eq:energy}
    \mathcal{E}_{\gamma}(a,b) = \frac{1}{2}\int_{0}^{1}\dot{\gamma}^{\top}(t)G(\gamma(t))\dot{\gamma}(t)\,\dif t, \quad \gamma(0)=a, \gamma(1)=b,
\end{equation}
where $a,b \in \mathcal{U}$ are the fixed start and endpoint of the curve $\gamma$, respectively, where $\phi: U \rightarrow \mathcal{M}$ is a local chart for an open subset $U \subseteq \mathbb{R}^{d}$. Thus, for general Riemannian manifolds, the Fr\'echet mean is iteratively updated by constructing geodesics minimizing Eq.~\ref{eq:energy} with respect to a curve $\gamma$ connecting each data point and the candidate of the Fr\'echet mean. This makes computations of the Fr\'echet mean time consuming and possibly intractable in practice. 

Even computing geodesics inside a geodesically strongly convex set by minimizing Eq.~\ref{eq:energy} can in itself be complicated and time consuming. Many off-the-shelf algorithms, such as \textit{ADAM} \citep{kingma2017adam} and \textit{BFGS} \citep{shanno_bfgs, broyden_bfgs, fletcher_bfgs, Goldfarb1970AFO}, either scale poorly or are not sufficiently accurate due to slow convergence \citep{georce}. Recently, the \textit{GEORCE}-algorithm \citep{georce} has shown promising results in computing geodesics on general Riemannian manifolds with fast convergence and high accuracy. 

\textbf{In this paper}, we show how to rewrite the distance-based notion of the Fr\'echet mean in Eq.~\ref{eq:frechet_mean} for both Riemannian and Finslerian manifolds to jointly minimize the Fr\'echet mean and geodesics. Using this, we introduce \textit{GEORCE-FM} (\textit{GEORCE for Fr\'echet Mean}) by extending \textit{GEORCE} to simultaneously update the geodesics and the Fréchet mean in each iteration in a local chart, as illustrated in Fig.~\ref{fig:conceptual_riemannian_frechet}. Not only does this significantly reduce runtime; we also prove that the extended version converges to a local minimum (global convergence) and local quadratic convergence similar to the original \textit{GEORCE} algorithm. We further introduce an adaptive extension to estimate the Fréchet mean to make it scalable to a high number of data points. We prove that the adaptive extension converges in expectation to the Fr\'echet mean. Empirically, we show that \textit{GEORCE-FM} outperforms baseline methods to compute the Fréchet mean in terms of accuracy and runtime for various manifolds and dimensions.
\begin{figure}[t!]
    \centering
    \includegraphics[width=1.0\textwidth]{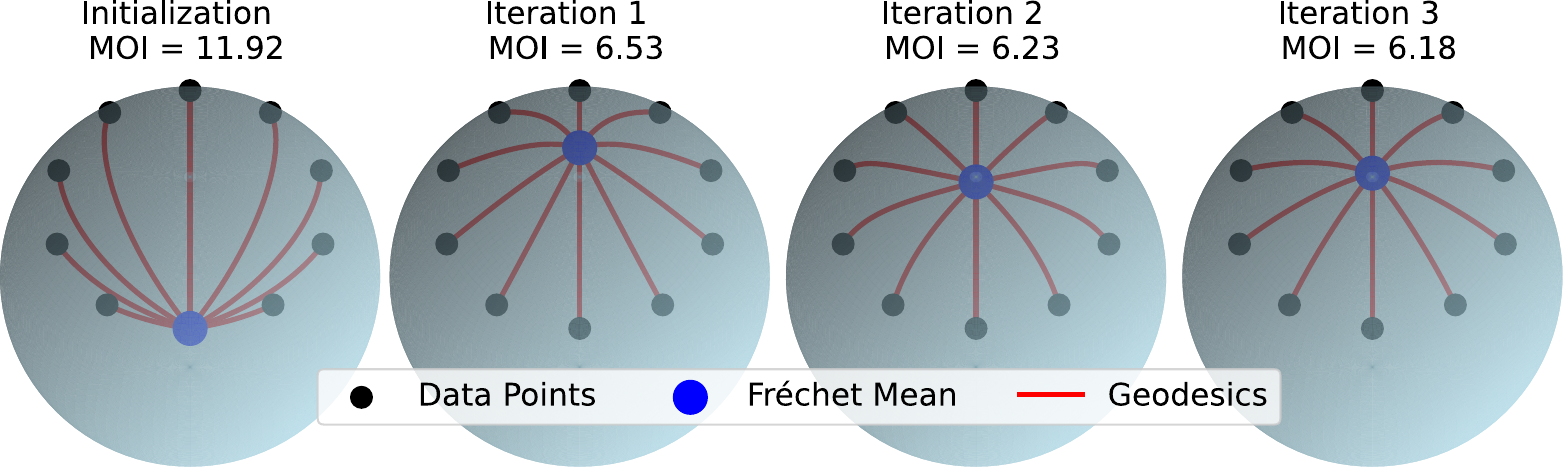}
    \caption{Iterative computation for 10 data points on a circle within the convexity radius of the north pole, and hence guaranteeing the existence and uniqueness of the Fr\'echet mean of the data points around the north pole on a unit sphere, $\mathbb{S}^{2}$, using \textit{GEORCE-FM}. After only 3 iterations the gradient of the discretized geodesic points and Fr\'echet mean is less than $10^{-4}$. ``MOI'' refers to moment of inertia given by the sum of squared distances in Eq.~\ref{eq:frechet_mean}.
    }
    \label{fig:conceptual_riemannian_frechet}
    \vspace{-1.em}
\end{figure}

%% file: Sections/background.tex
\textbf{A Riemannian manifold} is a differentiable manifold, $\mathcal{M}$, equipped with a Riemannian metric \citep[Page 38]{do1992riemannian}. The Riemannian metric, $g$, defines a smoothly varying inner product in the tangent bundle $T\mathcal{M}$. For each point $x \in \mathcal{M}$, the tangent space $T_{x}\mathcal{M}$ consists of tangent vectors to all smooth curves passing through $x \in \mathcal{M}$, and the tangent space itself is a vector space \citep[Page 7 Definition 2.6]{do1992riemannian}. The inner product induces the length of curves, $\gamma: [0,1] \rightarrow U$ connecting points.
\begin{equation*}
    \mathcal{L}_{\gamma}(a_{i},a_{j}) = \int_{0}^{1}\sqrt{\dot{\gamma}^{\top}(t)G(\gamma(t))\dot{\gamma}(t)}\,\dif t, \quad \gamma(0)=a_{i}, \gamma(1)=a_{j},
\end{equation*}
where $U \subseteq \mathbb{R}^{d}$ denotes an open set in a local chart $\phi: U \rightarrow \mathcal{M}$. A geodesic is a curve, $\gamma$ that minimizes the length between any points $x,y \in \mathcal{M}$ over all curves connecting $x$ and $y$, thus giving rise to the distance function, $\dist: \mathcal{M} \times \mathcal{M} \rightarrow \mathbb{R}$ \citep[Page 146]{do1992riemannian}. We assume that all manifolds are \textit{geodesically complete} inside a strongly convex set containing the data points such that for any two datapoints $x,y \in \mathcal{M}$ there exists at least one connecting geodesic. Geodesics can be found by minimizing the energy functional in Eq.~\ref{eq:energy} or by solving a system of \textsc{ode} [Lemma 2.3 in \cite{do1992riemannian}]
\begin{equation} \label{eq:bvp_ode}
    \ddot{\gamma}^{k}(t) + \Gamma_{ij}^{k}\left(\gamma(t)\right)\dot{\gamma}^{i}(t)\dot{\gamma}^{j}(t) = 0, \quad \gamma(0)=x,\gamma(1)=y,
\end{equation}
where $\{\Gamma_{ij}^{k}\}$ denote the Christoffel symbols and the multiplication of the upper and lower indices implies a sum following the Einstein convention. The \emph{exponential map} is constructed explicitly with geodesics, $\mathrm{Exp}_{x}: T_{x}\mathcal{M} \rightarrow \mathcal{M}$ defined by $\mathrm{Exp}_{x}(v) = \gamma(1)$ with $\gamma$ a geodesic such that $\gamma(0)=x$ and $\dot{\gamma}(0)=v$ \citep[Page 70]{do1992riemannian}. The \emph{exponential map} is locally a diffeomorphism in a sufficiently small star-shaped neighborhood $U \subset T_{x}\mathcal{M}$ for $x \in \mathcal{X}$ and within this neighborhood, the \emph{logarithmic map} is defined as $\mathrm{Log}_{x}: \mathcal{M} \rightarrow T_{x}\mathcal{M}$ as $\mathrm{Log}_{x} = \mathrm{Exp}_{x}^{-1}$ \citep[Page 70]{do1992riemannian}. 

\begin{wrapfigure}{r}{0.5\textwidth}
    \vspace{-3.em}
    \centering
    \includegraphics[width=0.5\textwidth]{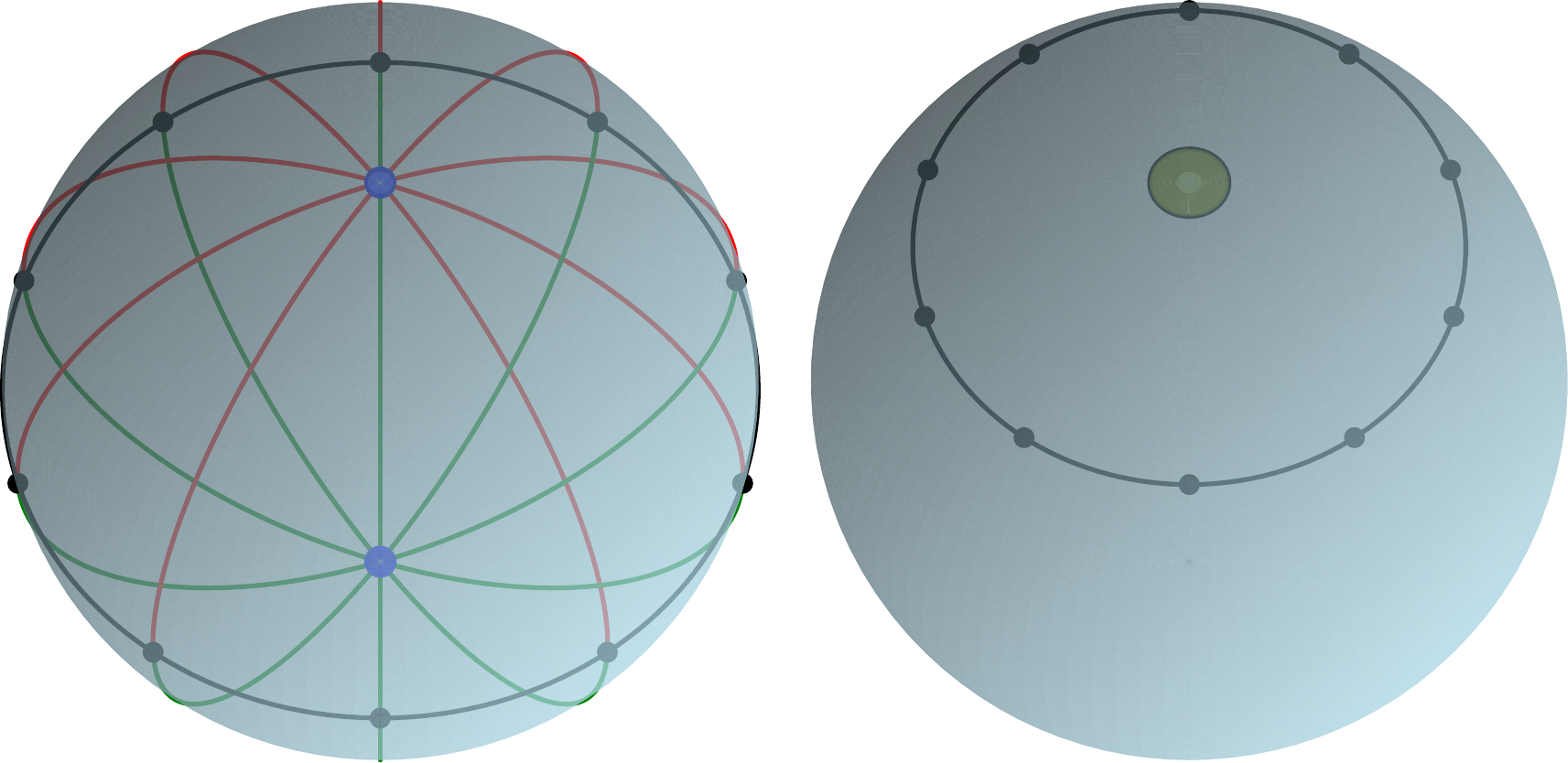}
    \caption{The left figure shows a distribution of data points (black) on the equator, where both the north and south pole are Fr\'echet means (blue). The right figure shows a distribution of data points on a circle around the north pole, but where a closed set represented by the green area is removed from the unit sphere. In this case, the Fr\'echet mean does not exist.}
    \label{fig:riemannian_non_uniqueness}
    \vspace{-6.5em}
\end{wrapfigure}

\textbf{A Finsler metric} is a generalization of Riemannian manifolds, where the tangent spaces are equipped with a Finsler metric rather than an inner product. A Finsler metric is a Minkowski norm, $F: \mathcal{M} \times T\mathcal{M}\setminus\{0\} \rightarrow \mathbb{R}_{+}$, in the sense that it is a smooth function and satisfies the following requirements \citep[Page 11]{ohta2021comparison}
\begin{itemize}
    \item $F(x, cv) = cF(x, v)$ for all $v \in T_{x}\mathcal{M} \setminus \{0\}$ and $c>0$ for $x \in \mathcal{M}$. 
    \item The fundamental tensor, $G := \frac{\partial^{2}F}{\partial v^{i}\partial v^{j}}(x,v)$, is symmetric and positive definite.
\end{itemize}
Geodesics for a Finsler manifold are curves that minimize the integrated squared Finsler length over a set of curves $\gamma$, 
\begin{equation} \label{eq:finsler_energy}
    \min_{\gamma} \frac{1}{2}\int_{0}^{1}F^{2}\left(\gamma(t), \dot{\gamma}(t)\right) \,\dif t.
\end{equation}
which gives rise to the following \textsc{ode} system generalizing Eq.~\ref{eq:bvp_ode}
\begin{equation} \label{eq:finsler_ode}
    \ddot{\gamma}^{k}(t) + \eta_{ij}^{k}\left(\gamma(t), \dot{\gamma}(t)\right)\dot{\gamma}^{i}(t)\dot{\gamma}^{j}(t) = 0, \quad \gamma(0)=x,\gamma(1)=y,
\end{equation}
where $\{\eta_{ij}^{k}\}$ generalizes the Riemannian Christoffel symbols to the Finslerian case and also depends on the velocity \citep[Page 27]{ohta2021comparison}.

\textbf{The Fr\'echet Mean} is defined in the continuous case as \citep{frechet1948, pennec2006statriemann}
\begin{equation} \label{eq:frechet_cont}
    \mu = \argmin_{y \in \mathcal{M}}\mathbb{E}\left[\dist^{2}\left(y, X\right)\right] = \argmin_{y \in \mathcal{M}}\int_{\mathcal{M}}\dist^{2}\left(y, z\right)p_{X}(z) \,\dif \mathcal{M}(z),
\end{equation}
for a random variable $X$ defined on a probability space $\left(\Omega, \mathcal{F}, \mathbb{P}\right)$, where $\dif \mathcal{M}(z)$ refers to the volume measure of the manifold $\mathcal{M}$ \citep{pennec2006statriemann}. The discretized estimator of the Fr\'echet mean in Eq.~\ref{eq:frechet_cont} is given by Eq.~\ref{eq:frechet_mean}. For symmetric manifolds, the Fr\'echet mean is a maximum likelihood estimator \citep{fletcher_geo_reg}. Despite the fact that the Fr\'echet mean naturally generalizes the Euclidean mean, it is not necessarily unique, nor does it exist, as illustrated in Fig.~\ref{fig:riemannian_non_uniqueness} for two simple data distributions on the unit sphere. If the data distribution for a complete Riemannian manifold is within a convex open ball of radius $\rho$ around a point $m \in \mathcal{M}$, then the Fr\'echet mean is unique and exists \citep{kendall_1990, karcher_1977}, while it has also been shown that the estimator of the Fr\'echet mean has strong consistency \cite{Ziezold1977, bpc2003, huckemann2016backwardnesteddescriptorsasymptotics}. In this paper, we will implicitly assume that there exists at least one Fr\'echet mean for any data distribution on a Riemannian manifold $\mathcal{M}$. The gradient of the sum of squared distances in Eq.~\ref{eq:frechet_mean} is the sum of logarithmic maps \citep{karcher_1977, pennec2006statriemann}
\begin{equation} \label{eq:frechet_gradient}
    \nabla_{y}\sum_{i=1}^{N}\dist^{2}(y,a_{i}) = \sum_{i=1}^{N}2\mathrm{Log}(y,a_{i}).
\end{equation}
If the logarithmic map is available in closed-form, the Fr\'echet mean can be computed using gradient descent \citep{pennec2006statriemann, udriste2013convex} or Karcher flows \citep{karcher_1977}. However, if the logarithmic map is not available, then using gradient descent will require estimating the logarithmic map numerically by either solving the \textsc{ode} system in Eq.~\ref{eq:bvp_ode} or minimizing the energy in Eq.~\ref{eq:energy}.

\paragraph{Computation of the Fr\'echet mean.} Different algorithms have been proposed to estimate the Fr\'echet mean such as Karcher-flows \citep{karcher_1977}, which has been used for computing the Fr\'echet mean for the space of symmetric positive definite matrices \cite{brooks2019riemannianbatchnormalizationspd}, and Riemannian gradient descent \citep{udriste2013convex}. However, both algorithms assume access to the logarithmic map, which requires computing geodesics. Different algorithms achieve better convergence than Karcher flows and Riemannian gradient descent, but are specifically designed for subclasses of Riemannian manifolds such as hyperbolic spaces \citep{lou2021differentiatingfrechetmean} and n-dimensional spheres \citep{eichfelder2018algorithmcomputingfrechetmeans}. Without access to a closed-form solution to the logarithmic map, the computation of the Fr\'echet mean is based on geodesic computations. Various methods for computing geodesics have been proposed by solving an \textsc{ode} system as a boundary value problem \citep{scipy_bvp, leapfrog_noakes, leapfrog_optimal_control, noakes_optimal_control}, as a fixed point problem \citep{hennig2014probabilisticsolutionsdifferentialequations, arvanitidis2019fastrobustshortestpaths} or by directly minimizing discretized geodesics \citep{georce, shao2017riemannian, arvanitidis2021latent}. However, the Fr\'echet mean requires computing a potentially high number of geodesics in each iteration of the Fr\'echet mean making it intractable for a high number of data points. Although competing definitions of a mean value on a manifold have been proposed \citep{eltzner2022diffusionmeansgeometricspaces, nielsen2024fastproxycentersjeffreys, Fletcher2009GeometricMedian}, the Fr\'echet mean remains the go-to standard for geometric computations on manifolds. Estimation of alternative definitions of the mean on a manifold is also computationally expensive, where, for example, the \emph{diffusion mean} \citep{eltzner2022diffusionmeansgeometricspaces} involves Monte Carlo estimation of diffusion bridges \citep{sommer2017bridge} or gradient descent methods using neural networks \citep{rygaard2025scorematchingriemanniandiffusion}.

%% file: Sections/georce_fm.tex
Rather than minimizing the sum of squared distance in Eq.~\ref{eq:frechet_mean}, we will instead formulate the Fr\'echet mean as minimizing the sum of weighted energies in local coordinates.
\begin{equation} \label{eq:frechet_energy}
    \mu, \{\gamma_{i}\}_{i=1}^{N} = \argmin_{\substack{y \in \mathcal{M} \\ \{\gamma_{i}\}_{i=1}^{N}}}\sum_{i=1}^{N}w_{i}\mathcal{E}_{\gamma_{i}}\left(a_{i}, y\right),
\end{equation}
where $\mathcal{E}_{\gamma}(a,b)$ is the energy corresponding to Eq.~\ref{eq:energy} for a curve, $\gamma$, with starting point $a \in U$ and end point $b \in I$ for a local coordinate system $\phi: U \subseteq \mathbb{R}^{d} \rightarrow \mathcal{M}$, while $\{w_{i}\}_{i=1}^{N} \subset \mathbb{R}_{+}$ are positive weights. We assume that $\{a_{i}\}_{i=1}^{N} \subset U$ are observed data in a local coordinate system within a geodesically convex set such that there exists at least one Fr\'echet mean. For the Fr\'echet mean in Eq.~\ref{eq:frechet_mean} the weights are one. We show in the following proposition that the minima of Eq.~\ref{eq:frechet_energy} are equivalent to the Fr\'echet means.
\begin{lemma} \label{lemma:energy_frechet}
    Let $\mu^{\mathrm{FM}}$ denote the set of minima of Eq.~\ref{eq:frechet_mean}, and let $\mu^{\mathrm{Energy}}$ denote the set of minima of Eq.~\ref{eq:frechet_energy} with $w_{i}=1$ for all $i \in \{1,\dots,N\}$. Then
    \begin{equation*}
        \mu^{\mathrm{Energy}} = \mu^{\mathrm{FM}}.
    \end{equation*}
\end{lemma}
\begin{proof}
    See Appendix~\ref{ap:frechet_equivalence}.
\end{proof}

\begin{wrapfigure}{r}{0.3\textwidth}
    \vspace{-2.em}
    \centering
    \includegraphics[width=0.3\textwidth]{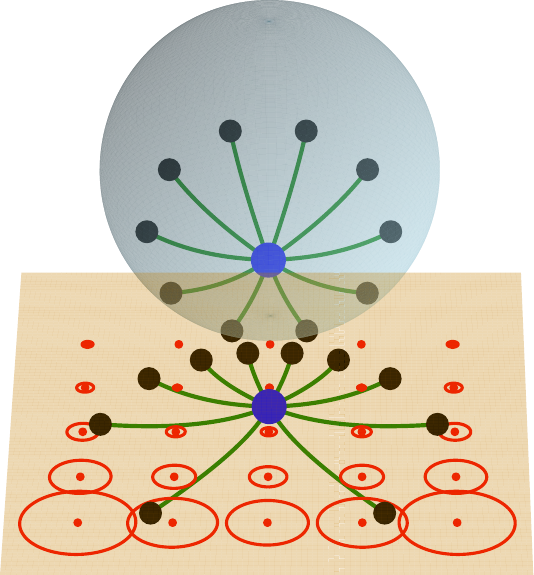}
    \caption{We compute the geodesics and Fr\'echet mean simultaneously in a local chart of a Riemannian manifold. Green curves are the geodesics, black points are the datapoints, while red is the indicatrix field.}
    \label{fig:riemannian_chart_conceptual}
    \vspace{-7em}
\end{wrapfigure}

We can therefore compute geodesics and the Fr\'echet mean simultaneously by minimizing Eq.~\ref{eq:frechet_energy} rather than minimizing Eq.~\ref{eq:frechet_mean}. We discretize the geodesics in Eq.~\ref{eq:frechet_energy} similarly to \cite{shao2017riemannian} in a local chart as
\begin{equation} \label{eq:frechet_energy_disc}
    \begin{split}
        \min_{\left(x_{t,i}, y\right)} \quad &\sum_{i=1}^{N}w_{i}\sum_{t=0}^{T-1} \left(x_{t,i+1}-x_{t,i}\right)^{\top}G(x_{t,i})\left(x_{t,i+1}-x_{t,i}\right), \\
        \text{s.t.} \quad &x_{0,i}= a_{i}, x_{T,i}=y, \quad \forall i \in \{1,\dots,N\},
    \end{split}
\end{equation}
where $\{x_{ti}\}_{t,i}$ corresponds to the discretized geodesics connecting the data points, $\{a_{i}\}_{i=1}^{N}$, and the candidate mean point, $y$. In this way, the estimation of the Fr\'echet mean corresponds to estimating $N$ geodesics, where the Fr\'echet mean is the end point of the estimation. 

Similarly to \cite{georce}, we formulate the above as a control problem in a local chart extending it to the computation of multiple geodesics, where all geodesics have the same non-fixed endpoint corresponding to the Fr\'echet mean
\begin{equation} \label{eq:energy_control}
    \begin{split}
        \min_{(x_{t,i},u_{t,i})} E(x) &:= \min_{(x_{t,i},u_{t,i})}\left\{\sum_{i=1}^{N}w_{i}\sum_{t=0}^{T-1}u_{t,i}^{\top}G(x_{t,i})u_{t,i}\right\} \\
        x_{t+1,i} &= x_{t,i}+u_{t,i}, \quad t=0,\dots,T-1, \, i=1,\dots,N, \\
        x_{0,i}&=a_{i},x_{T,i}=y, \quad i=1,\dots,N.
    \end{split}
\end{equation}
Unlike \cite{georce}, we consider the end points of the geodesics as a variable to be estimated. We illustrate our approach in Fig.~\ref{fig:riemannian_chart_conceptual}. The advantage of formulating the optimization problem in Eq.~\ref{eq:frechet_energy_disc} as a control problem is that we are able to decompose the optimization problem in Eq.~\ref{eq:energy_control} into convex subproblems, resulting in the following necessary conditions for a minimum.
\begin{proposition} \label{prop:riemann_cond}
    The necessary conditions for a minimum in Eq.~\ref{eq:energy_control} are
    \begin{equation} \label{eq:energy_opt_condtions}
        \begin{split}
            &2w_{i}G(x_{t,i})u_{t,i}+\mu_{t,i}=0, \quad t=0,\dots, T-1, \, i=1,\dots,N, \\
            &x_{t+1,i}=x_{t,i}+u_{t,i}, \quad t=0,\dots,T-1, \, i=1,\dots,N, \\
            &\restr{\nabla_{x}\left[w_{i}u_{t,i}^{\top}G(x)u_{t,i}\right]}{x=x_{t,i}}+\mu_{t,i}=\mu_{t-1,i}, \quad t=1,\dots,T-1, \, i=1,\dots,N ,\\
            &0 = \sum_{i=1}^{N}\mu_{T-1,i}, \\
            &x_{0,i}=a_{i}, x_{T,i}=y, \quad i=1,\dots,N.
        \end{split}
    \end{equation}
    where $\mu_{t,i} \in \mathbb{R}^{d}$ denotes the dual prices for the control problem for $t=0,\dots,T-1$ and $i=1,\dots,N$.
\end{proposition}
\begin{proof}
    See Appendix~\ref{ap:riemann_cond}.
\end{proof}
If $u^{\top}G(x)u$ is convex in $x$, then the necessary conditions are also sufficient for a global minimum point, that is, there exists only one Fr\'echet mean, which is a global minimum point. If not, a local minimum point may be determined \cite{karcher_1977}. Note that compared to \textit{GEORCE} \citep{georce}, there is an additional condition related to the dual prices, $\{\mu_{t,i}\}_{t,i}$, since we consider the end point of the geodesics as a variable. The system of equations in Eq.~\ref{eq:energy_opt_condtions} cannot be solved for general $G(\cdot)$. We aim to derive an iterative scheme to solve Eq.~\ref{eq:energy_opt_condtions} and consider the variables $\{x_{t,i}^{(k)}\}_{t,i}$ and $\{u_{t,i}^{(k)}\}_{t,i}$ in iteration $k$. We apply a similar ``trick'' as \cite{georce} and fix the following variables.
\begin{equation*}
    \begin{split}
        \nu_{t,i} &:= \restr{\nabla_{x}\left(w_{i}u_{t,i}^{\top}G(x)u_{t,i}\right)}{x=x_{t,i}^{(k)},u_{t,i}=u_{t,i}^{(k)}}, \quad t=1,\dots,T-1, \, i=1,\dots,N, \\
        G_{t,i} &:= G\left(x_{t,i}^{(k)}\right), \quad t=0,\dots,T-1, \, i=1,\dots,N. \\
    \end{split}
\end{equation*}
Using this, Eq.~\ref{eq:energy_opt_condtions} reduces to
\begin{equation} \label{eq:energy_zero_point_problem}
    \begin{split}
        &\nu_{t,i}+\mu_{t,i} = \mu_{t-1,i}, \quad t=1,\dots,T-1, \, i=1,\dots,N, \\
        &2w_{i}G_{t,i}u_{t,i}+\mu_{t,i} = 0, \quad t=0,\dots,T-1, \, i=1,\dots,N, \\
        &\sum_{t=0}^{T-1}u_{t,i}=y-a_{i}, \quad i=1,\dots,N, \\
        &\sum_{i=1}^{N}\mu_{T-1,i} = 0. \\
    \end{split}
\end{equation}
The subproblems are tied together by the equation at time $T-1$ for the dual variables $\mu_{T-1,i}$, and the boundary conditions for the end point. In the following, this is circumvented by deriving a closed-form formula for $y$ that allows a decomposition into $N$ geodesic problems.
\begin{proposition} \label{prop:update_scheme}
    The solution to Eq.~\ref{eq:energy_zero_point_problem} for $u_{t},\mu_{t}$ and $x_{t}$ is
    \begin{equation} \label{eq:energy_update_schem}
        \begin{split}
            &y = W^{-1}V, \\
            &\mu_{T-1,i} = \left(\sum_{t=0}^{T-1}G_{t,i}^{-1}\right)^{-1}\left(2w_{i}(a_{i}-y)-\sum_{t=0}^{T-1}G_{t,i}^{-1}\sum_{t>j}^{T-1}\nu_{j,i}\right), \quad i=1,\dots,N, \\
            &u_{t,i} = -\frac{1}{2w_{i}}G_{t,i}^{-1}\left(\mu_{T-1,i}+\sum_{j>t}^{T-1}\nu_{j,i}\right), \quad t=0,\dots,T-1, \, i=1,\dots,N \\
            &x_{t+1,i} = x_{t,i}+u_{t,i}, \quad t=0,\dots,T-2, \, i=1,\dots,N, \\
            &x_{0,i}=a_{i} \quad i=1,\dots,N,
        \end{split}
    \end{equation}
    where
    \begin{equation} \label{eq:w_v_def}
        \begin{split}
            W &= \sum_{i=1}^{N}w_{i}\left(\sum_{t=0}^{T-1}G_{t,i}^{-1}\right)^{-1}, \\
            V &= \sum_{i=1}^{N}w_{i}\left(\sum_{t=0}^{T-1}G_{t,i}^{-1}\right)^{-1}a_{i}-\frac{1}{2}\sum_{i=1}^{N}\left(\sum_{t=0}^{T-1}G_{t,i}^{-1}\right)^{-1}\sum_{t=0}^{T-1}G_{t,i}^{-1}\sum_{j>t}^{T-1}\nu_{j,i}.
        \end{split}
    \end{equation}
\end{proposition}
\begin{proof}
    See Appendix~\ref{ap:update_scheme}.
\end{proof}
\begin{figure}[t!]
    \centering
    \includegraphics[width=1.0\textwidth]{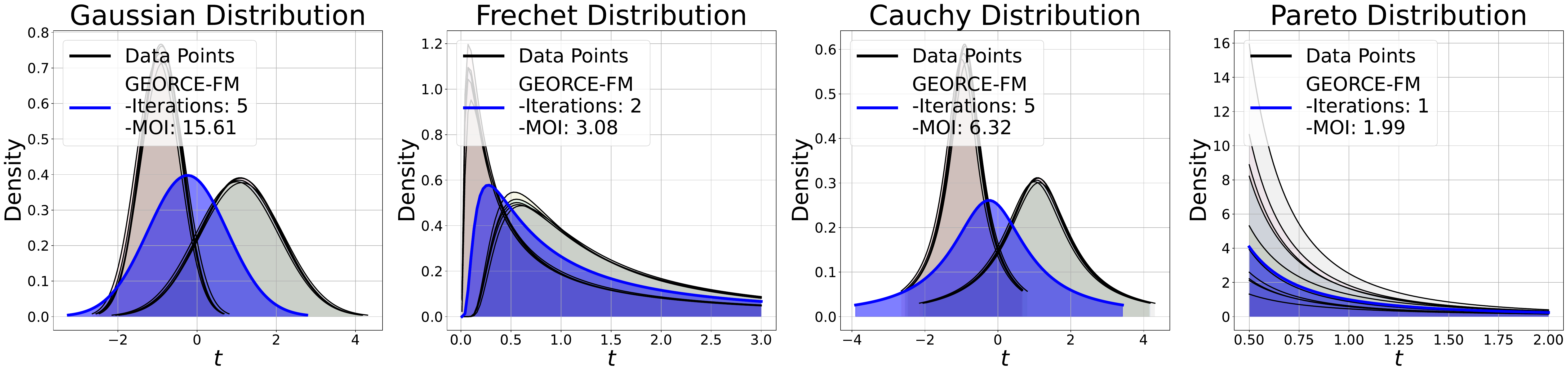}
    \vspace{-1.em}
    \caption{The application of \textit{GEORCE-FM} to Information geometry for 4 different distributions equipped with the Fisher-Rao metric \citep{miyamoto2024closedformexpressionsfisherraodistance} with synthetic data. From left to right we have: Gaussian distribution, Fr\'echet distribution, Cauchy distribution and Pareto distribution. The black outlined distributions are the data, where the blue outlined distribution is the estimated Fr\'echet mean using \textit{GEORCE-FM} and ``MOI'' refers to moment of inertia given by the sum of squared distances in Eq.~\ref{eq:frechet_mean}. The details of the manifolds and data points can be found in Appendix~\ref{ap:manifold_description}.}
    \label{fig:information_geometry_frechet_mean}
    \vspace{-1.em}
\end{figure}
Using the update scheme in Proposition~\ref{prop:update_scheme}, we show \textit{GEORCE-FM} in pseudo-code in Algorithm~\ref{al:georce_fm}, where we apply line-search for the sum of discretized energy using backtracking with Armijo condition \citep{armijo1966minimization} with decay rate $\rho=0.5$ \cite{georce}. We stop \textit{GEORCE-FM} when the average gradient norm is below a threshold. Note that in Algorithm~\ref{al:georce_fm}, then $x_{T,i}=y$ for all $i=1,\dots,N$ in each iteration. Furthermore, when $y$ has been determined, the remaining update formulas can be run in parallel for $i=1,\dots,N$. Note that similar to \textit{GEORCE} \citep{georce}, all computations are within a local chart, and it is implicitly assumed that the initial curve and updates are defined within this chart. This can be ensured using line-search, or by changing the chart if the computations are close to the boundary of the domain. For a detailed discussion on this, we refer to \citep{georce}. For the examples in this paper, we will assume that the computations are well-defined within the chart, and we will therefore not change the chart doing computations, although this can easily be added to the algorithm.

Similarly to \textit{GEORCE} \citep{georce}, the algorithm, as will be proven below, will converge to a local minimum. To check whether the found solution is actually a local minimum or a global minimum, the algorithm can be run for different initialization curves and compare the corresponding sum of squared lengths in Eq.~\ref{eq:frechet_energy_disc}.

We estimate in Fig.~\ref{fig:information_geometry_frechet_mean} the Fr\'echet mean using \textit{GEORCE-FM} for four different examples of distributions in information geometry: normal distribution, Fr\'echet distribution, Cauchy distribution and Pareto distribution for $T=100$ grid points, which are equipped with the Fisher-Rao metric \citep{miyamoto2024closedformexpressionsfisherraodistance}. We see that \textit{GEORCE-FM} converges after only a few iterations.
\begin{algorithm}[!ht]
    \caption{GEORCE-FM for Riemannian Manifolds}
    \label{al:georce_fm}
    \begin{algorithmic}[1]
        \State \textbf{Input}: $\mathrm{tol}$, $a_{1:N}$, $T$    \State \textbf{Output}: Geodesic estimate $x_{0:T}$        Set $y^{(0)} \leftarrow a_{0}$, $x_{t,i}^{(0)} \leftarrow a_{i}+\frac{y^{(0)}-a_{i}}{T}t$ and  $u_{t,i}^{(0)} \leftarrow \frac{y^{(0)}-a_{i}}{T}$ for $t=0.,\dots,T$ and $i=1,\dots,N$.
        \While $\frac{1}{N}\norm{\restr{\nabla_{y}E(y)}{y=x_{t,i}^{(k)}}}_{2} > \mathrm{tol}$ 
        \State $G_{t,i} \leftarrow G\left(x_{t,i}^{(k)}\right)$ for $t=0,\dots,T-1$ and $i=1,\dots,N$.
        \State $\nu_{t,i} \leftarrow \restr{\nabla_{x}\left(u_{t,i}^{(k)}G\left(x\right)u_{t,i}^{(k)}\right)}{x=x_{t,i}^{(k)}}$ for $t=1,\dots,T-1$ and $i=1,\dots,N$.
        \State $y \leftarrow W^{-1}V$ with $W,V$ given by Eq.~\ref{eq:w_v_def}.
        \State $\mu_{T-1,i} \leftarrow \left(\sum_{t=0}^{T-1}G_{t,i}^{-1}\right)^{-1}\left(2w_{i}(a_{i}-y)-\sum_{t=0}^{T-1}G_{t,i}^{-1}\sum_{t>j}^{T-1}\nu_{j,i}\right)$ for $i=1,\dots,N$ and $t=1,\dots,T-1$.
        \State $u_{t,i} \leftarrow -\frac{1}{2w_{i}}G_{t,i}^{-1}\left(\mu_{T-1,i}+\sum_{j>t}^{T-1}\nu_{j,i}\right)$ for $t=0,\dots,T-1$ and  $i=1,\dots,N$.
        \State $x_{t+1,i} \leftarrow x_{t,i}+u_{t,i}$ for $t=0,\dots,T-2$ and $i=1,\dots,N$.
        \State Using line search find $\alpha^{*}$ for the following optimization problem for the discrete sum of energy $E$
        \begin{equation*}
            \begin{split}
                \min_{\alpha}\quad &E\left(x_{0:T,1:N}\right) \quad \text{(exact line search)} \\
                \text{s.t.} \quad &x_{t+1,i}=x_{t,i}+\alpha \tilde{u}_{t,i}+(1-\alpha)u_{t,i}^{(k)}, \quad t=0,\dots,T-1, \, i=1,\dots,N. \\
                &\tilde{u}_{t,i} = \alpha u_{t,i}+(1-\alpha)u_{t,i}^{(k)}, \quad t=0,\dots,T-1, \, i=1,\dots,N. \\
                &x_{0,i}=a_{i}.
            \end{split}
        \end{equation*}
        \State Set $u_{t,i}^{(k+1)} \leftarrow \alpha^{*}u_{t,i}+(1-\alpha^{*})u_{t,i}^{(k)}$ for $t=0,\dots,T-1$ and $i=1,\dots,N$.
        \State Set $x_{t+1,i}^{(k+1)} \leftarrow x_{t,i}^{(k+1)}+u_{t,i}^{(k+1)}$ for $t=0,\dots,T-1$ and $i=1,\dots,N$.
        \EndWhile
        \State return $x_{t,i}$ for $t=0,\dots,T-1$ for $i=1,\dots,N$.
    \end{algorithmic}
\end{algorithm}
\paragraph{Complexity.} \textit{GEORCE-FM} in Algorithm~\ref{al:georce_fm} requires matrix inversions for $N$ geodesics and therefore scales $\mathcal{O}\left(NTd^{3}\right)$ in each iteration, where $d$ is the manifold dimension. The advantage of \textit{GEORCE-FM} is that we can simultaneously compute geodesics and Fr\'echet mean rather than solving the geodesic subproblem in each iteration of the Fr\'echet mean.

\paragraph{Convergence results.} We prove that \textit{GEORCE-FM} has global convergence and local quadratic convergence by generalizing the original proofs in the \textit{GEORCE}-algorithm \citep{georce}.
\begin{proposition}[Convergence Properties]\label{prop:global_local_convergence}
    Let $E^{(k)}$ be the value of the sum of discretized energy functional for the solution after iteration $k$ (with line search) in \textit{GEORCE-FM}. The \textit{GEORCE-FM} algorithm has the following properties
    \begin{itemize}
        \item Global convergence: If the starting point $\left(x^{(0)}, u^{(0)}\right)$ is feasible in the sense that the geodesic candidates have the correct start and share the same end point, then the series $\left\{E^{(k)}\right\}_{i>0}$ will converge to a (local) minimum. 
        \item Local quadratic convergence: Under sufficient regularity of the energy functional (see Appendix~\ref{ap:assumptions}) it holds that if the energy functional has a strongly unique (local) minimum point $z^{*}=\left(x^{*}, u^{*}, y^{*}\right)$ and locally the step-size $\alpha^{*}=1$, i.e. no line search, then \textit{GEORCE-FM} has locally quadratic convergence, i.e.
        \begin{equation*}
                \exists \epsilon>0:\quad \exists c>0: \quad \forall z^{(k)} \in B_{\epsilon}(z^{*}): \quad \norm{z^{(k+1)}-z^{*})} \leq c \norm{z^{(k)}-z^{*}}^{2},
        \end{equation*}
        where $z^{(k)}=(x^{(k)}, u^{(k)}, y^{(k)})$ is the solution from \textit{GEORCE-FM} in iteration $k$, and $B_{\epsilon}\left(z^{*}\right)$ is the ball with radius $\epsilon>0$ and center $z^{*}$.
    \end{itemize}
\end{proposition}
\begin{proof}
    See Appendix~\ref{ap:global_convergence} for the proof of global convergence and Appendix~\ref{ap:local_convergence} for the proof of local convergence.
\end{proof}
With Proposition~\ref{prop:global_local_convergence} we see that not only does \textit{GEORCE-FM} allow simultaneous optimization of the Fr\'echet mean geodesics, but it also has global convergence similar to gradient descent methods and local quadratic convergence similar to the Newton method.

%% file: Sections/adaptive.tex
In the previous section, we have seen that \textit{GEORCE-FM} can compute geodesics simultaneously with the Fr\'echet mean. However, \textit{GEORCE-FM} scales by $\mathcal{O}\left(NTd^{3}\right)$, making it computationally expensive to use for a large number of datapoints. In this section, we introduce an adaptive extension of \textit{GEORCE-FM} by considering a stochastic estimator of the minimization problem in Eq.~\ref{eq:frechet_energy} with mini-batches applying subsets of data points.
\begin{equation} \label{eq:frechet_energy_disc_stoch}
   \min_{y \in \mathcal{M}} \sum_{i \in \mathcal{I}}w_{i}\mathcal{E}\left(a_{i}, y\right),
\end{equation}
where $\mathcal{I} = \{i_{1}, \dots, i_{n}\}$ is an index set with $n$ distinct random integer variables of the form $1 \leq i_{1} < i_{2} < \dots i_{n} \leq N$. The idea is to adaptively update the estimator of the Fr\'echet mean using only the stochastic estimates from random mini-batches of data. Inspired by the update scheme in Proposition~\ref{prop:update_scheme} and Eq.~\ref{eq:w_v_def}, we define the following stochastic values using only mini-batches of the data,
\begin{equation} \label{eq:w_v_adaptive}
    \begin{split}
        \tilde{W} &= \sum_{i \in \mathcal{I}}w_{i}\left(\sum_{t=0}^{T-1}G_{t,i}^{-1}\right)^{-1}, \\
        \tilde{V} &= \sum_{i \in \mathcal{I}}w_{i}\left(\sum_{t=0}^{T-1}G_{t,i}^{-1}\right)^{-1}a_{i}-\frac{1}{2}\sum_{i \in \mathcal{I}}\left(\sum_{t=0}^{T-1}G_{t,i}^{-1}\right)^{-1}\sum_{t=0}^{T-1}G_{t,i}^{-1}\sum_{j>t}^{T-1}g_{j,i}
    \end{split}
\end{equation}
such that the stochastic estimate, $\tilde{y}$, of the Fr\'echet mean in the \textit{GEORCE-FM} algorithm is
\begin{equation} \label{eq:adaptive_update}
    \tilde{y} = \tilde{W}^{-1}\tilde{V},
\end{equation}
We propose to adaptively estimate $\tilde{W}$ and $\tilde{V}$ for different mini-batches of data and then update the Fr\'echet mean by Eq.~\ref{eq:adaptive_update}. We show the algorithm in pseudo-code in Algorithm~\ref{al:adaptive_georce_fm}. Note that to have a proper estimate of $V$ and $W$ in Eq.~\ref{eq:w_v_adaptive}, we run \textit{GEORCE-FM} for a fixed number of iterations on the mini-batch of the dataset.
\begin{algorithm}[H]
    \caption{Adaptive GEORCE-FM for Riemannian Manifolds}
    \label{al:adaptive_georce_fm}
    \begin{algorithmic}[1]
        \State \textbf{Input}: $\mathrm{tol}$, $a_{1:N,i}$, $T$, $\text{sub\_iters}$.
        \State \textbf{Output}: Geodesic estimate $x_{0:T}$.
        \State Compute random index set, $\mathcal{I}$.
        \State Compute $\text{sub\_iters}$ iterations using $\left(\hat{W}^{(0)}, \hat{V}^{(0)}\right) \leftarrow $\textit{GEORCE\_FM}$\left(\{a_{i}\}_{i \in \mathcal{I}}\right)$ using algorithm~\ref{al:georce_fm} and Eq.~\ref{eq:w_v_adaptive}.
        \State $k \leftarrow 1$.
        \While $\norm{y^{(k)}-y^{(k-1)}}_{2} > \mathrm{tol}$
        \State Compute random index set, $\mathcal{I}$.
        \State Compute $\text{sub\_iters}$ iterations using $\left(\tilde{W}, \tilde{V}\right) \leftarrow $\textit{GEORCE\_FM}$\left(\{a_{i}\}_{i \in \mathcal{I}}\right)$ using algorithm~\ref{al:georce_fm} and Eq.~\ref{eq:w_v_adaptive}.
        \If {convergence}
        \State $\alpha^{*} \leftarrow \frac{1}{k+1}$ 
        \Else
        \State $\alpha^{*} \leftarrow \lambda$
        \EndIf
        \State $\hat{W}^{(k)} \leftarrow \alpha^{*} \tilde{W}+(1-\alpha^{*}) \hat{W}^{(k-1)}$.
        \State $\hat{V}^{(k)} \leftarrow \alpha^{*} \tilde{V}+(1-\alpha^{*}) \hat{V}^{(k-1)}$.
        \State $y^{(k)} \leftarrow \left(\hat{W}^{(k)}\right)^{-1}\hat{V}^{(k)}$.
        \State $k \leftarrow k +1$.
        \EndWhile
        \State return $y^{(k)}$.
    \end{algorithmic}
\end{algorithm}

\paragraph{Convergence results.} We prove in the following that the adaptive extension of \textit{GEORCE-FM} converges in expectation assuming sufficient regularity of the index set size and adaptive scheme.
\begin{proposition}[Adaptive Convergence]\label{prop:adaptive_convergence}
    Assume sufficient regularity of the index set and adaptive scheme (see Appendix~\ref{ap:adaptive_convergence}). Then the adaptive version of \textit{GEORCE-FM} will converge to a local minimum point in expectation.
\end{proposition}
\begin{proof}
    See Appendix~\ref{ap:adaptive_convergence} for the proof and assumptions.
\end{proof}
In Proposition~\ref{prop_adaptive_convergence} we only assume sufficient regularity of the adaptive update scheme, and therefore the current update scheme in Algorithm~\ref{al:adaptive_georce_fm} can easily be replaced by other adaptive updating schemes. Note that the ``momentum'' type of estimation for $V$ and $W$ is not assumed in the convergence proof, but is added in order to increase convergence speed similar to e.g. \textit{ADAM} \citep{kingma2017adam}.

%% file: Sections/extensions.tex
%
%
\begin{figure}[t!]
    \centering
    \includegraphics[width=1.0\textwidth]{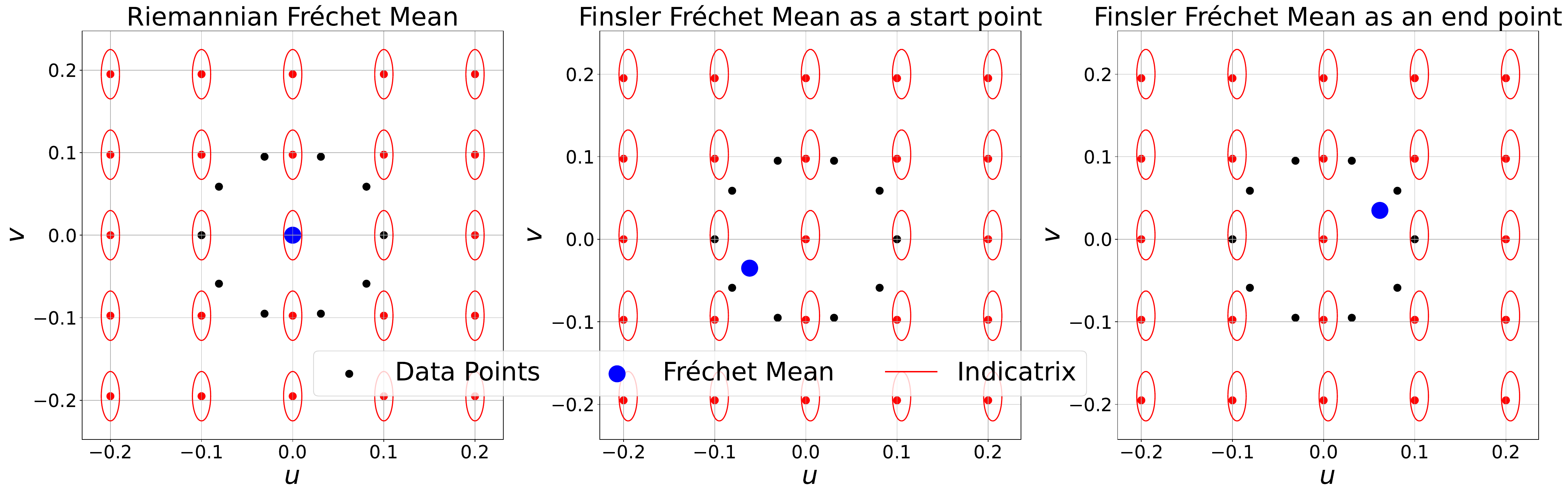}
    \vspace{-3.0mm}
    \caption{We consider a constant indicatrix field of centered ellipsis in the left-most figure, and the corresponding Riemannian Fr\'echet mean. The center figures shows displaced ellipses such that they are non-centered and the corresponding Fr\'echet mean as a starting point. The right-most figure shows the Finslerian Fr\'echet mean as an end point.}
    \label{fig:finsler_frechet_conceptual}
    \vspace{-1.em}
\end{figure}
%
In this section, we generalize the previous results to the Finslerian manifolds. Let $\dist_{F}$ denote the distance function on a Finsler manifold and define the Fr\'echet mean similar to the Riemannian case as
\begin{equation} \label{eq:finsler_frechet_mean}
    \mu = \argmin_{y \in \mathcal{M}}\sum_{i=1}^{N} \dist_{F}^{2}(y, a_{i}).
\end{equation}
In Eq.~\ref{eq:finsler_frechet_energy} we define the Fr\'echet mean as the start point of the distance. Since the Finsler metric is not symmetric, this is not equivalent to having the Fr\'echet mean as an end point as illustrated in Fig.~\ref{fig:finsler_frechet_conceptual}. Note that the Fr\'echet mean as a starting point or end point is also denoted the \textit{forward p-mean} and \textit{backward p-mean} \citep{Arnaudon_Nielsen_2012}. Analogously to the Riemannian case we prove in Appendix~\ref{ap:finsler_energy_frechet} that we can re-write the optimization problem in Eq.~\ref{eq:finsler_frechet_mean} to jointly minimize the energy with respect to the energy and candiate Fr\'echet mean, that is,
\begin{equation} \label{eq:finsler_frechet_energy}
    \mu^{\mathrm{Energy}}, \{\gamma_{i}\}_{i=1}^{N} = \argmin_{\substack{y \in \mathcal{M} \\ \{\gamma_{i\}_{i=1}^{N}}}}\sum_{i=1}^{N} \mathcal{E}^{(F)}_{\gamma_{i}}(y, a_{i}),
\end{equation}
where $\mathcal{E}^{(F)}_{\gamma}(y, x) = \int_{0}^{1}g\left(\gamma(t),\dot{\gamma}(t)\right)\,\dif t$ denotes the energy of the curve $\gamma$ with $\gamma(0)=y$ and $\gamma(1)=x$. In a local coordinate system, the fundamental tensor can be written as $g\left(\gamma(t),\dot{\gamma}(t)\right) = \dot{\gamma}(t)^{\top}G\left(\gamma(t), \dot{\gamma}(t)\right)\dot{\gamma}(t)$, and therefore the discretized version Eq.~\ref{eq:finsler_frechet_energy} becomes
\begin{equation} \label{eq:finsler_energy_control}
    \begin{split}
        &\min_{(x_{t,i},u_{t,i})}\left\{\sum_{i=1}^{N}w_{i}\sum_{t=0}^{T-1}u_{t,i}^{\top}G(x_{t,i}, u_{t,i})u_{t,i}\right\} \\
        &x_{t+1,i} = x_{t,i}+u_{t,i}, \quad t=0,\dots,T-1, \, i=1,\dots,N, \\
        &x_{0,i}=y,\, x_{T,i}=a_{i}, \quad i=1,\dots,N,
    \end{split}
\end{equation}
The only difference compared to the Riemannian case is that $G$ now also depends on the velocity and position. However, we cannot directly apply the same approach as in the Riemannian case. In the Riemannian case we had implicitly changed the order of start and end point such that the Fr\'echet mean was the end point of the connecting geodesics. Unlike Riemannian manifolds, the distance is not symmetric on Finsler manifolds and we can therefore not apply a similar ``trick'' in the Finslerian case. To circumvent this, consider the following transformations.
\begin{equation*}
    \begin{split}
        u_{t,i} &= -\tilde{u}_{T-t-1,i}, \quad t=0,\dots,T-1, \, i=1,\dots,N, \\
        x_{t,i} &= \tilde{x}_{T-t,i}, \quad t=0,\dots,T-1, \, i=1,\dots,N. \\
    \end{split}
\end{equation*}
With the change of variable, Eq.~\ref{eq:finsler_energy_control} can be formulated as
\begin{equation*}
    \begin{split}
        &\min_{(\tilde{x}_{t,i},\tilde{u}_{t,i})}\left\{\sum_{i=1}^{N}w_{i}\sum_{t=0}^{T-1}\tilde{u}_{T-t-1,i}^{\top}G\left(\tilde{x}_{T-t,i}, -\tilde{u}_{T-t-1,i}\right)\tilde{u}_{T-t-1,i}\right\} \\
        &\tilde{x}_{T-t-1,i} = \tilde{x}_{T-t,i}-\tilde{u}_{T-t-1,i}, \quad t=0,\dots,T-1, \, i=1,\dots,N, \\
        &\tilde{x}_{0,i}=a_{i},\, \tilde{x}_{T,i}=y, \quad i=1,\dots,N.
    \end{split}
\end{equation*}
Consider the change of index by $s:=T-1-t$ for $t=0,\dots,T-1$. When $T$ is sufficiently large, $G\left(\tilde{x}_{T-t,i}, \tilde{u}_{T-1-1, i}\right)$ can be approximated by $G\left(\tilde{x}_{T-t-1,i}, \tilde{u}_{T-1-1, i}\right)$. With this modification, we get the following.
\begin{equation} \label{eq:finsler_reverse_energy_control}
    \begin{split}
        \min_{(\tilde{x}_{s,i},\tilde{u}_{s,i})} E^{(F)}(x) &:= \min_{(x_{t,i},u_{t,i})}\left\{\sum_{i=1}^{N}w_{i}\sum_{s=0}^{T-1}\tilde{u}_{s,i}^{\top}\tilde{G}\left(\tilde{x}_{s,i},\tilde{u}_{s,i}\right)\tilde{u}_{s,i}\right\} \\
        \tilde{x}_{s+1,i} &= \tilde{x}_{s,i}+\tilde{u}_{s,i}, \quad s=0,\dots,T-1, \, i=1,\dots,N, \\
        \tilde{x}_{0,i}&=a_{i},\, \tilde{x}_{T,i}=y, \quad i=1,\dots,N,
    \end{split}
\end{equation}
where $\tilde{G}\left(x, u\right) := G\left(x, -u\right)$. In Eq.~\ref{eq:finsler_reverse_energy_control} the Fr\'echet mean can be treated as the end point. By following a completely similar approach as in the Riemannian case with some small modifications, we get the following iterative scheme to estimate the geodesics and Fr\'echet mean for Finsler manifolds.
\begin{proposition} \label{prop:finsler_update_scheme}
    The update scheme for $u_{t},\mu_{t}$ and $x_{t}$ is
    \begin{equation} \label{eq:finsler_energy_update_schem}
        \begin{split}
            &y = W^{-1}V, \\
            &\mu_{T-1,i} = \left(\sum_{t=0}^{T-1}\tilde{G}_{t,i}^{-1}\right)^{-1}\left(2w_{i}(a_{i}-y)-\sum_{t=0}^{T-1}\tilde{G}_{t,i}^{-1}\left(\zeta_{t,i}+\sum_{t>j}^{T-1}\nu_{j,i}\right)\right), \quad i=1,\dots,N, \\
            &u_{t,i} = -\frac{1}{2w_{i}}\tilde{G}_{t,i}^{-1}\left(\mu_{T-1,i}+\zeta_{t,i}+\sum_{j>t}^{T-1}\nu_{j,i}\right), \quad t=0,\dots,T-1, \, i=1,\dots,N \\
            &x_{t+1,i} = x_{t,i}+u_{t,i}, \quad t=0,\dots,T-2, \, i=1,\dots,N, \\
            &x_{0,i}=a_{i} \quad i=1,\dots,N,
        \end{split}
    \end{equation}
    where
    \begin{equation*}
        \begin{split}
            W &= \sum_{i=1}^{N}w_{i}\left(\sum_{t=0}^{T-1}\tilde{G}_{t,i}^{-1}\right)^{-1}, \\
            V &= \sum_{i=1}^{N}w_{i}\left(\sum_{t=0}^{T-1}\tilde{G}_{ti}^{-1}\right)^{-1}a_{i}-\frac{1}{2}\sum_{i=1}^{N}\left(\sum_{t=0}^{T-1}\tilde{G}_{t,i}^{-1}\right)^{-1}\sum_{t=0}^{T-1}\tilde{G_{ti}}^{-1}\left(\zeta_{t,i}+\sum_{j>t}^{T-1}\nu_{j,i}\right).
        \end{split}
    \end{equation*}
    Here $\nu_{t,i} := \restr{\nabla_{y}u_{t,i}^{\top}\tilde{G}(y,u_{t,i})u_{t,i}}{y=x_{t,i}}$ and $\zeta_{t,i} := \restr{\nabla_{v}u_{t,i}^{\top}\tilde{G}(x_{t,i},v)u_{t,i}}{v=u_{t,i}}$.
\end{proposition}
\begin{proof}
    See Appendix~\ref{ap:finsler_frechet}.
\end{proof}
\begin{figure}[t!]
    \centering
    \includegraphics[width=1.0\textwidth]{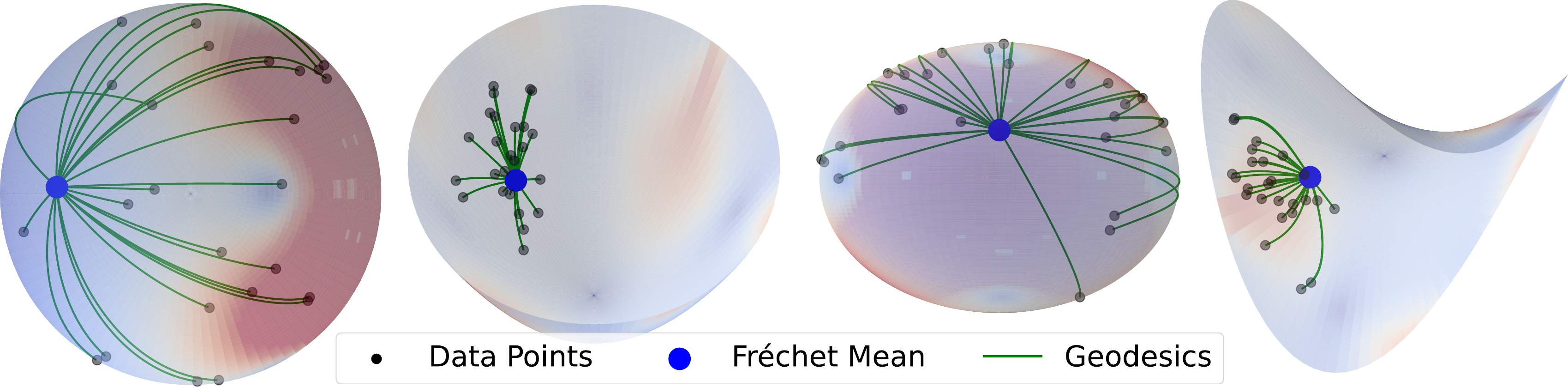}
    \caption{The application of \textit{GEORCE-FM} to Finsler manifolds, where we show the estimated Fr\'echet mean and geodesics using \textit{GEORCE-FM} for (from left to right) a sphere, paraboloid, ellipsoid and hyperbolic paraboloid all equipped with a wind field. The details of the manifolds and data points can be found in Appendix~\ref{ap:manifold_description}.}
    \label{fig:finsler_frechet}
    \vspace{-1.em}
\end{figure}
In Fig.~\ref{fig:finsler_frechet} we show the estimated Fr\'echet mean as well as the estimated connecting geodesics for four different Finsler manifolds. We also see that the adaptive extension of \textit{GEORCE-FM} in algorithm~\ref{al:adaptive_georce_fm} can be directly extended to Finsler manifolds, where $\tilde{V}$ and $\tilde{W}$ are given by
\begin{equation} \label{eq:w_v_adaptive_finlser}
    \begin{split}
        \tilde{W} &= \sum_{i \in \mathcal{I}}w_{i}\left(\sum_{t=0}^{T-1}G_{t,i}^{-1}\right)^{-1}, \\
        \tilde{V} &= \sum_{i \in \mathcal{I}}w_{i}\left(\sum_{t=0}^{T-1}G_{t,i}^{-1}\right)^{-1}a_{i}-\frac{1}{2}\sum_{i \in \mathcal{I}}\left(\sum_{t=0}^{T-1}G_{t,i}^{-1}\right)^{-1}\sum_{t=0}^{T-1}G_{t,i}^{-1}\left(\zeta_{t,i}+\sum_{j>t}^{T-1}g_{j,i}\right)
    \end{split}
\end{equation}
In the above, we have shown an iterative scheme to estimate the Fr\'echet mean as a starting point on a Finsler manifold. The Fr\'echet mean can also trivially be found as an end point by not applying the transformation of the variables $x$ and $u$.

%% file: Sections/experiments.tex
\begin{figure}[t!]
    \centering
    \includegraphics[width=1.0\textwidth]{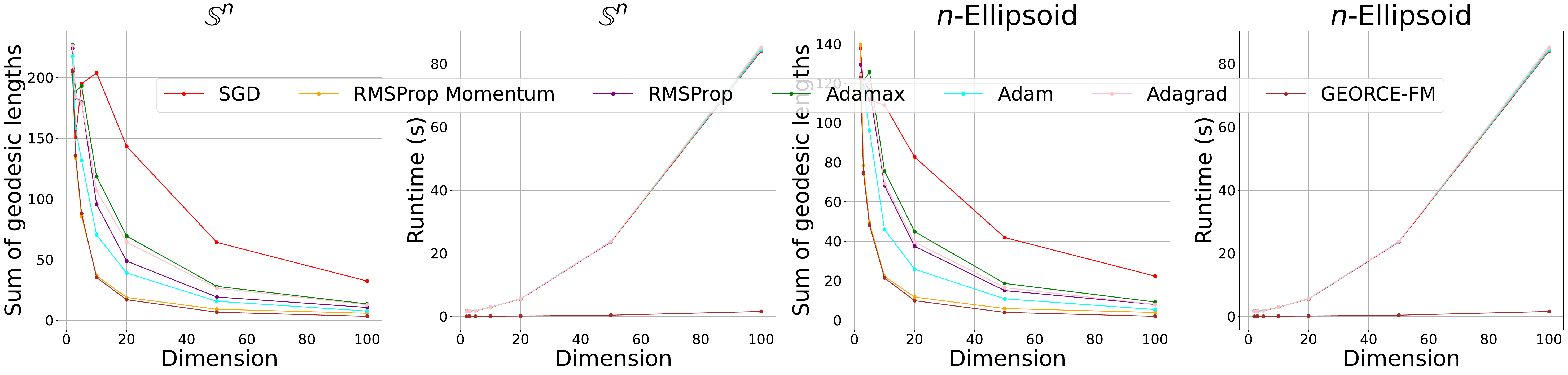}
    \caption{Estimated squared geodesic length and runtime for the sphere and ellipsoid for different dimensions and different optimization schemes.}
    \label{fig:sphere_ellipsoid_runtime}
    \vspace{-1.em}
\end{figure}

\paragraph{Riemannian.} To compare \textit{GEORCE-FM} to alternative similar methods, we use different optimizers for minimizing Eq.~\ref{eq:frechet_energy_disc} directly. Since multiple geodesics have to be computed, we use methods that are designed to handle large number of parameters. \cite{georce} observed that methods such as \textit{BFGS} and trust-region methods were slow, and therefore we will instead apply stochastic methods that are designed to handle a large number of parameters. Table~\ref{tab:riemmannian_comparison_table} shows the runtime estimates for \textit{ADAM} \citep{kingma2017adam}, \textit{RMSprop Momentum} \citep{ruder2017overviewgradientdescentoptimization} and \textit{GEORCE-FM}. To estimate the solution, we use \textit{GEORCE} \citep{georce} with $10$ iterations to estimate the geodesics between the data points and the estimated Fr\'echet mean to compute the sum of squared geodesics distances in Eq.~\ref{eq:frechet_mean}. We summarize the results for the $n$-sphere and $n$-ellipsoids in Fig.~\ref{fig:sphere_ellipsoid_runtime}. We see that \textit{GEORCE-FM} is significantly faster than the alternative methods and generally achieves a better solution. For details on the manifolds and data, we refer to Appendix~\ref{ap:manifold_description} and Appendix~\ref{ap:data_methods}, respectively. We also provide additional experiments in Appendix~\ref{ap:additional_experiments} for other methods.
\begin{sidewaystable}
    \centering
    \scriptsize
    \begin{tabular*}{\textheight}{@{\extracolsep\fill}lcccccc}
        \toprule%
        & \multicolumn{2}{c}{\textbf{ADAM (T=100)}} & \multicolumn{2}{c}{\textbf{RMSprop Momentum  (T=100)}} & \multicolumn{2}{c}{\textbf{GEORCE-FM  (T=100)}} \\\cmidrule{2-3}\cmidrule{4-5}\cmidrule{6-7}%
        Finsler Manifold & Length & Runtime & Length & Runtime & Length & Runtime \\
        \midrule
        $\mathbb{S}^{2}$ & $217.60$ & $1.6748 \pm 0.0010$ & $\pmb{202.28}$ & $1.6695 \pm 0.0013$ & $204.99$ & $\pmb{0.0611} \pm \pmb{ 0.0017 }$ \\ 
        $\mathbb{S}^{3}$ & $158.24$ & $1.7086 \pm 0.0007$ & $\pmb{133.92}$ & $1.7062 \pm 0.0012$ & $135.98$ & $\pmb{0.0534} \pm \pmb{ 0.0002 }$ \\ 
        $\mathbb{S}^{5}$ & $131.77$ & $1.8670 \pm 0.0009$ & $\pmb{85.48}$ & $1.8372 \pm 0.0016$ & $88.01$ & $\pmb{0.0537} \pm \pmb{ 0.0020 }$ \\ 
        $\mathbb{S}^{10}$ & $70.37$ & $2.9435 \pm 0.0015$ & $36.79$ & $2.9340 \pm 0.0011$ & $\pmb{35.20}$ & $\pmb{0.0808} \pm \pmb{ 0.0002 }$ \\ 
        $\mathbb{S}^{20}$ & $39.04$ & $5.5679 \pm 0.0009$ & $18.83$ & $5.5143 \pm 0.0010$ & $\pmb{16.97}$ & $\pmb{0.1250} \pm \pmb{ 0.0002 }$ \\ 
        $\mathbb{S}^{50}$ & $15.70$ & $25.1048 \pm 0.6575$ & $9.14$ & $24.1128 \pm 0.3806$ & $\pmb{6.67}$ & $\pmb{0.3963} \pm \pmb{ 0.0002 }$ \\ 
        $\mathbb{S}^{100}$ & $7.64$ & $85.3785 \pm 0.4844$ & $5.80$ & $84.8019 \pm 0.0659$ & $\pmb{3.27}$ & $\pmb{1.5691} \pm \pmb{ 0.0003 }$ \\ 
        \hline
        $\mathrm{E}\left( 2 \right)$ & $123.22$ & $1.6598 \pm 0.0020$ & $139.62$ & $1.6661 \pm 0.0008$ & $\pmb{122.68}$ & $\pmb{0.0357} \pm \pmb{ 0.0002 }$ \\ 
        $\mathrm{E}\left( 3 \right)$ & $117.67$ & $1.7159 \pm 0.0008$ & $78.37$ & $1.7002 \pm 0.0008$ & $\pmb{74.58}$ & $\pmb{0.0653} \pm \pmb{ 0.0022 }$ \\ 
        $\mathrm{E}\left( 5 \right)$ & $96.16$ & $1.8809 \pm 0.0010$ & $49.45$ & $1.8445 \pm 0.0017$ & $\pmb{48.14}$ & $\pmb{0.0529} \pm \pmb{ 0.0017 }$ \\ 
        $\mathrm{E}\left( 10 \right)$ & $45.91$ & $2.9558 \pm 0.0017$ & $22.22$ & $2.9349 \pm 0.0012$ & $\pmb{21.42}$ & $\pmb{0.0715} \pm \pmb{ 0.0002 }$ \\ 
        $\mathrm{E}\left( 20 \right)$ & $25.82$ & $5.6296 \pm 0.0029$ & $11.68$ & $5.5802 \pm 0.0024$ & $\pmb{9.89}$ & $\pmb{0.1300} \pm \pmb{ 0.0002 }$ \\ 
        $\mathrm{E}\left( 50 \right)$ & $10.82$ & $23.7450 \pm 0.0019$ & $5.91$ & $23.6992 \pm 0.0069$ & $\pmb{3.94}$ & $\pmb{0.3960} \pm \pmb{ 0.0001 }$ \\ 
        $\mathrm{E}\left( 100 \right)$ & $5.39$ & $84.4791 \pm 0.1039$ & $3.90$ & $84.3799 \pm 0.0973$ & $\pmb{1.94}$ & $\pmb{1.5685} \pm \pmb{ 0.0002 }$ \\ 
        \hline
        $\mathbb{T}^{2}$ & $1234.12$ & $1.6483 \pm 0.0011$ & $1227.92$ & $1.6429 \pm 0.0015$ & $\pmb{1225.65}$ & $\pmb{0.0743} \pm \pmb{ 0.0028 }$ \\ 
        \hline
        $\mathbb{H}^{2}$ & $\pmb{174.60}$ & $1.6616 \pm 0.0017$ & $174.61$ & $1.6564 \pm 0.0025$ & $174.60$ & $\pmb{0.1594} \pm \pmb{ 0.0015 }$ \\ 
        \hline
        Paraboloid & $\pmb{103.95}$ & $1.6080 \pm 0.0012$ & $123.29$ & $1.6086 \pm 0.0013$ & $103.95$ & $\pmb{0.0182} \pm \pmb{ 0.0002 }$ \\ 
        Hyperbolic Paraboloid & $\pmb{114.39}$ & $1.6085 \pm 0.0012$ & $129.02$ & $1.6069 \pm 0.0012$ & $114.40$ & $\pmb{0.0082} \pm \pmb{ 0.0001 }$ \\ 
        \hline
        Gaussian Distribution & $148.97$ & $0.2134 \pm 0.0028$ & $154.36$ & $0.2120 \pm 0.0126$ & $\pmb{148.52}$ & $\pmb{0.0154} \pm \pmb{ 0.0017 }$ \\ 
        Fr\'echet Distribution & $31.22$ & $0.2060 \pm 0.0123$ & $36.74$ & $0.2116 \pm 0.0136$ & $\pmb{31.15}$ & $\pmb{0.0065} \pm \pmb{ 0.0002 }$ \\ 
        Cauchy Distribution & $60.60$ & $0.2388 \pm 0.0003$ & $-/-$ & $-$ & $\pmb{60.27}$ & $\pmb{0.0138} \pm \pmb{ 0.0002 }$ \\ 
        Pareto Distribution & $24.03$ & $0.2054 \pm 0.0083$ & $1690.17$ & $0.2002 \pm 0.0009$ & $\pmb{24.02}$ & $\pmb{0.0062} \pm \pmb{ 0.0002 }$ \\ 
        \hline
        VAE MNIST & $858.06$ & $39.0780 \pm 0.0026$ & $813.19$ & $39.2018 \pm 0.0039$ & $\pmb{771.59}$ & $\pmb{7.1326} \pm \pmb{ 0.0027 }$ \\ 
        VAE CelebA & $9420.09$ & $832.8971 \pm 0.0761$ & $43617996.00$ & $862.4910 \pm 0.2627$ & $\pmb{9414.38}$ & $\pmb{79.2793} \pm \pmb{ 0.0061 }$ \\  
        \bottomrule
    \end{tabular*}
    \caption{The table shows the sum of squared geodesic length for the estimated Fr\'echet mean for \textit{ADAM}, \textit{RMSprop Momentum} and \textit{GEORCE-FM} on a GPU for Riemannian manifolds. The methods were terminated if the $\ell^{2}$-norm of the average gradient across the data was less than $10^{-4}$ or $1,000$ iterations have been reached. $\mathrm{E}(n)$ denotes an Ellipsoid of dimension $n$, while $\mathcal{P}(n)$ denotes the space of $n \times n$ symmetric positive definite matrices. When the computational time was longer than $24$ hours, or if the method returns $\mathrm{nan}$, the value is set to $-$. Since the \textsc{vae}'s are learned manifolds based on data, and therefore can be unstable, the methods are terminated if the average gradient across the data was less than $10^{-3}$. Further, to avoid memory complexity for the learned manifolds in storing the neural network, the geodesics are updated sequentially for VAE MNIST and VAE CelebA, while for the other manifolds the geodesics are updated in parallel. The \textsc{vae}'s use $10$ datapoints, while the remaining manifolds use $100$ synthetically generated datapoints. The data is described in Appendix~\ref{ap:data_methods}.}
    \label{tab:riemmannian_comparison_table}
    \vspace{-2.5em}
\end{sidewaystable}
\begin{figure}[h!]
    \centering
    \includegraphics[width=1.0\textwidth]{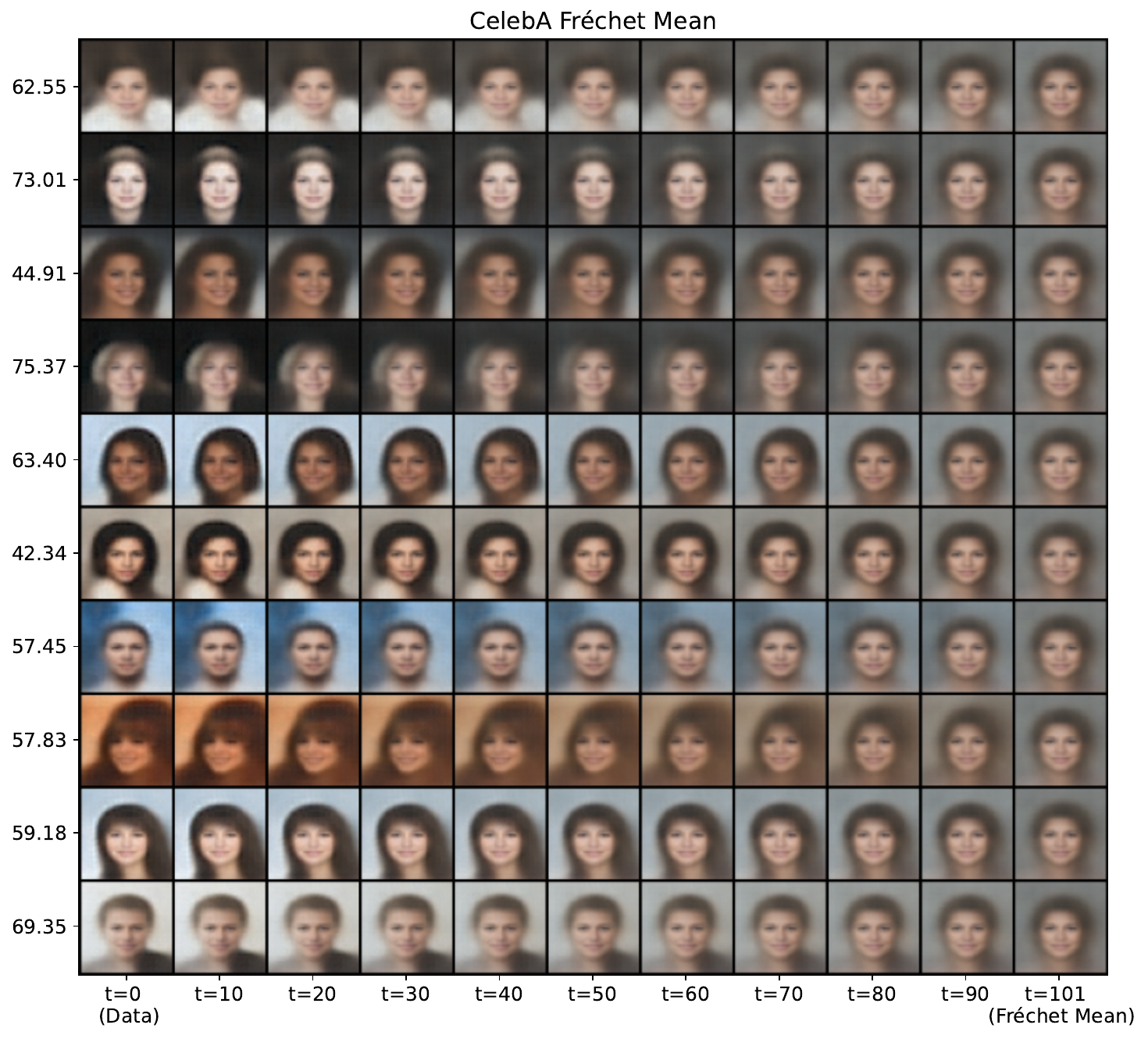}
    \caption{The application of \textit{GEORCE-FM} to manifold learning for a \textsc{vae} estimating the Fr\'echet mean for $10$ images of the CelebA dataset \citep{liu2015faceattributes} reconstructed using a \textsc{vae} similar to \citep{shao2017riemannian} Each row shows $10$ points on the geodesic, while the left most image is the data and the right most image is the estimated Fr\'echet mean using \textit{GEORCE-FM}. The length of the estimated geodesic using \textit{GEORCE-FM} is shown to the left, while the grid point number is shown at the bottom.}
    \label{fig:celeba_riemannian_32}
    \vspace{-1.em}
\end{figure}
To further show the application of \textit{GEORCE-FM} to real-world data, we consider a manifold learned using a Variational-Autoencoder (\textsc{vae}) equipped with the pull-back metric \citep{shao2017riemannian, arvanitidis2021latent}
\begin{equation*}
    G(z) = J_{f_{\theta}}(z)^{\top}J_{f_{\theta}}(z),
\end{equation*}
where $J_{f_{\theta}}$ is the Jacobian of the decoder, $f_{\theta}$, which is a neural network with parameters $\theta$. We provide additional details on the \textsc{vae} in Appendix~\ref{ap:manifold_description}. In Fig.~\ref{fig:celeba_riemannian_32} we show the estimated geodesics and Fr\'echet mean using $T=100$ grid points for $10$ images reconstructed using the \textsc{vae} corresponding to the data used in Table~\ref{tab:riemmannian_comparison_table}.
\paragraph{Finslerian.} To illustrate our method across different Finsler manifolds, we consider a Riemannian background metric, which is affected by a wind field similar to the construction in \cite{Piro_2021}. We consider the same generic wind field as in \citep{georce}
\begin{equation}
    f(x) = \frac{\sin x \odot \cos x}{(\cos x)^{\top} G(x) \cos x},
    \label{eq:generic_forcefield}
\end{equation}
We provide additional details on the construction in Appendix~\ref{ap:manifold_description}. Table~\ref{tab:finsler_comparison_table} shows the runtime and sum of squared geodesic length for different Finsler manifolds constructed using the wind field in Eq.~\ref{eq:generic_forcefield} on a Riemannian background metric. The sum of squared geodesic lengths is computed using \textit{GEORCE} \citep{georce} between the data points and the estimated Fr\'echet mean for the different methods. We see that \textit{GEORCE-FM} generally is faster and achieves a better estimate of the Fr\'echet mean. Additional experiments can be found in Appendix~\ref{ap:additional_experiments}.
\begin{sidewaystable}
    \centering
    \scriptsize
    \begin{tabular*}{\textheight}{@{\extracolsep\fill}lcccccc}
        \toprule%
        & \multicolumn{2}{c}{\textbf{ADAM (T=100)}} & \multicolumn{2}{c}{\textbf{RMSprop Momentum  (T=100)}} & \multicolumn{2}{c}{\textbf{GEORCE-FM  (T=100)}} \\\cmidrule{2-3}\cmidrule{4-5}\cmidrule{6-7}%
        Finsler Manifold & Length & Runtime & Length & Runtime & Length & Runtime \\
        \midrule
        $\mathbb{S}^{2}$ & $83.64$ & $2.3388 \pm 0.0020$ & $85.91$ & $2.3085 \pm 0.0029$ & $\pmb{83.63}$ & $\pmb{0.0230} \pm \pmb{ 0.0002 }$ \\ 
        $\mathbb{S}^{3}$ & $76.33$ & $2.4284 \pm 0.0013$ & $75.97$ & $2.4215 \pm 0.0016$ & $\pmb{75.16}$ & $\pmb{0.0330} \pm \pmb{ 0.0001 }$ \\ 
        $\mathbb{S}^{5}$ & $70.40$ & $3.1670 \pm 0.0022$ & $46.81$ & $3.1471 \pm 0.0021$ & $\pmb{46.79}$ & $\pmb{0.0763} \pm \pmb{ 0.0028 }$ \\ 
        $\mathbb{S}^{10}$ & $37.68$ & $5.5112 \pm 0.0047$ & $19.95$ & $5.4865 \pm 0.0047$ & $\pmb{18.99}$ & $\pmb{0.0938} \pm \pmb{ 0.0002 }$ \\ 
        $\mathbb{S}^{20}$ & $20.44$ & $12.4084 \pm 0.0008$ & $10.45$ & $12.3972 \pm 0.0021$ & $\pmb{9.44}$ & $\pmb{0.2542} \pm \pmb{ 0.0001 }$ \\ 
        $\mathbb{S}^{50}$ & $8.82$ & $55.1251 \pm 0.0010$ & $8.31$ & $55.0981 \pm 0.0017$ & $\pmb{4.32}$ & $\pmb{0.7822} \pm \pmb{ 0.0002 }$ \\ 
        $\mathbb{S}^{100}$ & $7.30$ & $180.1550 \pm 0.0409$ & $6.57$ & $180.1470 \pm 0.0476$ & $\pmb{2.95}$ & $\pmb{1.2838} \pm \pmb{ 0.0002 }$ \\ 
        \hline
        $\mathrm{E}\left( 2 \right)$ & $45.17$ & $2.3167 \pm 0.0006$ & $56.48$ & $2.3075 \pm 0.0002$ & $\pmb{44.75}$ & $\pmb{0.0508} \pm \pmb{ 0.0002 }$ \\ 
        $\mathrm{E}\left( 3 \right)$ & $\pmb{48.35}$ & $2.4337 \pm 0.0007$ & $50.36$ & $2.4048 \pm 0.0020$ & $48.68$ & $\pmb{0.0403} \pm \pmb{ 0.0003 }$ \\ 
        $\mathrm{E}\left( 5 \right)$ & $48.36$ & $3.1432 \pm 0.0015$ & $27.05$ & $3.1245 \pm 0.0025$ & $\pmb{24.85}$ & $\pmb{0.0530} \pm \pmb{ 0.0003 }$ \\ 
        $\mathrm{E}\left( 10 \right)$ & $23.57$ & $5.1367 \pm 0.0019$ & $11.86$ & $5.0912 \pm 0.0015$ & $\pmb{10.86}$ & $\pmb{0.0617} \pm \pmb{ 0.0002 }$ \\ 
        $\mathrm{E}\left( 20 \right)$ & $13.59$ & $12.5121 \pm 0.0026$ & $7.85$ & $12.4686 \pm 0.0023$ & $\pmb{5.60}$ & $\pmb{0.2530} \pm \pmb{ 0.0002 }$ \\ 
        $\mathrm{E}\left( 50 \right)$ & $6.46$ & $54.7758 \pm 0.0019$ & $6.87$ & $54.7446 \pm 0.0015$ & $\pmb{2.52}$ & $\pmb{0.7813} \pm \pmb{ 0.0001 }$ \\ 
        $\mathrm{E}\left( 100 \right)$ & $5.29$ & $177.0563 \pm 0.1030$ & $4.87$ & $177.0098 \pm 0.0526$ & $\pmb{1.95}$ & $\pmb{1.2813} \pm \pmb{ 0.0002 }$ \\ 
        \hline
        $\mathbb{T}^{2}$ & $\pmb{437.49}$ & $2.4922 \pm 0.0021$ & $467.28$ & $2.4912 \pm 0.0021$ & $467.19$ & $\pmb{0.1128} \pm \pmb{ 0.0008 }$ \\ 
        \hline
        $\mathbb{H}^{2}$ & $74.70$ & $2.5148 \pm 0.0015$ & $\pmb{74.69}$ & $2.5139 \pm 0.0022$ & $74.72$ & $\pmb{0.0831} \pm \pmb{ 0.0001 }$ \\ 
        \hline
        Paraboloid & $58.80$ & $2.2218 \pm 0.0008$ & $56.74$ & $2.2359 \pm 0.0009$ & $\pmb{56.68}$ & $\pmb{0.1344} \pm \pmb{ 0.0001 }$ \\ 
        Hyperbolic Paraboloid & $54.21$ & $2.2172 \pm 0.0010$ & $58.32$ & $2.2114 \pm 0.0010$ & $\pmb{46.69}$ & $\pmb{0.0730} \pm \pmb{ 0.0004 }$ \\ 
        \hline
        Gaussian Distribution & $62.73$ & $0.7512 \pm 0.0021$ & $75.61$ & $0.7504 \pm 0.0016$ & $\pmb{62.43}$ & $\pmb{0.0285} \pm \pmb{ 0.0006 }$ \\ 
        Fr\'echet Distribution & $11.58$ & $0.6400 \pm 0.0021$ & $50.27$ & $0.6389 \pm 0.0056$ & $\pmb{11.15}$ & $\pmb{0.0120} \pm \pmb{ 0.0004 }$ \\ 
        Cauchy Distribution & $24.96$ & $0.7524 \pm 0.0008$ & $26.46$ & $0.7438 \pm 0.0079$ & $\pmb{24.91}$ & $\pmb{0.0167} \pm \pmb{ 0.0004 }$ \\ 
        \hline
        Pareto Distribution & $11.63$ & $0.6514 \pm 0.0004$ & $31.71$ & $0.6544 \pm 0.0026$ & $\pmb{11.03}$ & $\pmb{0.0134} \pm \pmb{ 0.0004 }$ \\ 
        \hline
        VAE MNIST & $304.26$ & $463.7502 \pm 0.0112$ & $296.61$ & $463.6756 \pm 0.4144$ & $\pmb{287.44}$ & $\pmb{18.3540} \pm \pmb{ 0.0084 }$ \\ 
        VAE CelebA & $3173.79$ & $1596.7343 \pm 0.0697$ & $3227.14$ & $1445.5969 \pm 106.3969$ & $\pmb{3170.45}$ & $\pmb{40.1097} \pm \pmb{ 0.0064 }$ \\ 
        \bottomrule
    \end{tabular*}
    \caption{The table shows the sum of squared geodesic length for the estimated Fr\'echet mean for \textit{ADAM}, \textit{RMSprop Momentum} and \textit{GEORCE-FM} on a GPU for Finsler manifolds constructed with a Riemannian background metric equipped with the wind field in Eq.~\ref{eq:generic_forcefield}. The methods were terminated if the $\ell^{2}$-norm of the average gradient across the data was less than $10^{-4}$ or $1,000$ iterations have been reached. $\mathrm{E}(n)$ denotes an Ellipsoid of dimension $n$, while $\mathcal{P}(n)$ denotes the space of $n \times n$ symmetric positive definite matrices. When the computational time was longer than $24$ hours, or if the method returns $\mathrm{nan}$, the value is set to $-$. Since the \textsc{vae}'s are learned manifolds based on data, and therefore can be unstable, the methods are terminated if the average gradient across the data was less than $10^{-3}$. Further, to avoid memory complexity for the learned manifolds in storing the neural network, the geodesics are updated sequentially for VAE MNIST and VAE CelebA, while for the other manifolds the geodesics are updated in parallel. The \textsc{vae}'s use $10$ datapoints, while the remaining manifolds use $100$ synthetically generated datapoints. The data is described in Appendix~\ref{ap:data_methods}.}
    \label{tab:finsler_comparison_table}
    \vspace{-2.5em}
\end{sidewaystable}
\paragraph{Adaptive Estimation} To compare the adaptive estimation of the Fr\'echet mean with alternative methods, we consider alternative stochastic methods estimating the Fr\'echet using mini-batches of data as in Eq.~\ref{eq:frechet_energy_disc_stoch}. For alternative methods, we estimate stochastically the Fr\'echet mean by minimizing
\begin{equation} \label{eq:stoch_disc_frechet}
    \min_{\left(x_{t,i}, y\right)} \sum_{i \in \mathcal{I}}w_{i}\sum_{t=0}^{T-1} \left(x_{t,i+1}-x_{t,i}\right)^{\top}G(x_{t,i})\left(x_{t,i+1}-x_{t,i}\right), \quad x_{0,i}= a_{i}, x_{T,i}=y,
\end{equation}
where $\mathcal{I} = \{i_{1},i_{2},\dots,i_{n}\}$ is an index set consisting of $n$ distinct random integers $i_{1} < i_{2} < \dots <i_{n}$. Thus, for standard stochastic optimization solvers, we estimate only mini-batches of the geodesics as well as the Fr\'echet mean in each iteration.

We show the result in Table~\ref{tab:riemannian_adaptive}, where we consider datasets of $1,000$ data points and use $10\%$ of the data in the estimation. For the adaptive estimation of \textit{GEORCE-FM} we use 5 sub-iterations and terminate the algorithm if the change in the estimated Fr\'echet mean is less than $10^{-4}$. For the other methods, we terminate the algorithm if the mean stochastic gradient of Eq.~\ref{eq:stoch_disc_frechet} over the number of data points is less than $10^{-4}$. Since it is difficult to have similar stopping criteria for \textit{GEORCE-FM} and the alternative methods, we do not report the runtime in Table~\ref{tab:riemannian_adaptive}, but only the sum of the squared geodesic length for all data using \textit{GEORCE} \citep{georce} for the estimated Fr\'echet mean for each algorithm. In general, we see that the adaptive extension of \textit{GEORCE-FM} is more accurate than the alternative methods.

\begin{sidewaystable}
    \centering
    \scriptsize
    \begin{tabular*}{\textheight}{@{\extracolsep\fill}lccc}
        \hline
        & \multicolumn{3}{c}{\textbf{Adaptive Estimation of the Fr\'echet Mean for Riemannian Manifolds}} \\
        \cmidrule{1-4}
        & \multicolumn{3}{c|}{\textbf{Batch 10\%}} \\
        \cmidrule{1-4}
        \textbf{Manifold} & ADAM & RMSprop Momentum &
        GEORCE-FM \\
        \hline
        $\mathbb{S}^{2}$ & $2370.92$ & $\pmb{2033.23}$ & $2179.35$ \\ 
        $\mathbb{S}^{3}$ & $2339.52$ & $\pmb{1378.82}$ & $1430.56$ \\ 
        $\mathbb{S}^{5}$ & $2025.96$ & $851.12$ & $\pmb{821.47}$ \\ 
        $\mathbb{S}^{10}$ & $1277.13$ & $394.19$ & $\pmb{374.00}$ \\ 
        $\mathbb{S}^{20}$ & $810.52$ & $188.29$ & $\pmb{173.39}$ \\ 
        $\mathbb{S}^{50}$ & $378.83$ & $83.76$ & $\pmb{68.65}$ \\ 
        $\mathbb{S}^{100}$ & $192.71$ & $51.20$ & $\pmb{33.50}$ \\ 
        \hline
        $\mathrm{E}\left( 2 \right)$ & $1364.47$ & $1482.42$ & $\pmb{1308.92}$ \\ 
        $\mathrm{E}\left( 3 \right)$ & $1446.98$ & $\pmb{846.17}$ & $1163.30$ \\ 
        $\mathrm{E}\left( 5 \right)$ & $1333.88$ & $504.87$ & $\pmb{439.33}$ \\ 
        $\mathrm{E}\left( 10 \right)$ & $846.91$ & $237.06$ & $\pmb{211.01}$ \\ 
        $\mathrm{E}\left( 20 \right)$ & $556.93$ & $116.52$ & $\pmb{101.55}$ \\ 
        $\mathrm{E}\left( 50 \right)$ & $267.66$ & $52.28$ & $\pmb{40.59}$ \\ 
        $\mathrm{E}\left( 100 \right)$ & $-/-$ & $-/-$ & $\pmb{19.83}$ \\ 
        \hline
        $\mathbb{T}^{2}$ & $14986.02$ & $\pmb{14839.74}$ & $14953.47$ \\ 
        \hline
        $\mathbb{H}^{2}$ & $\pmb{1255.85}$ & $1256.54$ & $1256.70$ \\ 
        \hline
        Paraboloid & $\pmb{1040.85}$ & $1142.48$ & $1041.53$ \\ 
        Hyperbolic Paraboloid & $\pmb{1109.81}$ & $1117.19$ & $1110.46$ \\ 
        \hline
        Gaussian Distribution & $1509.58$ & $1667.65$ & $\pmb{1437.59}$ \\ 
        Fr\'echet Distribution & $308.39$ & $26604758.00$ & $\pmb{306.09}$ \\ 
        Cauchy Distribution & $648.87$ & $645.29$ & $\pmb{585.35}$ \\ 
        Pareto Distribution & $273.26$ & $798.23$ & $\pmb{269.40}$ \\ 
        \hline
        VAE MNIST & $7886.67$ & $6890.60$ & $\pmb{6462.85}$ \\ 
        VAE CelebA & $-/-$ & $-/-$ & $\pmb{80255.43}$ \\ 
        \hline
    \end{tabular*}
    \caption{The table shows the sum of squared geodesic length for the estimated Fr\'echet mean for \textit{ADAM}, \textit{RMSprop Momentum} and \textit{GEORCE-FM} on a GPU using only mini-batches of data. $\mathrm{E}(n)$ denotes an Ellipsoid of dimension $n$, while $\mathcal{P}(n)$ denotes the space of $n \times n$ symmetric positive definite matrices. When the computational time was longer than $24$ hours, or if the method returns $\mathrm{nan}$, the value is set to $-$. The \textsc{vae}'s use $100$ datapoints, while the remaining manifolds use $1,000$ synthetically generated datapoints. The data is described in Appendix~\ref{ap:data_methods}.}
    \label{tab:riemannian_adaptive}
    \vspace{-2.5em}
\end{sidewaystable}

\begin{figure}[h!]
    \centering
    \includegraphics[width=1.0\textwidth]{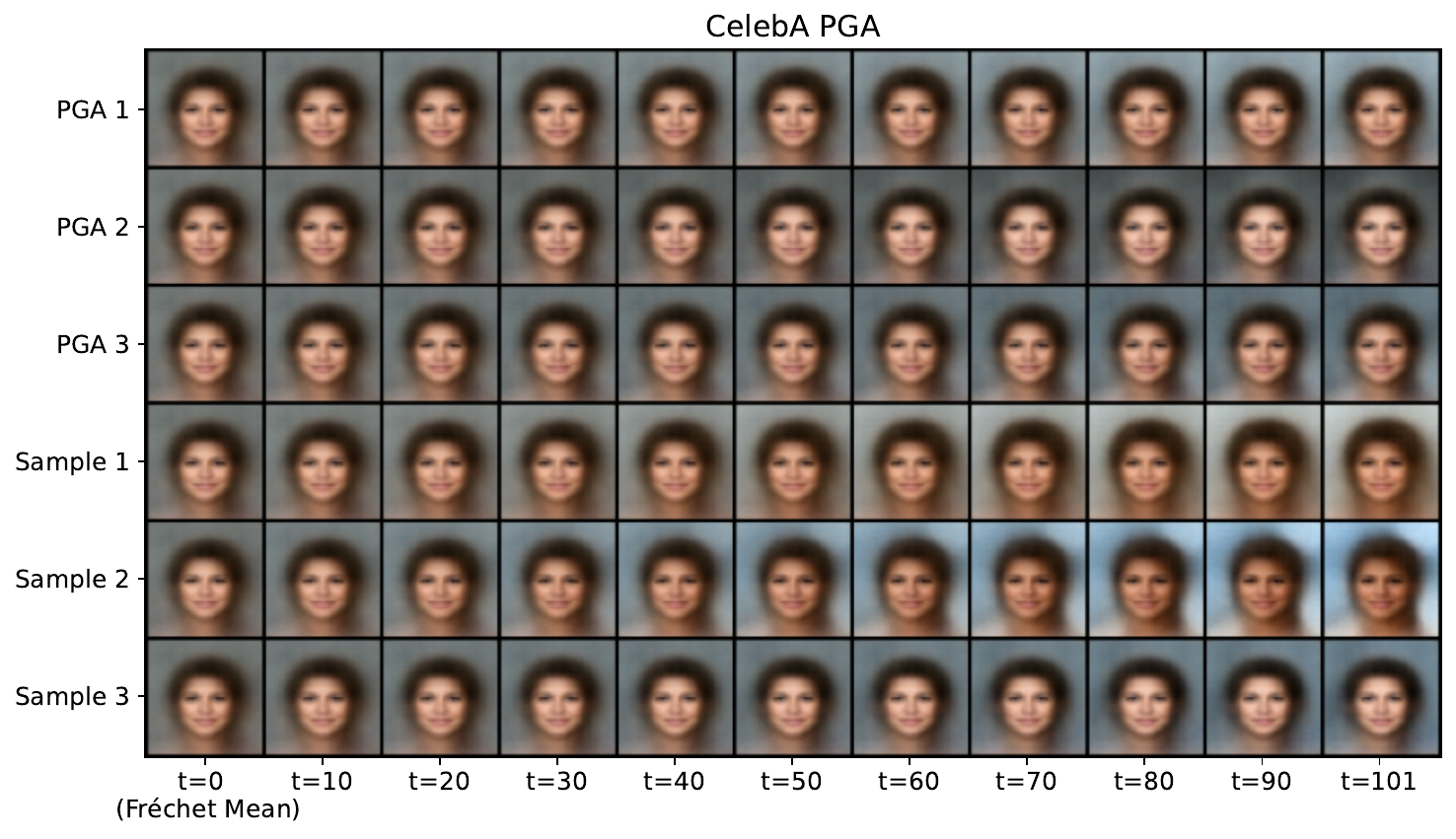}
    \caption{The first three rows shows the three-most variance explaing principal geodesics for the 10 images in Fig.~\ref{fig:celeba_riemannian_32}. The bottom three rows shows samples using the three principal geodesics by $\mathrm{Exp}\left(\mu, \sum_{k=1}^{3}\alpha_{k}v_{k}\right)$, where $\mu$ is the Fr\'echet mean, $\{v_{k}\}_{k=1}^{3}$ are the initial direction of the principal geodesics and $\alpha_{k} \sim \mathcal{N}(0,1)$. We compute the exponential map using the \textsc{ode} in Eq.~\ref{eq:bvp_ode} with the explicit Runge-Kutta method of order 5 (4) \citep{rk45}.}
    \label{fig:celeba_pga_riemannian_32}
    \vspace{-1.em}
\end{figure}

\paragraph{Application to geometric statistics} In this section, we illustrate how geometric statistics can be easily computed using \textit{GEORCE-FM}. To estimate the Fr\'echet mean, we compute both the Fr\'echet mean and geodesics as well as computing $u_{0,i}$ for each data point $a_{i}$ for $i=1,\dots,N$, which can be interpreted as the logarithmic map modulo scaling. These three ``ingredients'' often provide the basis for more elaborate statistics. To illustrate this, consider principal geodesic analysis \citep{fletcher_pga}, where data is projected onto geodesics that describe the most of the variance in the data. In this case, the principal geodesics are computed as the eigenvectors of
\begin{equation*}
    S = \frac{1}{N}\sum_{i=1}^{N}\mathrm{Log}_{\mu}(a_{i})\mathrm{Log}_{\mu}(a_{i})^{\top},
\end{equation*}
where $\mu$ is the Fr\'echet mean for the data $\{a_{i}\}_{i=1}^{N}$. By computing the Fr\'echet mean using \textit{GEORCE-FM}, the eigenvectors of $S$ can be directly computed approximating the logarithmic map as $\mathrm{Log}_{\mu}(a_{i}) \approx u_{0,i}$. In Fig.~\ref{fig:celeba_pga_riemannian_32}, we illustrate the principal geodesics and samples for the \textsc{vae} on the CelebA dataset \citep{liu2015faceattributes} for the images in Fig.~\ref{fig:celeba_riemannian_32}.

%% file: General/conclusion.tex
In this paper, we have introduced \textit{GEORCE-FM}; an algorithm to compute geodesics and the Fr\'echet mean simultaneously, which significantly reduces runtime for Riemannian and Finslerian manifolds. We have shown that the algorithm has global convergence and local quadratic convergence. In addition, we have derived an adaptive extension to make it scalable for a large number of data and have proven that it converges in expectation to the Fr\'echet mean. Empirically, we have shown that \textit{GEORCE-FM} exhibits superior runtime and accuracy of the Fr\'echet mean compared to other optimization methods. 

\paragraph{Limitations} Estimating geodesics and the Fr\'echet mean at the same time can be numerically expensive, as it requires storing multiple matrices for each grid point, which makes it less efficient in memory. Although this can be solved by only computing the metric matrix function when needed, this will increase the runtime as the metric matrix function would have to be computed multiple times again in each iteration. Thus, the method provides a trade-off between memory efficiency and runtime efficiency, where the optimal choice depends on the number of data and dimensionality. In the adaptive update scheme for \textit{GEORCE-FM} we have proposed rather simple update formulas. In future research, it should be studied whether the adaptive scheme can be improved by incorporating, e.g., momentum of the update similar to other stochastic optimizers. Our algorithm also assumes access to a local coordinate system, which can be a limitation for manifolds, where this is not practically possible.

%% file: General/appendix/ackowledgement.tex
We acknowledge the Python library JAX \citep{jax2018github}, from which our implementation has been built.

%% file: General/appendix/proofs/frechet_equivalence.tex
\begin{lemma}
    Let $\mu^{\mathrm{FM}}$ denote the set of minima of eq.~\ref{eq:frechet_mean}, and let $\mu^{\mathrm{Energy}}$ denote the set of minima of eq.~\ref{eq:frechet_energy} with $w_{i}=1$ for all $i \in \{1,\dots,N\}$. Then
    \begin{equation*}
        \mu^{\mathrm{Energy}} = \mu^{\mathrm{FM}}.
    \end{equation*}
\end{lemma}
\begin{proof}
    Let $\mathcal{L}_{\gamma}(a,y)$ denote the length for a curve $\gamma$ between points $a,y \in \mathcal{M}$. Similarly, let $\mathcal{E}_{\gamma}(a,y)$ denote the energy for a curve $\gamma$ between points $a,y \in \mathcal{M}$. By the proof of Lemma 2.3 page 194 in \citep{do1992riemannian} it follows that
    \begin{equation*}
        \mathcal{E}_{\gamma}(a,y) \propto \left(\mathcal{L}_{\gamma}(a,y)\right)^{2},
    \end{equation*}
    if and only if $\gamma$ is a geodesic. This implies that
    \begin{equation*}
        \argmin_{y \in \mathcal{M}}\sum_{i=1}^{N} \dist^{2}(y, a_{i}) = \argmin_{y \in \mathcal{M}}\sum_{i=1}^{N} \min_{\gamma_{i}}\left(\mathcal{L}_{\gamma_{i}}(y,a_{i})\right)^{2} = \argmin_{y \in \mathcal{M}}\sum_{i=1}^{N} \min_{\gamma_{i}}\mathcal{E}_{\gamma_{i}}(y,a_{i}),
    \end{equation*}
    since minimizing the energy is a geodesic \citep{gallot2004riemannian}. In the latter equation, we minimize the starting point of the minimizing curves that connect the start and data points. This can equivalently be minimized jointly as
    \begin{equation*}
        \argmin_{\substack{y \in \mathcal{M} \\ \left\{\gamma_{i}\right\}_{i=1}^{N}}}\sum_{i=1}^{N} \mathcal{E}_{\gamma}(a_{i},y),
    \end{equation*}
    where we have changed the order of the start and end point, since the Riemannian energy is symmetric.
\end{proof}

%% file: General/appendix/proofs/finsler_energy_frechet.tex
\begin{lemma}
    Let $\mu^{\mathrm{FM}}$ denote the set of minima of eq.~\ref{eq:finsler_frechet_mean}, and let $\mu^{\mathrm{Energy}}$ denote the set of minima given by
    \begin{equation}
        \mu^{\mathrm{Energy}}, \{\gamma_{i}\}_{i=1}^{N} = \argmin_{\substack{y \in \mathcal{M} \\ \{\gamma_{i}\}_{i=1}^{N}}}\sum_{i=1}^{N} \mathcal{E}^{(F)}_{\gamma_{i}}(y, a_{i}),
    \end{equation}
    where the energy $\mathcal{E}^{(F)}$ is defined using the fundamental tensor $g$, i.e.
    \begin{equation*}
        \mathcal{E}^{(F)}_{\gamma}(y, a_{i}) = \int_{0}^{1}g\left(\gamma(t),\dot{\gamma}(t)\right)\,\dif t,
    \end{equation*}
    with $\gamma(0)=y$ and $\gamma(1)=a_{i}$, Then 
    \begin{equation*}
        \mu^{\mathrm{Energy}} = \mu^{\mathrm{FM}}.
    \end{equation*}
\end{lemma}
\begin{proof}
    Let $\mathcal{L}^{(F)}_{\gamma}(a,y)$ denote the length for a curve $\gamma$ between points $a,y \in \mathcal{M}$. Similarly, let $\mathcal{E}^{(F)}_{\gamma}(a,y)$ denote the energy for a curve $\gamma$ between points $a,y \in \mathcal{M}$. By Cauchy-Schwarz inequality it follows competently similar to Lemma 2.3 in \cite{do1992riemannian}
    \begin{equation*}
        \left(\mathcal{L}^{(F)}_{\gamma}(a,y)\right)^{2} \leq c\mathcal{E}^{(F)}_{\gamma}(a,y),
    \end{equation*}
    where it is assumed that $\gamma$ is paraemetrized by the interval $[0,c]$. Equality holds if and only if the $g(\dot{\gamma}(t),\dot{\gamma}(t))$ is constant. By \cite{ohta2021comparison} 3.3 page 24 then critical points of the energy functional are both locally length minimizing and of constant speed. Thus if and only if $\gamma$ is a geodesic, then
    \begin{equation*}
        \argmin_{y \in \mathcal{M}}\sum_{i=1}^{N} \dist_{F}^{2}(y, a_{i}) = \argmin_{y \in \mathcal{M}}\sum_{i=1}^{N} \min_{\gamma_{i}}\left(\mathcal{L}_{\gamma_{i}}(y,a_{i})\right)^{2} = \argmin_{y \in \mathcal{M}}\sum_{i=1}^{N} \min_{\gamma_{i}}\mathcal{E}^{(F)}_{\gamma_{i}}(y,a_{i}),
    \end{equation*}
    since minimizing the energy is a geodesic \citep{gallot2004riemannian}. In the latter equation, we minimize the starting point of the minimizing curves that connect the start and data points. This can equivalently be minimized jointly as
    \begin{equation*}
        \argmin_{\substack{y \in \mathcal{M} \\ \left\{\gamma_{i}\right\}_{i=1}^{N}}}\sum_{i=1}^{N} \mathcal{E}^{(F)}_{\gamma}(y, a_{i}).
    \end{equation*}
\end{proof}

%% file: General/appendix/proofs/assumptions.tex
For the proof of local quadratic convergence we will assume the same regularity of the discretized energy functional as in \citep{georce}. We re-state the assumptions below
\begin{assumption}[\citep{georce}] \label{assum:quad_conv_assumptions}
    We assume the following regarding the discretized energy functional.
    \begin{itemize}
        \item We assume that the discretized energy functional $E(z)$ is locally strictly convex in the (local) minimum point $z^{*}=\left(x^{*},u^{*}\right)$ in the sense that
        \begin{equation*}
            \exists \epsilon>0: \, \forall z \in B_{\epsilon}\left(z^{*}\right), z \neq z^{*}: \, \forall \alpha ]0,1[: \, E\left(\left(1-\alpha\right)z+ \alpha z^{*}\right) < \left(1-\alpha\right)E(z)+\alpha E\left(z^{*}\right),
        \end{equation*}
        where $B_{\epsilon}=\left\{z \, |\, \norm{z-z^{*}} < \epsilon\right\}$.
        \item Assume that the discretized energy functional $E(z)$ is smooth and thus locally Lipschitz, and consider the first order Taylor approximation of the discretized energy functional
        \begin{equation*}
            \Delta E = \langle \nabla E(z_{0}), \Delta z \rangle + \mathcal{O}\left(\Delta z\right)\norm{\Delta z},
        \end{equation*}
        Then the term $\mathcal{O}\left(\Delta z\right)\norm{\Delta z}$ can be re-written as $\mathcal{\tilde{O}}\left(\norm{\Delta z}^{2}\right)$ locally.
    \end{itemize}
\end{assumption}
In the proofs for global and local quadratic convergence, we let $\norm{\cdot}$, $\langle \cdot,\cdot \angle$ and $\nabla$ denote the Euclidean norm, inner product and gradient, respectively.

%% file: General/appendix/proofs/necessary_cond.tex
\begin{proposition} 
    The necessary conditions for a minimum in Eq.~\ref{eq:energy_control} is
    \begin{equation} 
        \begin{split}
            &2w_{i}G(x_{t,i})u_{t,i}+\mu_{t,i}=0, \quad t=0,\dots, T-1, \, i=1,\dots,N, \\
            &x_{t+1,i}=x_{t,i}+u_{t,i}, \quad t=0,\dots,T-1, \, i=1,\dots,N, \\
            &\restr{\nabla_{x}\left[w_{i}u_{t,i}^{\top}G(x)u_{t,i}\right]}{x=x_{t,i}}+\mu_{t}=\mu_{t-1}, \quad t=1,\dots,T-1, \\
            &0 = \sum_{i=1}^{N}\mu_{T-1,i}, \\
            &x_{0,i}=a_{i}, x_{T,i}=y, \quad i=1,\dots,N.
        \end{split}
    \end{equation}
    where $\mu_{t,i} \in \mathbb{R}^{d}$ for $t=0,\dots,T-1$ and $i=1,\dots,N$.
\end{proposition}
\begin{proof}
    The corresponding Hamiltonian function of eq.~\ref{eq:energy_control} for $t=0,\dots,T-1$ and $i=1,\dots,N$ is
    \begin{equation} \label{eq:hamiltonian}
        \begin{split}
            H_{t}\left(x_{t},u_{t},\mu_{t}\right) &= \sum_{i=1}^{N}H_{t,i}\left(x_{t,i},u_{t,i},\mu_{t,i}\right), \\
            H_{t,i}\left(x_{t,i},u_{t,i},\mu_{t,i}\right) &= w_{i}u_{t,i}^{\top}G(x_{t,i})u_{t,i}+\mu_{t,i}^{\top}\left(x_{t,i}+u_{t,i}\right).
        \end{split}
    \end{equation}
    The time-discrete version of Pontryagin's maximum problem on for eq.~\ref{eq:energy_control} gives the following optimization problem with the slight modification that the endpoint, $x_{T,i}$, replaced by the variable Fr\'echet mean $y$ compared to \citep{georce} 
    \begin{align}
        \min_{u_{t,i}} \quad &\sum_{i=1}^{N}\sum_{t=0}^{T} H_{t,i}(x_{t,i},u_{t,i},\mu_{t,i}) \label{eq:energy_pontryagin_start}\\
        \text{s.t.} \quad &x_{t+1,i}=x_{t,i}+u_{t,i}, \quad t=0,\dots,T-2, \, i=1,\dots,N, &&(\text{state equation}),  \\
        &y=x_{T-1,i}+u_{T-1,i}, \quad i=1,\dots,N, &&(\text{state equation}),  \\
        &\nabla_{x_{t,i}}H_{t,i}(x_{t,i},u_{t,i},\mu_{t,i}) =\mu_{t-1,i}, \quad t=1,\dots,T-1, \, i=1,\dots,N &&(\text{co-state equation}), \\
        &0 = \sum_{i=1}^{N}\mu_{T-1,i}, \quad i=1,\dots,N, &&(\text{co-state equation}), \\
        &x_{0,i}=a_{i}, \quad i=1,\dots,N. \label{eq:energy_pontryagin_end}
    \end{align}
    The minimization problem can be decomposed into the following sub-problems for each $i \in \{1,\dots,N\}$, where the equality constraints on the control variables are relaxed in Lagrange sense
    \begin{equation*}
        \min_{u_{t,i}} H_{t,i}\left(x_{t,i},u_{t,i},\mu_{t,i}\right), \quad t=0,\dots,T-1.
    \end{equation*}
    We observe that $H_{t,i}\left(x_{t,i},u_{t,i},\mu_{t,i}\right)$ is strictly convex in $u_{t,i}$, since $G(x_{t,i})$ is positive definite. The stationary point wrt. $u_{t,i}$ is therefore also the global minimum point, which gives the following equations including state and co-state equations from eq.~\ref{eq:energy_control}
    \begin{equation}
        \begin{split}
            &2w_{i}G(x_{t,i})u_{t,i}+\mu_{t,i}=0, \quad t=0,\dots, T-1, \, i=1,\dots,N, \\
            &x_{t+1,i}=x_{t,i}+u_{t,i}, \quad t=0,\dots,T-2, \, i=1,\dots,N, \\
            &y=x_{T-1,i}+u_{T-1,i}, \quad i=1,\dots,N, \\
            &\restr{\nabla_{x}\left[w_{i}u_{t,i}^{\top}G(x)u_{t,i}\right]}{x=x_{t,i}}+\mu_{t}=\mu_{t-1}, \quad t=1,\dots,T-1, \\
            &0 = \sum_{i=1}^{N}\mu_{T-1,i}, \\
            &x_{0,i}=a_{i}, x_{T}=b, \quad i=1,\dots,N.
        \end{split}
    \end{equation}
\end{proof}

%% file: General/appendix/proofs/update_scheme.tex
\begin{proposition} 
    The update scheme for $u_{t},\mu_{t}$ and $x_{t}$ in eq.~\ref{eq:energy_zero_point_problem} can be reduced to
    \begin{equation*}
        \begin{split}
            &y = W^{-1}V, \\
            &\mu_{T-1,i} = \left(\sum_{t=0}^{T-1}G_{t,i}^{-1}\right)^{-1}\left(2w_{i}(a_{i}-y)-\sum_{t=0}^{T-1}G_{t,i}^{-1}\sum_{t>j}^{T-1}\nu_{j,i}\right), \quad i=1,\dots,N, \\
            &u_{t,i} = -\frac{1}{2w_{i}}G_{t,i}^{-1}\left(\mu_{T-1,i}+\sum_{j>t}^{T-1}\nu_{j,i}\right), \quad t=0,\dots,T-1, \, i=1,\dots,N \\
            &x_{t+1,i} = x_{t,i}+u_{t,i}, \quad t=0,\dots,T-2, \, i=1,\dots,N, \\
            &x_{0,i}=a_{i} \quad i=1,\dots,N,
        \end{split}
    \end{equation*}
    where
    \begin{equation*}
        \begin{split}
            W &= \sum_{i=1}^{N}w_{i}\left(\sum_{t=0}^{T-1}G_{t,i}^{-1}\right)^{-1}, \\
            V &= \sum_{i=1}^{N}w_{i}\left(\sum_{t=0}^{T-1}G_{t,i}^{-1}\right)^{-1}a_{i}-\frac{1}{2}\sum_{i=1}^{N}\left(\sum_{t=0}^{T-1}G_{t,i}^{-1}\right)^{-1}\sum_{t=0}^{T-1}G_{t,i}^{-1}\sum_{j>t}^{T-1}\nu_{j,i}.
        \end{split}
    \end{equation*}
\end{proposition}
\begin{proof}
    The equation system solved in \textit{GEORCE-FM} in iteration $k+1$ is
    \begin{equation*}
        \begin{split}
            &\nu_{t,i}+\mu_{t,i} = \mu_{t-1,i}, \quad t=1,\dots,T-1, \, i=1,\dots,N, \\
            &2w_{i}G_{t,i}u_{t,i}+\mu_{t,i} = 0, \quad t=0,\dots,T-1, \, i=1,\dots,N, \\
            &\sum_{t=0}^{T-1}u_{t,i}=y-a_{i}, \quad i=1,\dots,N, \\
            &\sum_{i=1}^{N}\mu_{T-1,i} = 0, \\
            &\nu_{t,i} := \restr{\nabla_{x}\left(w_{i}u_{t,i}^{\top}G(x)u_{t,i}\right)}{y=x_{t,i}^{(k)},u_{t,i}=u_{t,i}^{(k)}}, \quad t=1,\dots,T-1, \, i=1,\dots,N\\
            &G_{t,i} := G\left(x_{t,i}^{(k)}\right), \quad t=0,\dots,T-1,
        \end{split}
    \end{equation*}
    where the related state variables (before line-search) are found by the state equation
    \begin{equation*}
        x_{t+1,i} = x_{t,i} + u_{t,i}, t=0,\dots,T-2, i=1,\dots,N, \\
        x_{0,i} = a_{i}, \quad i=1,\dots,N.
    \end{equation*}
    As we will see in the following $y=x_{T-1,i}+u_{T-1,i}$ for $i=1,\dots,N$ is omitted in the update step, since $y$ will be determined initially in each iteration. By re-arranging the terms we see that
    \begin{equation*}
        \begin{split}
            &\mu_{t,i} = \mu_{T-1,i} + \sum_{j>t}^{T-1} \nu_{j,t}, \quad t=0,\dots,T-1,i=1,\dots,N, \\
            &u_{t,i} = -\frac{1}{2}G_{t,i}^{-1}\mu_{t,i} = -\frac{1}{2w_{i}}G_{t,i}^{-1}\left(\mu_{T-1,i}+\sum_{j>t}^{T-1}\nu_{j,i}\right), \quad t=0,\dots,T-1,i=1,\dots,N, \\
            &-\frac{1}{2\omega_{i}}\sum_{t=0}^{T-1}G_{t,i}^{-1}\left(\mu_{T-1,i}+\sum_{j>t}^{T-1}g_{ji}\right)=y-a_{i}, \quad i=1,\dots,N,
        \end{split}
    \end{equation*}
    where it is exploited that $G_{t,i}$ is positive definite for all $t,i$ and has a well defined inverse. These equations can be used to provide explicit solutions for the state and control variables. First observe that
    \begin{equation*}
        \mu_{T-1,i} = \left(\sum_{t=0}^{T-1}G_{t,i}^{-1}\right)^{-1}\left(2w_{i}(a_{i}-y)-\sum_{t=0}^{T-1}G_{t,i}^{-1}\sum_{j>t}^{T-1}\nu_{j,i}\right), \quad i=1,\dots,N.
    \end{equation*}
    Applying this result, we deduce that
    \begin{equation*}
        \sum_{i=1}^{N}\mu_{T-1,i} = 0 \Leftrightarrow \sum_{i=1}^{N}\left(\sum_{t=0}^{T-1}G_{t,i}^{-1}\right)^{-1}\left(2w_{i}(a_{i}-y)-\sum_{t=0}^{T-1}G_{t,i}^{-1}\sum_{j>t}^{T-1}\nu_{j,i}\right) = 0,
    \end{equation*}
    which implies that
    \begin{equation*}
        y = W^{-1}V.
    \end{equation*}
    where
    \begin{equation*}
        \begin{split}
            W &= \sum_{i=1}^{N}w_{i}\left(\sum_{t=0}^{T-1}G_{t,i}^{-1}\right)^{-1}, \\
            V &= \sum_{i=1}^{N}w_{i}\left(\sum_{t=0}^{T-1}G_{ti}^{-1}\right)^{-1}a_{i}-\frac{1}{2}\sum_{i=1}^{N}\left(\sum_{t=0}^{T-1}G_{t,i}^{-1}\right)^{-1}\sum_{t=0}^{T-1}G_{ti}^{-1}\sum_{j>t}^{T-1}g_{ji}.
        \end{split}
    \end{equation*}
    Thus, we get the update scheme in Proposition~\ref{prop:update_scheme}. 
\end{proof}

%% file: General/appendix/proofs/global_convergence.tex
We prove that \textit{GEORCE-FM} has global convergence similar to the original \textit{GEORCE} by extending the proof to estimating multiple geodesics and end point. We follow the same approach as in \citep{georce} and adapt the Lemma's and proofs to the general setting.

\begin{lemma} \label{lemma:global_conv_lin_comb}
    Assume that $\left(x^{(j)}, u^{(j)}\right)$ is a feasible solution, then the following properties hold.
    \begin{itemize}
        \item There exists a unique solution, $\left(x^{(j+1)}, u^{(j+1)}\right)$, to the system of equations in eq.~\ref{eq:energy_opt_condtions} based on $\left(x^{(j)}, u^{(j)}\right)$.
        \item All linear combinations $\left(1-\alpha\right)\left(x^{(j)},u^{(j)}\right)+\alpha\left(x^{(j+1)},u^{(j+1)}\right)$ for $0 \leq \alpha \leq 1$ are feasible solutions.
    \end{itemize}
\end{lemma}
\begin{proof}
    Since the metric tensors, $\left\{G\left(x_{s,i}^{(j)}, u_{s,i}^{(j)}\right)\right\}_{s=0,\dots,T-1,i=1,\dots,N}$ are positive definite matrices, then the matrices are regular, and the inverse is unique, and so are sum of the inverse matrices. That means that there exists a unique solution to the system of equations in iteration $j$ in eq.~\ref{eq:finsler_eq_system} wrt. $\left\{u_{s,i}\right\}_{s=0,\dots,T-1,i=1,\dots,N}$ and $\left\{x_{s,i}^{(j)}\right\}_{s=0,\dots,T-1,i=1,\dots,N}$ by the state equations.

    At first, assume that $\left(x^{(j+1)},u^{(j+1)}\right)$ is well-defined in the domain of the local chart.

    Since the solution, $\left(x^{(j+1)}, u^{(j+1)}\right)$ is also a feasible solution according to eq.~\ref{eq:finsler_eq_system}, then the linear combination of two feasible solutions will also be a feasible solution, i.e.
    \begin{equation*}
        \sum_{s=0}^{T-1}u_{s,i}^{(j+1)} = y-a_{i}, \quad i=1,\dots,N.
    \end{equation*}
    The new solution based on the linear combination gives
    \begin{equation*}
        \begin{split}
            &\left(a-\alpha\right)\sum_{s=0}^{T-1}u_{s,i}^{(j)}+\alpha\sum_{s=0}^{T-1}u_{s,i}^{(j+1)} = \left(1-\alpha\right)\left(y-a_{i}\right)+\alpha\left(y-a_{i}\right) = y-a_{i}, \quad i=1,\dots,N, \\
            &\left(1-\alpha\right)\left(a_{i},x_{1,i}^{(j)},\dots,x_{T-1,i}^{(j)},y\right)+\alpha\left(a_{i},x_{1,i}^{(j)},\dots,x_{T-1,i}^{(j)},y\right) \\
            &= \left(a_{i}, \left(1-\alpha\right)x_{1,i}^{(j)}+\alpha x_{1,i}^{(j)}, \dots, \left(1-\alpha\right)x_{T-1,i}^{(j)}+\alpha x_{T-1,i}^{(j+1)}, y\right), \quad i=1,\dots,N,
        \end{split}
    \end{equation*}
    where
    \begin{equation*}
        \begin{split}
            &\left(1-\alpha\right)x_{s,i}^{(j)}+\alpha x_{s,i}^{(j+1)} = \left(1-\alpha\right)\left(a_{i}+\sum_{k=0}^{t-1}u_{k,i}^{(j)}\right)+\alpha\left(a_{i}+\sum_{k=0}^{t-1}u_{k,i}^{(j+1)}\right) \\
            &= a_{i}+\sum_{k=0}^{t-1}\left(1-\alpha\right)u_{k,i}^{(j)}+\alpha u_{k,i}^{(j+1)}, \quad i=1,\dots,N.
        \end{split}
    \end{equation*}
    This shows the linear combination in the state variable is feasible as it produces a consistent state variable in terms of the start and end point, and furthermore each state variable is feasible if they were determined from the linear combination of the control vectors, which proves that the linear combinations of the state and control variables are also feasible solutions.

    In contrast, if $\left(x^{(j+1)},u^{(j+1)}\right)$ is not well-defined in the local chart, then the line-search in \textit{GEORCE-FM} will be able to determine a point belonging to the manifold and that will be on the line between $\left(x^{(j+1)},u^{(j+1)}\right)$ and $\left(x^{(j)},u^{(j)}\right)$ as the local chart is defined on an open set and $x^{(j)}$ is well-defined in the local chart similar to \textit{GEORCE} \citep{georce}. The new solution is therefore a feasible solution following the argument above with this modification.
\end{proof}

\begin{lemma} \label{lemma:global_conv_minimum}
    Let $\left\{x_{s,i}^{(j)}, u_{s,i}^{(j)}\right\}_{s=0,\dots,T,i=1,\dots,N}$ denote the (feasible) solution after iteration $j$ in \textit{GEORCE-FM}. If $\left\{x_{s,i}^{(j)}, u_{s,i}^{(j)}\right\}_{s=0,\dots,T,i=0,\dots,N}$ decreases the objective function in the sense that there exists an $\eta > 0$ such that for all $0 < \alpha \leq \eta \leq 1$, then
    \begin{equation*}
        E\left(x^{(j)}+\alpha\left(x^{(j+1)}-x^{(j)}\right), u^{(j)}+\alpha\left(u^{(j+1)}-u^{(j)}\right)\right) < E\left(x^{(j)}, u^{(j)}\right),
    \end{equation*}
    where $x^{(j)}[i] = \left(a_{i},x_{1,i}^{(j)}, \dots, x_{T-1,i}^{(j)}, y^{(j)}\right)$ and $u^{(j)}[i] = \left(u_{0,i},u_{1,i}^{(j)}, \dots, u_{T-1,i}^{(j)}, u_{T,i}^{(j)}\right)$.
\end{lemma}
\begin{proof}
    Since $E$ is a smooth function, the first order Taylor approximation is well defined
    \begin{equation*}
        \begin{split}
            \Delta E_{i}\left(x,u\right) &= \sum_{s=1}^{T}\sum_{i=1}^{N}\left(\left\langle \restr{\nabla_{x}E(x,u)}{\left(x,u\right)=\left(x_{s,i}^{(j)},u_{s,i}^{(j)}\right)}, \Delta x_{s,i}\right \rangle + \mathcal{O}\left(\Delta x_{s,i}\right)\norm{x_{s,i}}\right) \\
            &+ \sum_{s=1}^{T}\sum_{i=1}^{N}\left(\left\langle \restr{\nabla_{u}E(x,u)}{\left(x,u\right)=\left(x_{s,i}^{(j)},u_{s,i}^{(j)}\right)}, \Delta u_{s,i}\right \rangle + \mathcal{O}\left(\Delta u_{s,i}\right)\norm{u_{s,i}}\right),
        \end{split}
    \end{equation*}
    where $\langle \cdot, \cdot \rangle$ is the standard Euclidean inner product and $\lim_{v \rightarrow 0}\mathcal{O}\left(v\right) = 0$. The perturbation, $\left(\Delta x_{s,i}, \Delta u_{s,i}\right)$ is restricted to be a step of \textit{GEORCE-FM} including the end point. From the optimality conditions in eq.~\ref{eq:finsler_eq_system} we have that
    \begin{equation*}
        \begin{split}
            &\restr{\nabla_{x_{s,i}}E\left(x,u\right)}{\left(x,u\right)=\left(x_{s,i}^{(j)},u_{s,i}^{(j)}\right)} = \mu_{s-1,i}-\mu_{s,i}, \quad s=1,\dots,T-1, \, i=1,\dots, N. \\
            &\restr{\nabla_{x_{T,i}}E\left(x,u\right)}{\left(x,u\right)=\left(x_{s,i}^{(j)},u_{s,i}^{(j)}\right)} = \mu_{T-1,i}, \quad i=1,\dots, N. \\
            &\restr{\nabla_{u_{s,i}}E\left(x,u\right)}{\left(x,u\right)=\left(x_{s,i}^{(j)},u_{s,i}^{(j+1)}\right)} = 2\omega_{i}G_{s,i}u_{s,i}^{(j+1)}+\omega_{i}\zeta_{s,i} = -\mu_{s,i}, \quad s=1,\dots,T-1, \, i=1,\dots, N. \\
        \end{split}
    \end{equation*}
    Applying this to the first order Taylor expansion we have that
    \begin{equation*}
        \begin{split}
            \Delta E(x,u) &= \sum_{s=1}^{T-1}\sum_{i=1}^{N}\left(\left\langle \mu_{s-1,i}-\mu_{s,i}, \Delta x_{s,i} \right \rangle +  \mathcal{O}\left(\Delta x_{s,i}\right)\norm{\Delta x_{s,i}}\right) \\
            &+\sum_{i=1}^{N}\left\langle \mu_{T-1,i}, \Delta x_{T,i} \right \rangle \\
            &+ \sum_{s=0}^{T-1}\sum_{i=1}^{N}\left(\left\langle \restr{\nabla_{u_{s,i}}E(x,u)}{\left(x,u\right)=\left(x_{s,i}^{(j)}, u_{s,i}^{(j)}\right)}, \Delta u_{s,i} \right \rangle + \mathcal{O}\left(\Delta u_{s,i}\right)\norm{\Delta u_{s,i}}\right).
        \end{split}
    \end{equation*}
    Re-arranging the terms we see that
    \begin{equation*}
        \begin{split}
            \Delta E(x,u) &= \sum_{s=0}^{T-2}\sum_{i=1}^{N}\left\langle \mu_{s,i}, \Delta x_{s+1,i}\right \rangle \\
            &- \sum_{s=1}^{T-1}\sum_{i=1}^{N}\left(\left\langle \mu_{s,i}, \Delta x_{s,i} \right \rangle + \mathcal{O}\left(\Delta x_{s,i}\right)\norm{\Delta x_{s,i}}\right) \\
            &+ \sum_{i=1}^{N}\left\langle \mu_{T-1,i}, \Delta x_{T,i}\right \rangle \\
            &+ \sum_{s=0}^{T-1}\sum_{i=1}^{N}\left(\left\langle \restr{\nabla_{u_{s,i}}E(x,u)}{\left(x,u\right)=\left(x_{s,i}^{(j)}, u_{s,i}^{(j)}\right)}, \Delta u_{s,i}\right\rangle + \mathcal{O}\left(\Delta u_{s,i}\right)\norm{\Delta u_{s,i}}\right) \\
            &= \sum_{s=1}^{T-2}\sum_{i=1}^{N}\left\langle \mu_{s,i}, \Delta x_{s+1,i}-\Delta x_{s}\right\rangle \\
            &- \sum_{i=1}^{N}\left\langle \mu_{T-1,i}, \Delta x_{T-1,i}\right\rangle - \sum_{s=1}^{T-1}\sum_{i=1}^{N} \mathcal{O}\left(\Delta x_{s,i}\right)\norm{\Delta x_{s,i}} \\
            &+ \sum_{i=1}^{N}\left\langle \mu_{T-1,i}, \Delta x_{T-1,i} \right \rangle \\
            &+ \sum_{s=0}^{T-1}\sum_{i=1}^{N}\left(\left\langle \restr{\nabla_{u_{s,i}}E(x,u)}{\left(x,u\right)=\left(x_{s,i}^{(j)}, u_{s,i}^{(j)}\right)}, \Delta u_{s,i}\right\rangle + \mathcal{O}\left(\Delta u_{s,i}\right)\norm{u_{s,i}}\right)
        \end{split}
    \end{equation*}
    Observe that
    \begin{equation*}
        \begin{split}
            \Delta x_{s+1,i}-\Delta x_{s,i} &= \left(x_{s+1,i}^{(j+1)}-x_{s+1,i}^{(j)}\right) - \left(x_{s,i}^{(j+1)}-x_{s,i}^{(j)}\right) \\
            &= \left(x_{s+1,i}^{(j+1)}-x_{s,i}^{(j+1)}\right)-\left(x_{s+1,i}^{(j)}-x_{s,i}^{(j)}\right) \\
            &= u_{s,i}^{(j+1)}-u_{s,i}^{(j)} \\
            &= \Delta u_{s,i}.
        \end{split}
    \end{equation*}
    Note that $\Delta x_{1,i} = \Delta x_{1,i} - \Delta x_{0,i}$, since $\Delta x_{0,i}=0$. Using this we have that
    \begin{equation*}
        \begin{split}
            \Delta E(x,u) &= \sum_{s=0}^{T-1}\sum_{i=1}^{N}\left(\left\langle \mu_{s,i}, \Delta u_{s,i}\right\rangle + \mathcal{O}\left(\Delta x_{s,i}\right)\norm{\Delta x_{s,i}}\right) \\
            &+ \sum_{s=0}^{T-1}\sum_{i=1}^{N}\left(\left\langle \restr{\nabla_{s,i}E(x,u)}{\left(x,u\right)=\left(x_{s,i}^{(j)}, u_{s,i}^{(j)}\right)}, \Delta u_{s,i} \right \rangle + \mathcal{O}\left(\Delta u_{s,i}\right)\norm{\Delta u_{s,i}}\right) \\
            &= \sum_{s=0}^{T-1}\sum_{i=1}^{N}\left(\left\langle \mu_{s,i} + \restr{\nabla_{u_{s,i}}E(x,u)}{\left(x,u\right)=\left(x_{s,i}^{(j)}, u_{s,i}^{(j)}\right)}, \Delta u_{s,i}\right\rangle + \mathcal{O}\left(\Delta x_{s,i}\right)\norm{\Delta x_{s,i}}\right) \\
            &+ \sum_{t=0}^{T-1}\mathcal{O}\left(\Delta u_{t,i}\right)\norm{\Delta u_{t,i}}.
        \end{split}
    \end{equation*}
    Since 
    \begin{equation*}
        \begin{split}
            &\mu_{s,i} + \restr{\nabla_{u_{s,i}}E(x,u)}{\left(x,u\right)=\left(x_{s,i}^{(j)}, u_{s,i}^{(j)}\right)} \\
            &= -\restr{\nabla_{u_{s,i}}E(x,u)}{\left(x,u\right)=\left(x_{s,i}^{(j)}, u_{s,i}^{(j+1)}\right)} + \restr{\nabla_{u_{s,i}}E(x,u)}{\left(x,u\right)=\left(x_{s,i}^{(j)}, u_{s,i}^{(j)}\right)} \\
            &= -2\omega_{i}G\left(x_{s,i}^{(j)}, u_{s,i}^{(j)}\right)u_{s,i}^{(j+1)} - \omega_{i}\zeta_{s} + 2G\left(x_{s,i}^{(j)}, u_{s,i}^{(j)}\right)u_{s,i}^{(j)} + \omega_{i}\zeta_{s,i} \\
            &= -2\omega_{i}G\left(x_{s,i}^{(j)}, u_{s,i}^{(j)}\right)\left(u_{s,i}^{(j+1)}-u_{s,i}^{(j)}\right)
        \end{split}
    \end{equation*}
    Thus, the first order Taylor approximation becomes
    \begin{equation*}
        \Delta E(x,u) = \sum_{s=0}^{T-1}\sum_{i=1}^{N}\left(\left\langle -2 \omega_{i} G\left(x_{s,i}^{(j)}, u_{s,i}^{(j)}\right)\Delta u_{s,i}, \Delta u_{s,i}\right\rangle + \mathcal{O}\left(\Delta x_{s,i}\right)\norm{\Delta x_{s,i}} + \mathcal{O}\left(u_{s,i}\right)\norm{\Delta u_{s,i}}\right).
    \end{equation*}
    Since $G\left(x_{s,i}^{(j)}, u_{s,i}^{(j)}\right)$ is positive definite, the first summation is negative assuming at least one non-zero vector $\left\{\Delta u_{s,i}\right\}_{s=0,\dots,T-1,i=1,\dots,N}$, which is the case as otherwise the solution $\left(x_{s,i}^{(j)}, u_{s,i}^{(j)}\right)$ is a local minimum point. Note that if $\left\{\norm{u_{s,i}}\right\}_{s=0,\dots,T-1,i=1,\dots,N}$ converges to $0$, then $\norm{\Delta x_{s,i}}$ also converges to $0$ from the state equations of $x_{s,i}$. Thus,
    \begin{equation*}
        \Delta E(x,u) = \sum_{s=0}^{T-1}\left(\left\langle -2\omega_{i} G\left(x_{s,i}^{(j)}, u_{s,i}^{(j)}\right)\Delta u_{s,i}, \Delta u_{s,i} \right \rangle + \mathcal{O}\left(\Delta u_{s,i}\right)\norm{\Delta u_{s,i}}\right).
    \end{equation*}
    Now scale all $\{\Delta u_{s,i}\}_{s=0,\dots,T-1,i=1,\dots,N}$ with a scalar $0 < \alpha \leq 1$. Note that the term $-2\omega_{i}G\left(x_{s,i}^{(j)}, u_{s,i}^{(j)}\right)\Delta u_{s,i}$ is unaffected by the scaling, since this term is equal to $\mu_{s,i}+\restr{\nabla_{u_{s,i}}E(x,u)}{\left(x,u\right)=\left(x_{s,i}^{(j)}, u_{s,i}^{(j)}\right)}$, which is independent of $\alpha$. The first order Taylor approximation is then
    \begin{equation*}
        \Delta E(x,u) = \sum_{s=0}^{T-1}\left(\left\langle -2\omega_{i}G\left(x_{s,i}^{(j)}, u_{s,i}^{(j)}\right)\Delta u_{s,i}, \alpha \Delta u_{s,i}\right\rangle + \mathcal{O}\left(\alpha \Delta u_{s,i}\right)\norm{\alpha \Delta u_{s,i}}\right).
    \end{equation*}
    In the limit it follows that there exists a scalar $\eta > 0$ such that $\frac{\Delta E(x,u)}{\alpha} < 0$ for some $0 < \alpha \leq \eta \leq 1$, i.e. $\Delta E(x,u) < 0$. From Lemma~\ref{lemma:global_conv_lin_comb} it follows that $\left(x_{s,i}^{(j+1)}(\alpha), u_{s,i}^{(j+1)}(\alpha)\right)$ is a feasible solution.
\end{proof}

\begin{proposition}\label{prop:global_convergence}
    Let $E^{(i)}$ denote the discretized energy for the solution after iteration $i$ (including line search) in \textit{GEORCE-FM}. If the starting point, $\left(x^{(0)}, u^{(0)}\right)$ is feasible, then the series $\left\{E^{(i)}\right\}_{i}$ will converge to a local minimum.
\end{proposition}
\begin{proof}
    If the starting point, $\left(x^{(0)}, u^{(0)}\right)$ in \textit{GEORCE-FM} is feasible, then the points, $\left(x^{(j)}, u^{(j)}\right)$ after each iteration (including line search) in \textit{GEORCE-FM} is also a feasible solution by Lemma~\ref{lemma:global_conv_lin_comb}.

    The discretized energy functional is a positive function, i.e. a lower bound is higher or equal to $0$. Assume that the series $\left\{E^{(j)}\right\}_{j}$ is not converging to a local minimum. From Lemma~\ref{lemma:global_conv_minimum} it follows that $\left\{E^{(j)}\right\}_{j}$ is decreasing for increasing $j$, and since $E^{(j)}$ has a lower bound, the series $E^{(j)}$ can not be non-converging. Assume on the contrary that the convergence value for $\left\{E^{(j)}\right\}_{j}$, is not a local minimum. Denote the value $\hat{E}$ and the convergence point $\left(\hat{x}, \hat{u}\right)$. According to Lemma~\ref{lemma:global_conv_minimum} \textit{GEORCE-FM} will from the solution $\left(\hat{x}, \hat{u}\right)$ in the following iteration produce a new point $\left(\tilde{x}, \tilde{u}\right)$ such that $E\left(\tilde{x}, \tilde{u}\right) < \hat{F}$, which contradicts that the series $\left\{E^{(j)}\right\}_{j}$ is converging to $\hat{F}$. Thus, the series $\left\{E^{(j)}\right\}_{j}$ will converge to a local minimum.
\end{proof}

Note that by the same argument as in \citep{georce}, then \textit{GEORCE-FM} will not be able to jump between two local minima and will converge to only one local minimum.

%% file: General/appendix/proofs/local_convergence.tex
Before proving local convergence we re-state the following Lemma from \citep{georce}
\begin{lemma}[\citep{georce}]\label{lemma:quad_conv_bound}
    Assume $f(x)$ is a smooth function. Let $z^{*}$ be a (local) minimum point for $f$, and assume that $f$ is locally strictly convex in $z^{*}$. Then
    \begin{equation*}
        \exists \epsilon>0:\quad \forall z \in B_{\epsilon}(z^{*})\setminus\{z^{*}\}:\quad \langle \nabla f(z), z^{*}-z \rangle < 0. 
    \end{equation*}
\end{lemma}

In the following, we generalize the proof in \citep{georce} to the case of the Fr\'echet mean.

\begin{proposition}\label{prop:local_convergence}
    If the discretized energy functions is locally strictly convex in the (local) minimum point $z^{*}=\left(x^{*}, u^{*}\right)$ and locally $\alpha^{*}=1$, i.e. no line-search, then \textit{GEORCE-FM} has locally quadratic convergence, i.e.
    \begin{equation*}
        \exists \epsilon>0 \, \exists c > 0 \, \forall z^{(j)} \in B_{\epsilon}\left(z^{*}\right): \quad \norm{z^{(j+1)}-z^{*}} \leq c \norm{z^{(j)}-z^{*}}^{2},
    \end{equation*}
    where $B_{\epsilon}\left(z^{*}\right)$ is an open ball with radius $\epsilon$ around $z^{*}$ and $z^{(j)}=\left\{\left(x_{s,i}^{(j)},u_{s,i}^{(j)}\right)\right\}_{s,i}$ is the $i$'th iteration of \textit{GEORCE-FM}.
\end{proposition}
\begin{proof}
    Assume $z^{*}$ is a local minimum point. Since \textit{GEORCE-FM} has global convergence to a local minimum point by Proposition~\ref{prop:global_convergence} we assume that the algorithm has convergence point $E\left(z^{*}\right)$. Consider the solution $z^{(j)}=\left(x^{(j)}, u^{(j)}\right)$ from \textit{GEORCE-FM} and assume that for all the following iterations in \textit{GEORCE-FM} the solution is in the open ball $B_{\epsilon}\left(z^{*}\right)$, and the locally strictly convex assumption in Assumptions~\ref{assum:quad_conv_assumptions} holds. Since the series, $\left\{z^{(j)}\right\}$, converge to $z^{*}$, then from a certain point $k$ in the series all point $\left\{z^{(j)}\right\}_{j \leq k}$ will belong to the open ball $B_{\epsilon}\left(z^{*}\right)$.

    It then follows that for some $K>0$
    \begin{equation*}
        E\left(z^{(j+1)}\right)-E\left(z^{*}\right) = \norm{z^{(j+1)}-z^{*}}
    \end{equation*}
    since $z^{(j+1)} \in B_{\epsilon}\left(z^{*}\right)$. The left hand side can be rearranged into
    \begin{equation*}
        E\left(z^{(j+1)}\right)-E\left(z^{*}\right) = \left(E\left(z^{(j+1)}\right)-E\left(z^{(j)}\right)\right) + \left(E\left(z^{(j)}\right)-E\left(z^{*}\right)\right).
    \end{equation*}
    By the proof of Lemma~\ref{lemma:global_conv_minimum} we see
    \begin{equation*}
        \begin{split}
            &E\left(z^{(j+1)}\right)-E\left(z^{*}\right) = \left(E\left(z^{(j+1)}\right)-E\left(z^{(j)}\right)\right) + \left(E\left(z^{(j)}\right)-E\left(z^{*}\right)\right) \\
            &= \sum_{s=0}^{T-1}\sum_{i=0}^{N}\left(\left\langle -2\omega_{i}G\left(x_{s,i}^{(j)},u_{s,i}^{(j)}\right)\Delta u_{s,i}, \Delta u_{s,i} \right \rangle + \mathcal{O}\left(\Delta u_{s,i}\right)\norm{\Delta u_{s,i}}\right) \\
            &- \sum_{s=0}^{T-1}\sum_{i=1}^{N}\left(\left\langle -2\omega_{i}G\left(x_{s,i}^{(j)}, u_{s,i}^{(j)}\right)\Delta u_{s,i}, u_{s,i}^{*}-u_{s,i}^{(j)} \right \rangle + \mathcal{O}\left(u_{s,i}^{*}-u_{s,i}^{(j)}\right)\norm{u_{s,i}^{*}-u_{s,i}^{(j)}}\right) \\
            &= \sum_{s=0}^{T-1}\sum_{i=1}^{N}\left(\left\langle -2 \omega_{i}G\left(x_{s,i}^{(j)}, u_{s,i}^{(j)}\right)\Delta u_{s,i}, \Delta u_{s,i} - \left(u_{s,i}^{*}-u_{s,i}^{(j)}\right) \right \rangle\right) \\
            &+ \sum_{s=0}^{T-1}\sum_{i=1}^{N}\left(\mathcal{O}\left(\Delta u_{s,i}\right)\norm{\Delta u_{s,i}} + \mathcal{O}\left(u_{s,i}^{*}-u_{s,i}^{(j)}\right)\norm{u_{s,i}^{*}-u_{s,i}^{(j)}}\right) \\
            &= \sum_{s=0}^{T-1}\sum_{i=1}^{N}\left(\left\langle -2 \omega_{i}G\left(t{s,i}^{(j)}, x_{s,i}^{(j)}, u_{s,i}^{(j)}\right)\Delta u_{s,i}, u_{s,i}^{*}-u_{s,i}^{(j+1)} \right \rangle\right) \\
            &+ \sum_{s=0}^{T-1}\sum_{i=1}^{N}\left(\mathcal{O}\left(\Delta u_{s,i}\right)\norm{\Delta u_{s,i}} + \mathcal{O}\left(u_{s,i}^{*}-u_{s,i}^{(j)}\right)\norm{u_{s,i}^{*}-u_{s,i}^{(j)}}\right).
        \end{split}
    \end{equation*}
    Observe that the term $2 \omega_{i}G\left(x_{s,i}^{(j)}, u_{s,i}^{(j)}\right)$ is the gradient for $z_{s}^{(j)}$, which by the assumption of strongly unique minimum point in Assumptions~\ref{assum:quad_conv_assumptions} implies that
    \begin{equation*}
        \sum_{s=0}^{T-1}\sum_{i=1}^{N}\left\langle 2 \omega_{i} G\left(x_{s,i}^{(j)}, u_{s,i}^{(j)}\right) \Delta u_{s,i}, u_{s,i}^{*}-u_{s,i}^{(j+1)} \right \rangle < 0.
    \end{equation*}
    Combining the inequalities and utilizing that $E(z)$ is locally Lipschitz continuous by the Assumptions~\ref{assum:quad_conv_assumptions}, then
    \begin{equation*}
        \mathcal{O}\left(\norm{z^{(j+1)}-z^{*}}\right) = E\left(z^{(j+1)}\right) - E\left(z^{*}\right) \leq \mathcal{O}\left(\norm{z^{(j)}-z^{*}}^{2}\right),
    \end{equation*}
    such that there exists an $\epsilon > 0$
    \begin{equation*}
        \forall z^{(j)} \in B_{\epsilon}\left(z^{*}\right): \quad \norm{z^{(j+1)}-z^{*}} \leq c \norm{z^{(j)}-z^{*}}^{2},
    \end{equation*}
    where $c > 0$.
\end{proof}

%% file: General/appendix/proofs/adaptive_convergence.tex
Consider the general minimization for the energy functional for the Fr\'echet mean
\begin{equation*}
    \min_{z} E(z) = \min_{z} \mathbb{E}_{k \in I}\left[E_{k}(z)\right],
\end{equation*}
where $I$ is a finite index set of possible occurences, and $E_{k}(z)$ is the energy functional related to occurence $k$. The metric tensor related to occurence $k$ is denoted $G_{k}(z_{k})$ and is assumed positive definite. It follows that $\mathbb{E}_{k \in I}\left[G_{k}(z_{s,i})\right] = G\left(z_{s,i}\right)$.
\begin{lemma} \label{lemma:adaptive_convergence}
    Let $z_{s,i}^{(j)} = \left(x_{s,i}^{(j)}, u_{s,i}^{(j)}\right)$ be the feasible solution after iteration $j$ in \textit{GEORCE-FM}. If $z_{s,i}^{(j)}$ is not a local minimum point, then the feasible solution from the following iteration $z_{s,i}^{(j+1)}(k)$ for occurence $k$ will decrease the objective function in the sense that there exists an $\eta>0$ such that for all $0 < \alpha \leq \eta \leq 1$, then
    \begin{equation} \label{eq:adaptive_proof_eq1}
        \mathbb{E}_{k \in I}\left[E_{k}\left(x^{(j)}+\alpha\left(x^{(j+1)}(k)-x^{(j)}\right), u^{(j)} + \alpha \left(u^{(j+1)}(k)-u^{(j)}\right)\right)\right]-E\left(x^{(j)}, u^{(j)}\right] < 0,
    \end{equation}
    \begin{equation} \label{eq:adaptive_proof_eq2}
        \begin{split}
            &\mathbb{E}_{k \in I}\left[\Delta E_{k}\left(x^{(j)}, u^{(j)}\right)\right] \approx \sum_{s=0}^{T-1}\sum_{i=1}^{N}\left(\left\langle -2 \omega_{i}G\left(x_{s,i}^{(j)}, u_{s,i}^{(j)}\right)\left(\tilde{u}_{s,i}^{(j+1)}-u_{s,i}^{(j)}\right), \alpha\left(\tilde{u}_{s,i}^{(j+1)}-u_{s,i}^{(j)}\right)\right\rangle\right) \\
            &+ \sum_{s=0}^{T-1}\sum_{i=1}^{N}\left(\mathbb{E}_{k \in I}\left[\left\langle -2\omega_{i}G_{k}\left(x_{s,i}^{(j)}, u_{s,i}^{(j)}\right)\epsilon_{s,i}(k), \alpha \epsilon_{s,i}(k)\right\rangle\right]\right),
        \end{split}
    \end{equation}
    where $x^{(j)}[i] = \left(a_{i}, x_{s,i}^{(j)}, \dots, x_{T-1}^{(j)}, y\right)$, $u^{(j)}[i]=\left(u_{0,i}^{(j)}, u_{1,i}^{(j)}, \dots, u_{T-2,i}^{(j)}, u_{T-1,i}^{(j)}\right)$ and $\epsilon_{s,i}(k) = u_{s,i}^{(j+1)}(k)-\mathbb{E}_{k \in I}\left[u_{s,i}^{(j+1)}(k)\right]$ and $\tilde{u}_{s,i}^{(j+1)} = \mathbb{E}_{k \in I}\left[u_{s,i}^{(j+1)}(k)\right]$.
    Further, if the index set fulfills that
    \begin{equation} \label{eq:i_property}
        \sum_{s=0}^{T-1}\sum_{i=1}^{N}\mathbb{E}_{k \in I}\left[\left(u_{s,i}^{(j+1)}(k)-u_{s,i}^{(j)}\right)^{\top}\omega_{i}G\left(x_{s,i}^{(j)},u_{s,i}^{(j)}\right)\left(\hat{u}_{s,i}^{(j+1)}-u_{s,i}^{(j)}\right)\right] > 0,
    \end{equation}
    where $u_{s,i}^{(j)}$ is not a (local) minimum point, and $\hat{u}^{(j+1)}_{s,i}$ is the \textit{GEORCE-FM} iteration point, with no batching. If this property holds, then there exists an $\eta > 0$ such that for all $\alpha$ with $0 < \alpha \leq \eta \leq 1$, then
    \begin{equation} \label{eq:adaptive_proof_eq3}
        E\left(x^{(j)}+\alpha\left(\mathbb{E}_{k \in I}\left[x^{(j+1)}(k)\right]-x^{(j)}\right), u^{(i)}+\alpha\left(\mathbb{E}_{k \in I}\left[u^{(j+1)}(k)\right]-u^{(j)}\right)\right)-E\left(x^{j},u^{j}\right) < 0.
    \end{equation}
\end{lemma}
\begin{proof}
    Assume that $z_{s,i}^{(j)} = \left(x_{s,i}^{(j)}, u_{s,i}^{(j)}\right)$ is a feasible solution. For the occurrence $k$ then the change in the objective function after a \textit{GEORCE-FM} step with step size, $\alpha$, is by the proof of Lemma~\ref{lemma:global_conv_minimum}
    \begin{equation*}
        \begin{split}
            \Delta E_{k}\left(x^{(j)}, u^{(j)}\right) &= \sum_{s=0}^{T-1}\sum_{i=1}^{N}\left\langle -2\omega_{i}G_{k}\left(x_{s,i}^{(j)}, u_{s,i}^{(j)}\right)\Delta u_{s,i}^{(k)}, \alpha \Delta u_{s,i}(k) \right \rangle \\
            &+\sum_{s=0}^{T-1}\sum_{i=1}^{N}\left(\sum_{l=0}^{s-1}\mathcal{O}\left(\alpha \Delta u_{l,i}(k)\right)\norm{\sum_{l=0}^{s-1} \alpha \Delta u_{l,i}(k)}\right) \\
            &+ \sum_{i=1}^{N}\mathcal{O}\left(\alpha \Delta u_{s,i}(k)\right)\norm{\alpha \Delta u_{s,i}(k)}
        \end{split}
    \end{equation*}
    Thus, there exists an $\eta_{k}>0$ such that for all $0 < \alpha < \eta_{k}$, then $\frac{\Delta E_{k}\left(x^{(j)}, u^{(j)}\right)}{\alpha} < 0$, which implies that $\mathbb{E}_{k \in I}\left[\Delta E_{k}\left(x^{(j)}, u^{(j)}\right)\right] <0$. This proves eq.~\ref{eq:adaptive_proof_eq1}.

    Define $\epsilon_{s,i}=u_{s,i}^{(j+1)}(k) - \mathbb{E}_{k \in I}\left[u_{s,i}^{(j+1)}(k)\right]$. Then $\mathbb{E}_{k \in I}\left[\epsilon_{s,i}(k)\right] = 0$. To simplify notation, set $\tilde{u}_{s,i}^{(j+1)} := \mathbb{E}_{k \in I}\left[u_{s,i}^{(j+1)}\right]$. Applying this to the incremental change for \textit{GEORCE-FM} we see that
    \begin{equation*}
        \begin{split}
            &\mathbb{E}_{k \in I}\left[\Delta E_{k}\left(x^{(j)}, u^{(j)}\right)\right] \\
            &= \sum_{s=0}^{T-1}\sum_{i=1}^{N}\mathbb{E}_{k \in I}\left[\left\langle -2\omega_{i}G_{k}\left(x_{s,i}^{(j)},u_{s,i}^{(j)}\right)\left(\tilde{u}^{(j+1)}+\epsilon_{s,i}(k)-u_{s,i}^{(j)}\right), \alpha\left(\tilde{u}_{s,i}^{(j+1)}+\epsilon_{s,i}(k)-u_{s,i}^{(j)}\right)\right\rangle\right] \\
            &+ \sum_{s=0}^{T-1}\mathbb{E}_{k \in I}\left[\mathcal{O}\left(\sum_{l=0}^{s-1}\sum_{i=1}^{N}\alpha \Delta u_{l,k}(k)\right)\norm{\sum_{l=0}^{s-1}\sum_{i=1}^{N}\alpha \Delta u_{l,i}(k)}\right] \\
            &+ \sum_{s=0}^{T-1}\sum_{i=1}^{N}\mathcal{O}\left(\alpha \Delta u_{s,i}(k)\right)\norm{\alpha \Delta u_{s,i}(k)} \\
            &= \sum_{s=0}^{T-1}\sum_{i=1}^{N}\left(\left\langle -2\omega_{i}G\left(x_{s,i}^{(j)}, u_{s,i}^{(j)}\right)\left(\tilde{u}_{s,i}^{(j+1)}-u_{s,i}^{(j)}\right), \alpha \left(\tilde{u}^{(j+1)}-u_{s,i}^{(j)}\right) \right \rangle + \left\langle -2G_{k}\left(x_{s,i}^{(j)}, u_{s,i}^{(j)}\right)\epsilon_{s,i}(k),\alpha \epsilon_{s,i}(k)\right\rangle\right) \\
            &+ \sum_{s=0}^{T-1}\mathbb{E}_{k \in I}\left[\mathcal{O}\left(\sum_{l=0}^{s-1}\sum_{i=1}^{N}\alpha \Delta u_{l,i}(k)\right)\norm{\sum_{l=0}^{s-1}\sum_{i=1}^{N}\alpha \Delta u_{l,i}(k)}\right] \\
            &+ \sum_{s=0}^{T-1}\mathbb{E}_{k \in I}\left[\sum_{i=1}^{N} \mathcal{O}\left(\alpha \Delta u_{s,i}(k)\right)\norm{\alpha \Delta u_{s,i}(k)}\right],
        \end{split}
    \end{equation*}
    which proves eq.~\ref{eq:adaptive_proof_eq2}. The incremental change in the energy functional applying a \textit{GEORCE-FM} step is by the proof of Lemma~\ref{lemma:global_conv_minimum}
    \begin{equation*}
        \begin{split}
            \Delta E\left(x^{(j)}, u^{(j)}\right) &= \sum_{s=0}^{T-1}\sum_{i=1}^{N}\left\langle -2\omega_{i}G\left(x_{s,i}^{(j)}, u_{s,i}^{(j)}\right)\Delta u_{s,i}, \alpha \Delta u_{s,i}\right \rangle \\
            &+ \sum_{s=0}^{T-1}\sum_{i=1}^{N}\mathcal{O}\left(\sum_{l=0}^{s-1}\alpha u_{l,i}\right)\norm{\sum_{l=0}^{s-1}\alpha \Delta u_{l,i}} \\
            &+ \sum_{s=0}^{T-1}\sum_{i=1}^{N}\mathcal{O}\left(\alpha \Delta u_{s,i}\right)\norm{\alpha \Delta u_{s,i}}.
        \end{split}
    \end{equation*}
    Suppose the step is instead $\tilde{u}_{s,i}^{(j+1)}-u_{s,i}^{(j)}$, then the incremental energy function becomes
    \begin{equation*}
        \begin{split}
            \Delta E\left(x^{(j)}, u^{(j)}\right) &= \sum_{s=0}^{T-1}\sum_{i=1}^{N}\left\langle -2 \omega_{i}G\left(x_{s,i}^{(j)}, u_{s,i}^{(j)}\right)\Delta u_{s,i}, \alpha \left(\tilde{u}_{s,i}^{(j+1)}-u_{s,i}^{(j)}\right) \right \rangle \\
            &+ \sum_{s=0}^{T-1}\mathcal{O}\left(\sum_{l=0}^{s-1}\sum_{i=1}^{N}\alpha \left(\tilde{u}_{l,i}^{(j+1)}-u_{l,i}^{(j)}\right)\right)\norm{\sum_{l=0}^{s-1}\sum_{i=1}^{N}\alpha\left(\tilde{u}_{l,i}^{(j+1)}-u_{l,i}^{(j)}\right)} \\
            &+ \sum_{s=0}^{T-1}\sum_{i=1}^{N}\mathcal{O}\left(\alpha \left(\tilde{u}^{(j+1)}-u_{s,i}^{(j)}\right)\right)\norm{\alpha\left(\tilde{u}_{s,i}^{(j+1)}-u_{s,i}^{(j)}\right)}
        \end{split}
    \end{equation*}
    Using the property in eq.~\ref{eq:i_property}, then
    \begin{equation*}
        \begin{split}
            &\sum_{s=0}^{T-1}\sum_{i=0}^{N}\left\langle -2 \omega_{i}G\left(x_{s,i}^{(j)}, u_{s,i}^{(j)}\right)\Delta u_{s,i}, \alpha \left(\tilde{u}_{s,i}^{(j+1)}-u_{s,i}^{(j)}\right)  \right \rangle \\
            &= \sum_{s=0}^{T-1}\sum_{i=1}^{N}\mathbb{E}_{k \in I}\left[\left\langle -2\omega_{i}G\left(x_{s,i}^{(j)}, u_{s,i}^{(j)}\right)\left(\hat{u}_{s,i}^{(j+1)}-u_{s,i}^{(j)}\right), \alpha \left(u_{s,i}^{(j+1)}(k)-u_{s,i}^{(j)}\right) \right \rangle \right] \\
            &< 0.
        \end{split}
    \end{equation*}
    Since the quadratic forms under the summation are negative, then along the same lines as above, there exists an $\eta > 0$ such that for all $0 < \alpha < \eta \Delta$ then $E\left(x^{(j)}, u^{(j)}\right)<0$ with step $\tilde{u}_{s,i}^{(j+1)}-u_{s,i}^{(j)}$, which proves eq.~\ref{eq:adaptive_proof_eq3}.
\end{proof}

Note that eq.~\ref{eq:adaptive_proof_eq2} states that the average incremental change in $\mathbb{E}_{k \in I}\left[\Delta E_{k}\left(x^{(j)}, u^{(j)}\right)\right]$ in \textit{GEORCE-FM} steps over the index set $I$ has the same structure as for the energy functional, but the \textit{GEORCE-FM} step for the energy functional is replaced by the average \textit{GEORCE-FM} steps over the index set $I$- and an additional variance term.

Eq.~\ref{eq:i_property} is trivially fulfilled if the index set contains all data points. Eq.~\ref{eq:i_property} ensures that the steps generated by \textit{GEORCE-FM} for index $k$ on average over index set $I$ "points in the same direction" as the \textit{GEORCE-FM} step for the energy functional, i.e. that the scalar product between the average step and the \textit{GEORCE-FM} step for the energy fucntional is positive wrt. the metric tensor $G$.

\begin{proposition} \label{prop_adaptive_convergence}
    Assume Lemma~\ref{lemma:adaptive_convergence} is fulfilled and eq.~\ref{eq:i_property} holds. Then the series $\left\{E_{k \in I}\left[z^{(j)}(k)\right]\right\}_{j}$ will converge to a local minimum point in expectation.
\end{proposition}
\begin{proof}
    Since the energy functional has a lower bound of $0$, and the energy functional $\left\{E\left(\mathbb{E}_{k \in I}\left[z^{(j)}(k)\right]\right)\right\}_{j}$ is a decreasing series by eq.~\ref{eq:adaptive_proof_eq3}, then the series will converge. Assume that the series is converging to a point $z^{(c)}$ that is not a local minimum. Then there exists a \textit{GEORCE-FM} step $\Delta z^{(c)}(k)$ for the energy functionals $F_{k}$ at the point $z^{(c)}$ such that $\mathbb{E}_{k \in I}\left[E_{k}\left(z^{(c)}+\alpha_{1}\Delta z^{(c)}(k)\right)\right] < E\left(z^{(c)}\right)$ by eq.~\ref{eq:adaptive_proof_eq1}. It follows from eq.~\ref{eq:adaptive_proof_eq3} that $E\left(z^{(c)}+\alpha_{2}\mathbb{E}_{k \in I}\left[\Delta z^{(c)}(k)\right]\right) < E\left(z^{(c)}\right)$, which contradicts that $z^{(c)}$ is not a local minimum.
\end{proof}

%% file: General/appendix/proofs/finsler.tex
\begin{proposition}
    The update scheme for $u_{t},\mu_{t}$ and $x_{t}$ is
    \begin{equation}
        \begin{split}
            &y = W^{-1}V, \\
            &\mu_{T-1,i} = \left(\sum_{t=0}^{T-1}\tilde{G}_{t,i}^{-1}\right)^{-1}\left(2w_{i}(a_{i}-y)-\sum_{t=0}^{T-1}\tilde{G}_{t,i}^{-1}\left(\zeta_{t,i}+\sum_{t>j}^{T-1}\nu_{j,i}\right)\right), \quad i=1,\dots,N, \\
            &u_{t,i} = -\frac{1}{2w_{i}}\tilde{G}_{t,i}^{-1}\left(\mu_{T-1,i}+\zeta_{t,i}+\sum_{j>t}^{T-1}\nu_{j,i}\right), \quad t=0,\dots,T-1, \, i=1,\dots,N \\
            &x_{t+1,i} = x_{t,i}+u_{t,i}, \quad t=0,\dots,T-2, \, i=1,\dots,N, \\
            &x_{0,i}=a_{i} \quad i=1,\dots,N,
        \end{split}
    \end{equation}
    where
    \begin{equation*}
        \begin{split}
            W &= \sum_{i=1}^{N}w_{i}\left(\sum_{t=0}^{T-1}\tilde{G}_{t,i}^{-1}\right)^{-1}, \\
            V &= \sum_{i=1}^{N}w_{i}\left(\sum_{t=0}^{T-1}\tilde{G}_{ti}^{-1}\right)^{-1}a_{i}-\frac{1}{2}\sum_{i=1}^{N}\left(\sum_{t=0}^{T-1}\tilde{G}_{t,i}^{-1}\right)^{-1}\sum_{t=0}^{T-1}\tilde{G_{ti}}^{-1}\left(\zeta_{t,i}+\sum_{j>t}^{T-1}\nu_{j,i}\right).
        \end{split}
    \end{equation*}
    Here $\nu_{t,i} := \restr{\nabla_{y}u_{t,i}^{\top}\tilde{G}(y,u_{t,i})u_{t,i}}{y=x_{t,i}}$ and $\zeta_{t,i} := \restr{\nabla_{v}u_{t,i}^{\top}\tilde{G}(x_{t,i},v)u_{t,i}}{v=u_{t,i}}$.
\end{proposition}
\begin{proof}
    Consider the optimization problem
    \begin{equation*}
        \begin{split}
            \min_{(x_{t,i},u_{t,i})} E(x) &:= \min_{(x_{t,i},u_{t,i})}\left\{\sum_{i=1}^{N}w_{i}\sum_{t=0}^{T-1}u_{t,i}^{T}G(x_{t,i}, u_{t,i})u_{t,i}\right\} \\
            x_{t+1,i} &= x_{t,i}+u_{t,i}, \quad t=0,\dots,T-1, \, i=1,\dots,N, \\
            x_{0,i}&=a_{i},x_{T,i}=y, \quad i=1,\dots,N.
        \end{split}
    \end{equation*}
    For iteration $k$ we consider the following modified zero point problem derived using the same principle as in the Riemannian case.
    \begin{equation} \label{eq:finsler_eq_system}
        \begin{split}
            &\nu_{t,i}+\mu_{t,i} = \mu_{t-1,i}, \quad t=1,\dots,T-1, \, i=1,\dots,N, \\
            &2w_{i}G_{t,i}u_{t,i}+\zeta_{t,i}+\mu_{t,i} = 0, \quad t=0,\dots,T-1, \, i=1,\dots,N, \\
            &\sum_{t=0}^{T-1}u_{t,i}=y-a_{i}, \quad i=1,\dots,N, \\
            &\sum_{i=1}^{N}\mu_{T-1,i} = 0, \\
            &\nu_{t,i} := \restr{\nabla_{x}\left(w_{i}u_{t,i}^{\top}G(x)u_{t,i}\right)}{x=x_{t,i}^{(k)},u_{t,i}=u_{t,i}^{(k)}}, \quad t=1,\dots,T-1, \, i=1,\dots,N\\
            &\zeta_{t,i} := \restr{\nabla_{u}\left(w_{i}v^{\top}G(x,u)v\right)}{y=x_{t,i}^{(k)},v = u_{t,i}^{(k)}, u=u_{t,i}^{(k)}}, \quad t=1,\dots,T-1, \, i=1,\dots,N\\
            &G_{t,i} := G\left(x_{t,i}^{(k)}\right), \quad t=0,\dots,T-1.
        \end{split}
    \end{equation}
    From eq.~\ref{eq:finsler_eq_system} we get the following iterative scheme for updating $y$ and $\{x_{t,i}\}_{t,i}$
    By re-arranging the term we see that
    \begin{equation*}
        \begin{split}
            &\mu_{t,i} = \mu_{T-1,i} + \sum_{j>t}^{T-1} g_{jt}, \quad t=0,\dots,T-1,i=1,\dots,N, \\
            &u_{t,i} = -\frac{1}{2}G_{t,i}^{-1}\mu_{t,i} = -\frac{1}{2w_{i}}G_{t,i}^{-1}\left(\mu_{T-1,i}+\zeta_{t,i}+\sum_{j>t}^{T-1}\nu_{j,i}\right), \quad t=0,\dots,T-1,i=1,\dots,N, \\
            &-\frac{1}{2\omega_{i}}\sum_{t=0}^{T-1}G_{t,i}^{-1}\left(\mu_{T-1,i}+\zeta_{t,i}+\sum_{j>t}^{T-1}g_{ji}\right)=y-a_{i}, \quad i=1,\dots,N,
        \end{split}
    \end{equation*}
    where it is exploited that $G_{t,i}$ is positive definite for all $t,i$ and has a well defined inverse. These equations can be used to provide explicit solutions for the state and control variables. First observe that
    \begin{equation*}
        \mu_{T-1,i} = \left(\sum_{t=0}^{T-1}G_{t,i}^{-1}\right)^{-1}\left(2w_{i}(a_{i}-y)-\sum_{t=0}^{T-1}G_{t,i}^{-1}\left(\zeta_{t,i}+\sum_{j>t}^{T-1}\nu_{j,i}\right)\right), \quad i=1,\dots,N.
    \end{equation*}
    Applying this result, we deduce that
    \begin{equation*}
        \sum_{i=1}^{N}\mu_{T-1,i} = 0 \Leftrightarrow \sum_{i=1}^{N}\left(\sum_{t=0}^{T-1}G_{t,i}^{-1}\right)^{-1}\left(2w_{i}(a_{i}-y)-\sum_{t=0}^{T-1}G_{t,i}^{-1}\left(\zeta_{t,i}+\sum_{j>t}^{T-1}\nu_{j,i}\right)\right) = 0,
    \end{equation*}
    which implies that
    \begin{equation*}
        y = W^{-1}V,
    \end{equation*}
    where
    \begin{equation*}
        \begin{split}
            W &= \sum_{i=1}^{N}w_{i}\left(\sum_{t=0}^{T-1}G_{t,i}^{-1}\right)^{-1}, \\
            V &= \sum_{i=1}^{N}w_{i}\left(\sum_{t=0}^{T-1}G_{ti}^{-1}\right)^{-1}a_{i}-\frac{1}{2}\sum_{i=1}^{N}\left(\sum_{t=0}^{T-1}G_{t,i}^{-1}\right)^{-1}\sum_{t=0}^{T-1}G_{ti}^{-1}\left(\zeta_{t,i}+\sum_{j>t}^{T-1}g_{ji}\right).
        \end{split}
    \end{equation*}
    Thus, we get the update scheme in Proposition~\ref{prop:finsler_update_scheme}. 
\end{proof}

%% file: General/appendix/algorithms.tex
\subsubsection{GEORCE}

In algorithm~\ref{al:georce} we re-state the \textit{GEORCE} algorithm for Riemannian manifolds, and in algorithm~\ref{al:georcef} we re-state the \textit{GEORCE} algorithm for Finsler manifolds from \citep{georce}.

\begin{algorithm}[!ht]
    \caption{GEORCE for Riemannian Manifolds}
    \label{al:georce}
    \begin{algorithmic}[1]
        \State \textbf{Input}: $\mathrm{tol}$, $T$
        \State \textbf{Output}: Geodesic estimate $x_{0:T}$
        \State Set $x_{t}^{(0)} \leftarrow a+\frac{b-a}{T}t$, $u_{t}^{(0)} \leftarrow \frac{b-a}{T}$ for $t=0.,\dots,T$ and $i \leftarrow 0$
        \While $\norm{\restr{\nabla_{y}E(y)}{y=x_{t}^{(i)}}}_{2} > \mathrm{tol}$
        \State $G_{t} \leftarrow G\left(x_{t}^{(i)}\right)$ for $t=0,\dots,T-1$ \\
        \State $\nu_{t} \leftarrow \restr{\nabla_{y}\left(u_{t}^{(i)}G\left(y\right)u_{t}^{(i)}\right)}{y=x_{t}^{(i)}}$ for $t=,1\dots,T-1$
        \State $\mu_{T-1} \leftarrow \left(\sum_{t=0}^{T-1}G_{t}^{-1}\right)^{-1}\left(2(a-b)-\sum_{t=0}^{T-1}G_{t}^{-1}\sum_{t>j}^{T-1}\nu_{j}\right)$
        \State $u_{t} \leftarrow -\frac{1}{2}G_{t}^{-1}\left(\mu_{T-1}+\sum_{j>t}^{T-1}\nu_{j}\right)$ for $t=0,\dots,T-1$
        \State $x_{t+1} \leftarrow x_{t}+u_{t}$ for $t=0,\dots,T-1$
        \State Using line search find $\alpha^{*}$ for the following optimization problem with the discrete energy functional $E$
        \begin{equation*}
            \begin{split}
                \alpha^{*} = \argmin_{\alpha}\quad &E\left(x_{0:T}\right) \quad \text{(exact line search)} \\
                \text{s.t.} \quad &x_{t+1}=x_{t}+\alpha u_{t}+(1-\alpha)u_{t}^{(i)}, \quad t=0,\dots,T-1, \\
                &x_{0}=a.
            \end{split}
        \end{equation*}
        \State Set $u_{t}^{(i+1)} \leftarrow \alpha^{*}u_{t}+(1-\alpha^{*})u_{t}^{(i)}$ for $t=0,\dots,T-1$
        \State Set $x_{t+1}^{(i+1)} \leftarrow x_{t}^{(i+1)}+u_{t}^{(i+1)}$ for $t=0,\dots,T-1$
        \State $i \leftarrow i+1$
        \EndWhile
        \State return $x_{t}$ for $t=0,\dots,T-1$
    \end{algorithmic}
\end{algorithm}

\begin{algorithm}[hbt]
    \caption{GEORCE for Finsler Manifolds}
    \label{al:georcef}
    \begin{algorithmic}[1]
        \State \textbf{Input}: $\mathrm{tol}$, $T$
        \State \textbf{Output}: Geodesic estimate $x_{0:T}$
        \State Set $x_{t}^{(0)}\leftarrow a+\frac{b-a}{T}t$, $u_{t}^{(0)}\leftarrow \frac{b-a}{T}$ for $t=0.,\dots,T$  and $i \leftarrow 0$
        \While $\norm{\restr{\nabla_{y}E(y)}{y=x_{t}^{(i)}}}_{2} > \mathrm{tol}$
        \State $G_{t} \leftarrow G\left(x_{t}^{(i)}, u_{t}^{(i)}\right)$ for $t=0,\dots,T-1$
        \State $\nu_{t} \leftarrow \restr{\nabla_{y}\left(u_{t}^{(i)}G\left(y, u_{t}^{(i)}\right)u_{t}^{(i)}\right)}{y=x_{t}^{(i)}}$ for $t=1,\dots,T-1$
        \State $h_{t} \leftarrow \restr{\nabla_{y}\left(u_{t}^{(i)}G\left(x_{t}^{(i)}, y\right)u_{t}^{(i)}\right)}{y=u_{t}^{(i)}}$ for $t=1,\dots,T-1$
        \State $\mu_{T-1} \leftarrow \left(\sum_{t=0}^{T-1}G_{t}^{-1}\right)^{-1}\left(2(a-b)-\sum_{t=0}^{T-1}G_{t}^{-1}\left(h_{t}+\sum_{t>j}^{T-1}\nu_{j}\right)\right)$
        \State $u_{t} \leftarrow -\frac{1}{2}G_{t}^{-1}\left(\mu_{T-1}+h_{t}+\sum_{j>t}^{T-1}\nu_{j}\right)$ for $t=0,\dots,T-1$
        \State $x_{t+1} \leftarrow x_{t}+u_{t}$ for $t=0,\dots,T-1$
        \State Using line search find $\alpha^{*}$ for the following optimization problem with the discrete energy functional $E_{F}$
        \begin{equation*}
            \begin{split}
                \alpha^{*} = \argmin_{\alpha}\quad &E_{F}\left(x_{0:T}\right) \quad \text{(exact line-search)} \\
                \text{s.t.} \quad &x_{t+1}=x_{t}+\alpha u_{t}+(1-\alpha)u_{t}^{(i)}, \quad t=0,\dots,T-1, \\
                &x_{0}=a.
            \end{split}
        \end{equation*}
        \State Set $u_{t}^{(i+1)} \leftarrow \alpha^{*}u_{t}+(1-\alpha^{*})u_{t}^{(i)}$ for $t=0,\dots,T-1$
        \State Set $x_{t+1}^{(i+1)}\leftarrow x_{t}^{(i+1)}+u_{t}^{(i+1)}$ for $t=0,\dots,T-1$
        \State $i \leftarrow i+1$
        \EndWhile
        \State return $x_{t}$ for $t=0,\dots,T-1$
    \end{algorithmic}
\end{algorithm}

\subsubsection{GEORCE-FM for Finsler}

\begin{algorithm}[hbt]
    \caption{GEORCE-FM for Finsler Manifolds}
    \label{al:georcef_fm}
    \begin{algorithmic}[1]
        \State \textbf{Input}: $\mathrm{tol}$, $a_{1:N,i}$, $T$.
        \State \textbf{Output}: Geodesic estimate $x_{0:T}$.
        \State Set $y^{(0)} \leftarrow a_{0}$, $x_{t,i}^{(0)} \leftarrow a_{i}+\frac{y^{(0)}-a_{i}}{T}t$ and  $u_{t,i}^{(0)} \leftarrow \frac{y^{(0)}-a_{i}}{T}$ for $t=0.,\dots,T$ and $i=1,\dots,N$.
        \While stop criteria > $\mathrm{tol}$ 
        \State $G_{t,i} \leftarrow G\left(x_{t,i}^{(k)}\right)$ for $t=0,\dots,T-1$ and $i=1,\dots,N$.
        \State $g_{t,i} \leftarrow \restr{\nabla_{x}\left(u_{t,i}^{(k)}G\left(x\right)u_{t,i}^{(k)}\right)}{x=x_{t,i}^{(k)}}$ for $t=1,\dots,T-1$ and $i=1,\dots,N$.
        \State $y \leftarrow W^{-1}V$ using Eq.~\ref{eq:finsler_energy_update_schem}.
        \State $\mu_{T-1,i} \leftarrow \left(\sum_{t=0}^{T-1}G_{t,i}^{-1}\right)^{-1}\left(2w_{i}(a_{i}-y)-\sum_{t=0}^{T-1}G_{t,i}^{-1}\sum_{t>j}^{T-1}g_{j,i}\right)$ for $i=1,\dots,N$ and $t=1,\dots,T-1$.
        \State $u_{t,i} \leftarrow -\frac{1}{2w_{i}}G_{t,i}^{-1}\left(\mu_{T-1,i}+\sum_{j>t}^{T-1}g_{j,i}\right)$ for $t=0,\dots,T-1$ and  $i=1,\dots,N$.
        \State $x_{t+1,i} \leftarrow x_{t,i}+u_{t,i}$ for $t=0,\dots,T-2$ and $i=1,\dots,N$.
        \State Using line search find $\alpha^{*}$ for the following optimization problem for the discrete sum of energy $E$
        \begin{equation*}
            \begin{split}
                \min_{\alpha}\quad &E\left(x_{0:T,1:N}, \tilde{u}_{0:T,1:N}\right) \quad \text{(exact line search)} \\
                \text{s.t.} \quad &x_{t+1,i}=x_{t,i}+\alpha \tilde{u}_{t,i}+(1-\alpha)u_{t,i}^{(k)}, \quad t=0,\dots,T-1, \, i=1,\dots,N. \\
                &\tilde{u}_{t,i} = \alpha u_{t,i}+(1-\alpha)u_{t,i}^{(k)}, \quad t=0,\dots,T-1, \, i=1,\dots,N. \\
                &x_{0,i}=a_{i}.
            \end{split}
        \end{equation*} \\
        \State Set $u_{t,i}^{(k+1)} \leftarrow \alpha^{*}u_{t,i}+(1-\alpha^{*})u_{t,i}^{(k)}$ for $t=0,\dots,T-1$ and $i=1,\dots,N$.
        \State Set $x_{t+1,i}^{(k+1)} \leftarrow x_{t,i}^{(k+1)}+u_{t,i}^{(k+1)}$ for $t=0,\dots,T-1$ and $i=1,\dots,N$.
        \EndWhile
        \State return $x_{t,i}$ for $t=0,\dots,T-1$ for $i=1,\dots,N$. \\
    \end{algorithmic}
\end{algorithm}

%% file: General/appendix/manifolds.tex
In this section we give a brief background to the different manifolds used in the experiments in the paper. The manifolds are the same that can be found in \citep{georce}, and we re-state the definitions here.

\subsection{Riemannian Manifolds}

The Riemannian manifolds used in the paper are the same as in \citep{georce}, and we re-state the definitions from \citep{georce} in Table~\ref{tab:manifold_description}.

\begin{table}[ht]
    \centering
    \scriptsize
    \begin{tabular}{p{2cm} | p{5cm} | p{2cm} | p{4cm}}
    \textbf{Manifold} & \textbf{Description} & \textbf{Parameters} & \textbf{Local Coordinates} \\
    \hline
    $\mathbb{S}^{n}$ & The n-sphere, $\{x \in \mathbb{R}^{n+1} \;|\; ||x||_{2}=1\}$ & - & Stereographic coordinates \\
    \hline
    $E(n)$ & The n-Ellipsoid, $\{x \in \mathbb{R}^{n+1} \;|\; ||p\odot x||_{2}=1\}$ & Half-axes: $p=\left(0.5,\dots,1\right)$ equally distributed points. & Stereographic coordinates \\
    \hline
    $\mathbb{T}^{2}$ & Torus & Major radius $R=3.0$ and minor radius $r=1.0$ & Parameterized by $\left(\theta,\phi\right)$ with $x=(R+r\cos \theta)\cos\phi, y=(R+r\cos \theta) \sin \phi, z=r \sin \theta$. \\
    \hline
    $\mathbb{H}^{2}$ & 2 dimensional Hyperbolic Space embedded into Minkowski pace $\mathbb{R}^{3}$ & - & Parameterized by $\left(\alpha,\beta\right)$ with $x=\cosh\alpha, y=\sinh\alpha\cos\beta, z=\sinh\alpha\sin\beta$. \\
    \hline
    Paraboloid & Paraboloid of dimension $n$ & - & Parameterized by standard coordinates by $\left(x^{1},\dots,x^{n},\sum_{i=1}^{n}x_{i}^{2}\right)$ \\
    \hline
    Hyperbolic Paraboloid & Hyperbolic Paraboloid of dimension $2$ & - & Parameterized by standard coordinates by $\left(x^{1},\dots,x^{n},x_{1}^{2}-x_{2}^{2}\right)$ \\
    \hline
    $\mathcal{P}(n)$ & Symmetric positive definite matrices of size $n^{2}$ & - & Embedded into $\mathbb{R}^{n^{2}}$ by the mapping $f(x) = l(x)l(x)^{T}$, where $l: \mathbb{R}^{n(n+1)/2} \rightarrow \mathbb{R}^{n \times n}$ maps $x$ into a lower triangle matrix consisting of the elements in $x$. \\
    \hline
    Gaussian Distribution & Parameters of the Gaussian distribution equipped with the Fischer-Rao metric & - & Parametrized by $(\mu,\sigma) \in \mathbb{R} \times \mathbb{R}_{+}$. \\
    \hline
    Fr\'echet Distribution & Parameters of the Fr\'echet distribution equipped with the Fischer-Rao metric & - & Parametrized by $(\beta,\lambda) \in \mathbb{R}_{+} \times \mathbb{R}_{+}$. \\
    \hline
    Cauchy Distribution & Parameters of the Cauchy distribution equipped with the Fischer-Rao metric & - & Parametrized by $(\mu,\sigma) \in \mathbb{R} \times \mathbb{R}_{+}$. \\
    \hline
    Pareto Distribution & Parameters of the Pareto distribution equipped with the Fischer-Rao metric & - & Parametrized by $(\theta,\alpha) \in \mathbb{R}_{+} \times \mathbb{R}_{+}$. \\
    \hline
    VAE MNIST & Variational-Autoencoder equipped with the pull-back metric for MNIST-data & - & Parametrized by the encoded latent space. \\
    \hline
    VAE CelebA & Variational-Autoencoder equipped with the pull-back metric for CelebA-data & - & Parametrized by the encoded latent space. \\
    \hline
    \end{tabular}
    \caption{Description of Riemannian Manifolds.}
    \label{tab:manifold_description}
\end{table}

\subsection{Finsler Manifolds}

Let $\mathcal{M}$ be a differentiable manifold equipped with a Riemannian metric. Consider a force field $f: \mathcal{M} \rightarrow T_{x}\mathcal{M}$ acting on the manifold. If a microswimmer wants to find the shortest curve on the manifold, $\mathcal{M}$, under influence of $f$ with constrant velocity $v$, then this corresponds to finding geodesics on a Randers metric \citep{Piro_2021}, which is a special case of a Finsler manifold, where

\begin{equation}
    \begin{split}
        F &= \sqrt{a_{ij}\dot{r}^{i}\dot{r}^{j}}+b_{i}\dot{r}^{i} \\
        a_{ij} &= g_{ij} \lambda + f_{i}f_{j}\lambda^{2} \\
        b_{i} &= -f_{i}\lambda \\
        f_{i} &= g_{ij}f^{j} \\
        \lambda^{-1} &= v_{0}^{2}-g_{ij}f^{i}f^{j},
    \end{split}
\end{equation}

where $g_{ij}$ denotes the elements of the Riemannian background metric, $f^{j}$ denotes the $j$th element of the force field, while $v_{0}$ denotes the initial velocity Riemannian length of the microswimmer with constant velocity $v$ on the surface. The Finsler metric is well defined if and only if $||f||_{g} < v_{0}$ \citep{Piro_2021}.

\subsection{Variational-Autoencoders} 

The Variational-Autoencoders (VAE) \citep{kingma2022autoencodingvariationalbayes} is generative model that learns the data manifold, where the manifold can be equipped with the pull-back metric. A variational-autoencoder consists of an encoder, $h_{\Phi}: \mathcal{X} \subset \mathbb{R}^{D} \rightarrow \mathcal{Z} \subset \mathbb{R}^{d}$ that encodes the data $x \in \mathcal{R}^{D}$ into a latent representation $z \in \mathcal{Z}$ with $d < D$ parameters $\Phi$. The data is then decoded by $f_{\theta}: \mathcal{Z} \rightarrow \mathcal{X}$. The parameters, $\{\Phi, \theta\}$, are estimated by minimizing the Evidence Lower Bound (ELBO) \citep{kingma2022autoencodingvariationalbayes}. The VAE can be seen as a Riemannian manifold with one chart with the Riemannian pull-back metric \citep{arvanitidis2021latent, shao2017riemannian}
\begin{equation*}
    G(z) = J_{f}^{\top}J_{f},
\end{equation*}
where $J_{f}$ denotes the Jacobian of $f$. We train a VAE for the MNIST data \citep{deng2012mnist} and CelebA data \citep{liu2015faceattributes}. We use the exact same architecture as in \citep{georce}, which is a modification of the architecture in \citep{shao2017riemannian}. We re-state the architecture from \citep{georce} in Table~\ref{tab:mnist_vae_architecture} and Table~\ref{tab:celeba_vae_architecture}.

\begin{table}[!ht]
    \centering
    \begin{tabular}{c}
        \hline
        \multicolumn{1}{c}{\textbf{Encoder}} \\
        \hline
        $\mathrm{Conv}\left(\mathrm{output\_channels}=64, \mathrm{kernel\_shape}=4 \times 4, \mathrm{stride}=2, \mathrm{bias}=\mathrm{False}\right)$ \\ 
        $\mathrm{GELU}$ \\
        $\mathrm{Conv}\left(\mathrm{output\_channels}=64, \mathrm{kernel\_shape}=4 \times 4, \mathrm{stride}=2, \mathrm{bias}=\mathrm{False}\right)$ \\ 
        $\mathrm{GELU}$ \\
        $\mathrm{Conv}\left(\mathrm{output\_channels}=64, \mathrm{kernel\_shape}=4 \times 4, \mathrm{stride}=1, \mathrm{bias}=\mathrm{False}\right)$ \\
        $\mathrm{GELU}$ \\
        $\mathrm{Conv}\left(\mathrm{output\_channels}=64, \mathrm{kernel\_shape}=4 \times 4, \mathrm{stride}=2, \mathrm{bias}=\mathrm{False}\right)$ \\
        $\mathrm{GELU}$ \\
        \hline
        \multicolumn{1}{c}{\textbf{Latent Parameters}} \\
        \hline
        $\mu$: $\mathrm{Linear}(\mathrm{out\_feature}=8, \mathrm{bias}=\mathrm{True})$, $\mathrm{Identity}$ \\
        $\sigma$: $\mathrm{Linear}(\mathrm{out\_feature}=8, \mathrm{bias}=\mathrm{True})$, $\mathrm{Sigmoid}$ \\
        \hline
        \multicolumn{1}{c}{\textbf{Decoder}} \\
        \hline
        $\mathrm{Linear}\left(\mathrm{out\_feature}=50, \mathrm{bias}=\mathrm{True}\right)$ \\
        $\mathrm{GELU}$ \\
        $\mathrm{Conv2dTransposed}\left(\mathrm{output\_channels}=64, \mathrm{kernel\_shape}=4 \times 4, \mathrm{stride}=2, \mathrm{bias}=\mathrm{False}\right)$ \\
        $\mathrm{GELU}$ \\
        $\mathrm{Conv2dTransposed}\left(\mathrm{output\_channels}=32, \mathrm{kernel\_shape}=4 \times 4, \mathrm{stride}=1, \mathrm{bias}=\mathrm{False}\right)$ \\
        $\mathrm{GELU}$ \\
        $\mathrm{Conv2dTransposed}\left(\mathrm{output\_channels}=16, \mathrm{kernel\_shape}=4 \times 4, \mathrm{stride}=1, \mathrm{bias}=\mathrm{False}\right)$ \\
        $\mathrm{GELU}$ \\
        $\mathrm{Linear}\left(\mathrm{out\_feature}=784, \mathrm{bias}=\mathrm{True}\right)$ \\
        \hline
    \end{tabular}
    \caption{Architecture for MNIST}
    \label{tab:mnist_vae_architecture}
\end{table}

\begin{table}[!ht]
    \centering
    \begin{tabular}{c}
        \hline
        \multicolumn{1}{c}{\textbf{Encoder}} \\
        \hline
        $\mathrm{Conv}\left(\mathrm{output\_channels}=32, \mathrm{kernel\_shape}=4 \times 4, \mathrm{stride}=2, \mathrm{bias}=\mathrm{False}\right)$ \\ 
        $\mathrm{GELU}$ \\
        $\mathrm{Conv}\left(\mathrm{output\_channels}=32, \mathrm{kernel\_shape}=4 \times 4, \mathrm{stride}=2, \mathrm{bias}=\mathrm{False}\right)$ \\ 
        $\mathrm{GELU}$ \\
        $\mathrm{Conv}\left(\mathrm{output\_channels}=64, \mathrm{kernel\_shape}=4 \times 4, \mathrm{stride}=2, \mathrm{bias}=\mathrm{False}\right)$ \\
        $\mathrm{GELU}$ \\
        $\mathrm{Conv}\left(\mathrm{output\_channels}=64, \mathrm{kernel\_shape}=4 \times 4, \mathrm{stride}=2, \mathrm{bias}=\mathrm{False}\right)$ \\
        $\mathrm{GELU}$ \\
        \hline
        \multicolumn{1}{c}{\textbf{Latent Parameters}} \\
        \hline
        $\mu$: $\mathrm{Linear}(\mathrm{out\_feature}=32, \mathrm{bias}=\mathrm{True})$, $\mathrm{Identity}$ \\
        $\sigma$: $\mathrm{Linear}(\mathrm{out\_feature}=32, \mathrm{bias}=\mathrm{True})$, $\mathrm{Sigmoid}$ \\
        \hline
        \multicolumn{1}{c}{\textbf{Decoder}} \\
        \hline
        $\mathrm{Conv2dTransposed}\left(\mathrm{output\_channels}=64, \mathrm{kernel\_shape}=4 \times 4, \mathrm{stride}=2, \mathrm{bias}=\mathrm{False}\right)$ \\
        $\mathrm{GELU}$ \\
        $\mathrm{Conv2dTransposed}\left(\mathrm{output\_channels}=64, \mathrm{kernel\_shape}=4 \times 4, \mathrm{stride}=2, \mathrm{bias}=\mathrm{False}\right)$ \\
        $\mathrm{GELU}$ \\
        $\mathrm{Conv2dTransposed}\left(\mathrm{output\_channels}=32, \mathrm{kernel\_shape}=4 \times 4, \mathrm{stride}=2, \mathrm{bias}=\mathrm{False}\right)$ \\
        $\mathrm{GELU}$ \\
        $\mathrm{Conv2dTransposed}\left(\mathrm{output\_channels}=32, \mathrm{kernel\_shape}=4 \times 4, \mathrm{stride}=2, \mathrm{bias}=\mathrm{False}\right)$ \\
        $\mathrm{GELU}$ \\
        $\mathrm{Conv2dTransposed}\left(\mathrm{output\_channels}=3, \mathrm{kernel\_shape}=4 \times 4, \mathrm{stride}=4, \mathrm{bias}=\mathrm{False}\right)$ \\
        \hline
    \end{tabular}
    \caption{Architecture for CelebA}
    \label{tab:celeba_vae_architecture}
\end{table}

%% file: General/appendix/data_methods.tex
\subsection{Data}

We apply \textit{GEORCE-FM} to the following real-world datasets
\begin{itemize}
    \item MNIST data \citep{deng2012mnist}: We train a VAE for MNIST, which consists of $28 \times 28$ images of handwritten digits between 0 and 9. The latent space is $8$ dimensional. We use $60,000$ images for training the VAE.
    \item CelebA data \citep{liu2015faceattributes}: We train a VAE for CelebA dataset, which consists of $202,599$ images of famous people of size $64 \times 64 \times 3$, where we use approximately $80\%$ for training similar to \citep{shao2017riemannian}.
\end{itemize}

In table~\ref{tab:manifold_data} we summarize the data generation, which is done in local coordinates except for the VAE's.

\begin{table}[!ht]
    \centering
    \scriptsize
    \begin{tabular}{p{2cm} | p{7cm}}
    \textbf{Manifold} & \textbf{Data Generation} \\
    \hline
    $\mathbb{S}^{n}$ & $N$ normally distributed points with variance $1$ and mean of $n$ equally spaced points between $0$ and $1$ \\
    \hline
    $E(n)$ & $N$ normally distributed points with variance $1$ and mean of $n$ equally spaced points between $0.5$ and $1$ (end point is excluded) \\
    \hline
    $\mathbb{T}^{2}$ & $N$ normally distributed points with variance $1$ and mean $(0,0)$ \\
    \hline
    $\mathbb{H}^{2}$ & $N$ normally distributed points with variance $1$ and mean $(0,0)$ \\
    \hline
    Paraboloid & $N$ normally distributed points with variance $0.1$ and mean $(1,1)$  \\
    \hline
    Hyperbolic Paraboloid & $N$ normally distributed points with variance $0.1$ and mean $(1,1)$  \\
    \hline
    $\mathcal{P}(n)$ & $N$ normally distributed points with variance $1.0$ and mean $10 I$, where $I$ denotes the identity matrix. \\
    \hline
    Gaussian Distribution & $N/2$ normally with variance $.1$ and mean $\left(-1.0,0.5\right)$ and $N/2$ normally with variance $0.1$ and mean $\left(1.0, 1.0\right)$ \\
    \hline
    Fréchet Distribution & $N/2$ normally with variance $.1$ and mean $\left(0.5,0.5\right)$ and $N/2$ normally with variance $0.1$ and mean $\left(1.0, 1.0\right)$  \\
    \hline
    Cauchy Distribution & $N/2$ normally with variance $.1$ and mean $\left(-1.0,0.5\right)$ and $N/2$ normally with variance $0.1$ and mean $\left(1.0, 1.0\right)$  \\
    \hline
    Pareto Distribution & $N/2$ normally with variance $.1$ and mean $\left(0.5,0.5\right)$ and $N/2$ normally with variance $0.1$ and mean $\left(1.0, 1.0\right)$  \\
    \hline
    MNIST & $N$ randomly selected images from the MNIST dataset. \\
    \hline
    CelebA & $N$ randomly selected images from the CelebA dataset. \\
    \end{tabular}
    \caption{The start and end point for the geodesics estimated for each manifold.}
    \label{tab:manifold_data}
\end{table}

\subsection{Methods and Hyper-Parameters}

In table~\ref{tab:hyper_parameters} we provide a description of the methods used in the paper, and what the hyper-parameters have been set to.

\begin{table}[!ht]
    \centering
    \scriptsize
    \begin{tabular}{p{2cm} |p{4cm} | p{4cm}}
    \textbf{Method} & \textbf{Description} & \textbf{Parameters} \\
    \hline
    \textit{GEORCE-FM} & See algorithm in main paper & Backtracking: $\rho=0.5$. \\
    \hline
    \textit{ADAM} \citep{kingma2017adam} & Moments based gradient descent method & $\alpha=0.01$ (step size), $\beta_{1}=0.9$, $\beta_{2}=0.999$, $\epsilon=10^{-8}$ \\
    \hline
    \textit{SGD} \citep{ruder2017overviewgradientdescentoptimization} & Stochastic gradient descent & $\gamma=0.01$ (step size) \\
    \textit{RMSprop Momentum} & Stochastic gradient descent & $\alpha=0.01$ (step size), $\mathrm{momentum}=0.9$., $\epsilon=10^{-8}$. \\
    \hline 
    \textit{RMSprop} \citep{ruder2017overviewgradientdescentoptimization} & Moments based gradient descent method  & $\alpha=0.01$ (step size), $\gamma=0.9$, $\epsilon=10^{-8}$ \\
    \hline
    \textit{Adamax} \citep{kingma2017adam} & Corresponds to \textit{ADAM} using the infinity norm in \textit{ADAM} & $\alpha=0.01$ (step size), $\beta_{1}=0.9$, $\beta_{2}=0.999$, $\epsilon=10^{-8}$. \\
    \hline
    \textit{Adagrad} \citep{duchi2011adaptive} & Stochastic gradient method with adaptive updates of the step size & $\alpha=0.01$ (step size), $\mathrm{momentum}=0.9$. \\
    \hline
    \end{tabular}
    \caption{The Hyper-parameters for estimating the Fr\'echet mean.}
    \label{tab:hyper_parameters}
\end{table}

%% file: General/appendix/experiments/intro.tex
In this section we provide details on hardware and additional experiments in the paper.

%% file: General/appendix/experiments/hardware.tex
All figures have been computed on a \textit{HP} computer with Intel Core i9-11950H 2.6 GHz 8C, 15.6'' FHD, 720P CAM, 32 GB (2$\times$16GB) DDR4 3200 So-Dimm, Nvidia Quadro TI2004GB Discrete Graphics, 1TB PCle NVMe SSD, backlit Keyboard, 150W PSU, 8cell, W11Home 64 Advanced, 4YR Onsite NBD.

The runtime estimates and the training of the VAE's have been computed on a GPU for at most 24 hours with a maximum memory of $10$ GB. The $GPU$ consists of $4$ nodes on a \textit{Tesla V100}.

%% file: General/appendix/experiments/additional_experiments.tex
\subsubsection{Fr\'echet Mean for VAE}

In Fig.~\ref{fig:mnist_riemannian_8} we show the estimated Fr\'echet mean and geodesics using $10$ reconstructed with a VAE for the pull-back metric of the decoder using \textit{GEORCE-FM}. In Fig.~\ref{fig:mnist_finsler_8} and Fig.~\ref{fig:celeba_finsler_32} we show a similar plot, but where the pull-back metric of the decoder is a background with the generic force field in eq.~\ref{eq:generic_forcefield} acting on the manifold, and thus converting the navigation problem into finding geodesics under a Finsler metric as described in Appendix~\ref{ap:manifold_description}.

\begin{figure}[t!]
    \centering
    \includegraphics[width=1.0\textwidth]{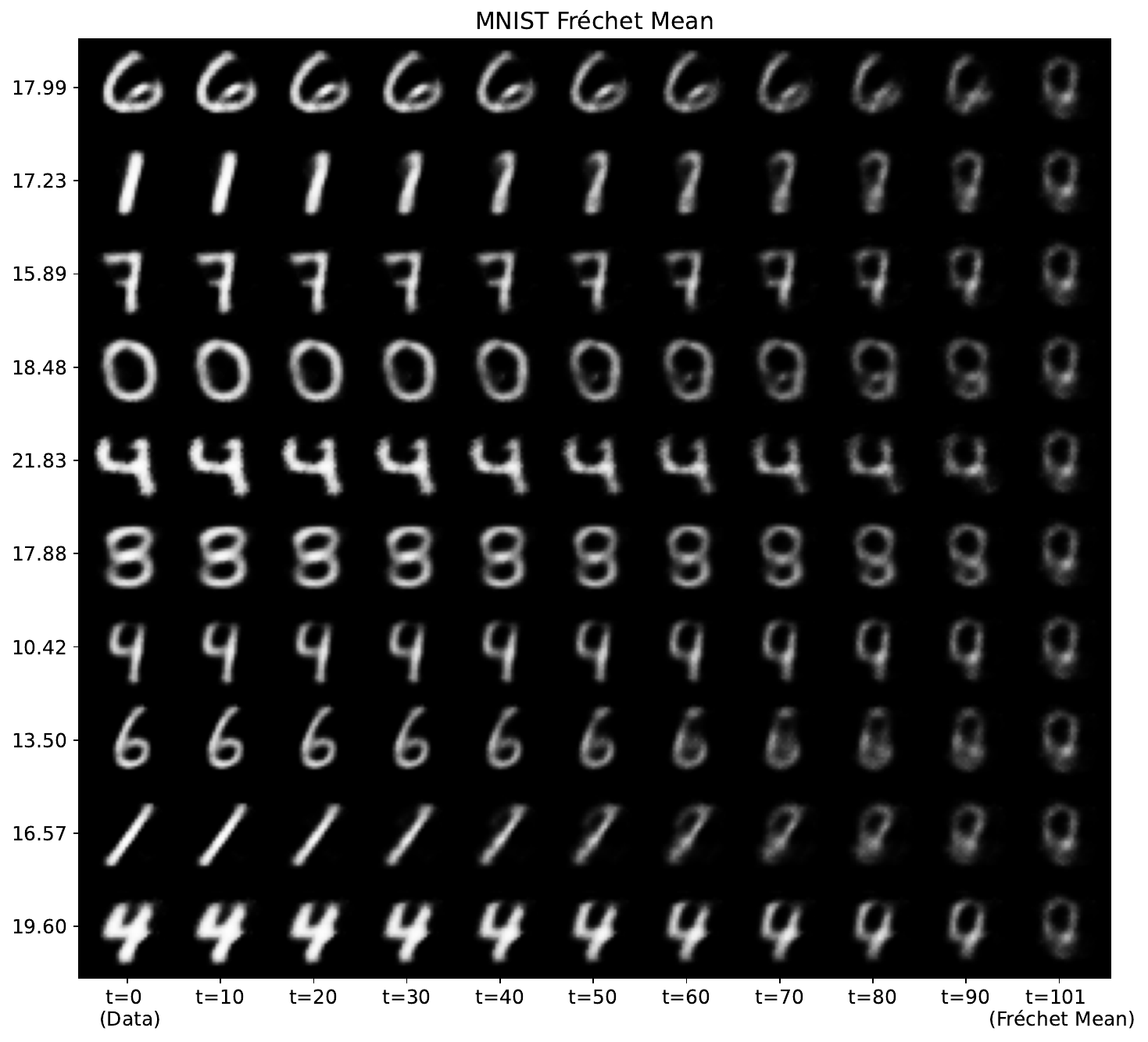}
    \caption{The application of \textit{GEORCE-FM} to manifold learning for a VAE estimating the Fr\'echet mean for $10$ images of the MNIST dataset reconstructed using a VAE \citep{deng2012mnist}. Each row shows $10$ points on the geodesic, while the left most image is the data and the right most image is the estimated Fr\'echet mean using \textit{GEORCE-FM}. The length of the estimated geodesic using \textit{GEORCE-FM} is shown to the left, while the grid point number is shown at the bottom.}
    \label{fig:mnist_riemannian_8}
    \vspace{-1.5em}
\end{figure}

\begin{figure}[t!]
    \centering
    \includegraphics[width=1.0\textwidth]{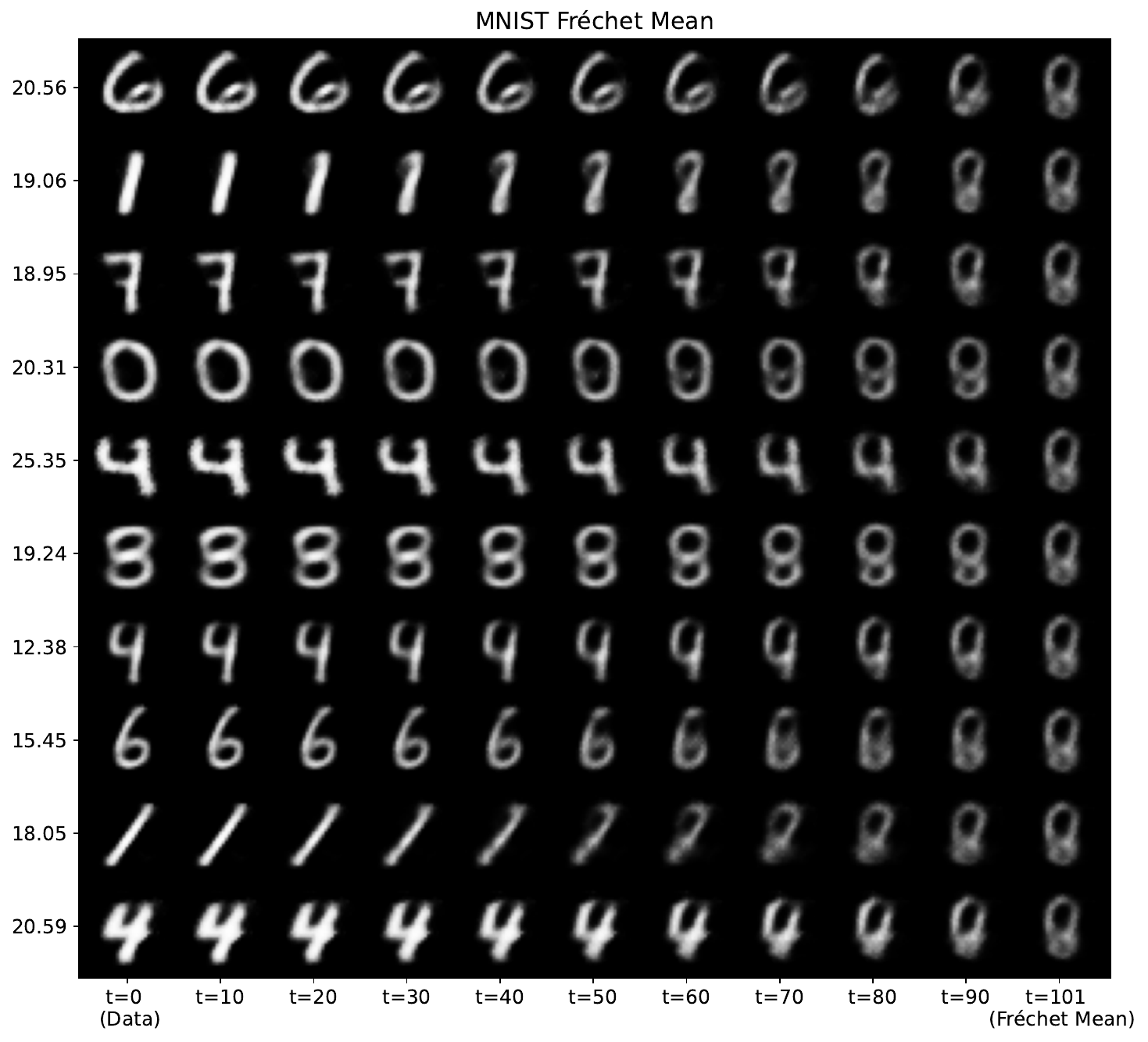}
    \caption{The application of \textit{GEORCE-FM} to manifold learning for a VAE estimating the Fr\'echet mean for $10$ images of the MNIST dataset reconstructed using a VAE \citep{deng2012mnist}. The learned manifold is equipped with the generic force field in eq,~\ref{eq:generic_forcefield} such that it is a Finsler manifold as described in Appendix~\ref{ap:manifold_description}. Each row shows $10$ points on the geodesic, while the left most image is the data and the right most image is the estimated Fr\'echet mean using \textit{GEORCE-FM}. The length of the estimated geodesic using \textit{GEORCE-FM} is shown to the left, while the grid point number is shown at the bottom.}
    \label{fig:mnist_finsler_8}
    \vspace{-1.5em}
\end{figure}

\begin{figure}[t!]
    \centering
    \includegraphics[width=1.0\textwidth]{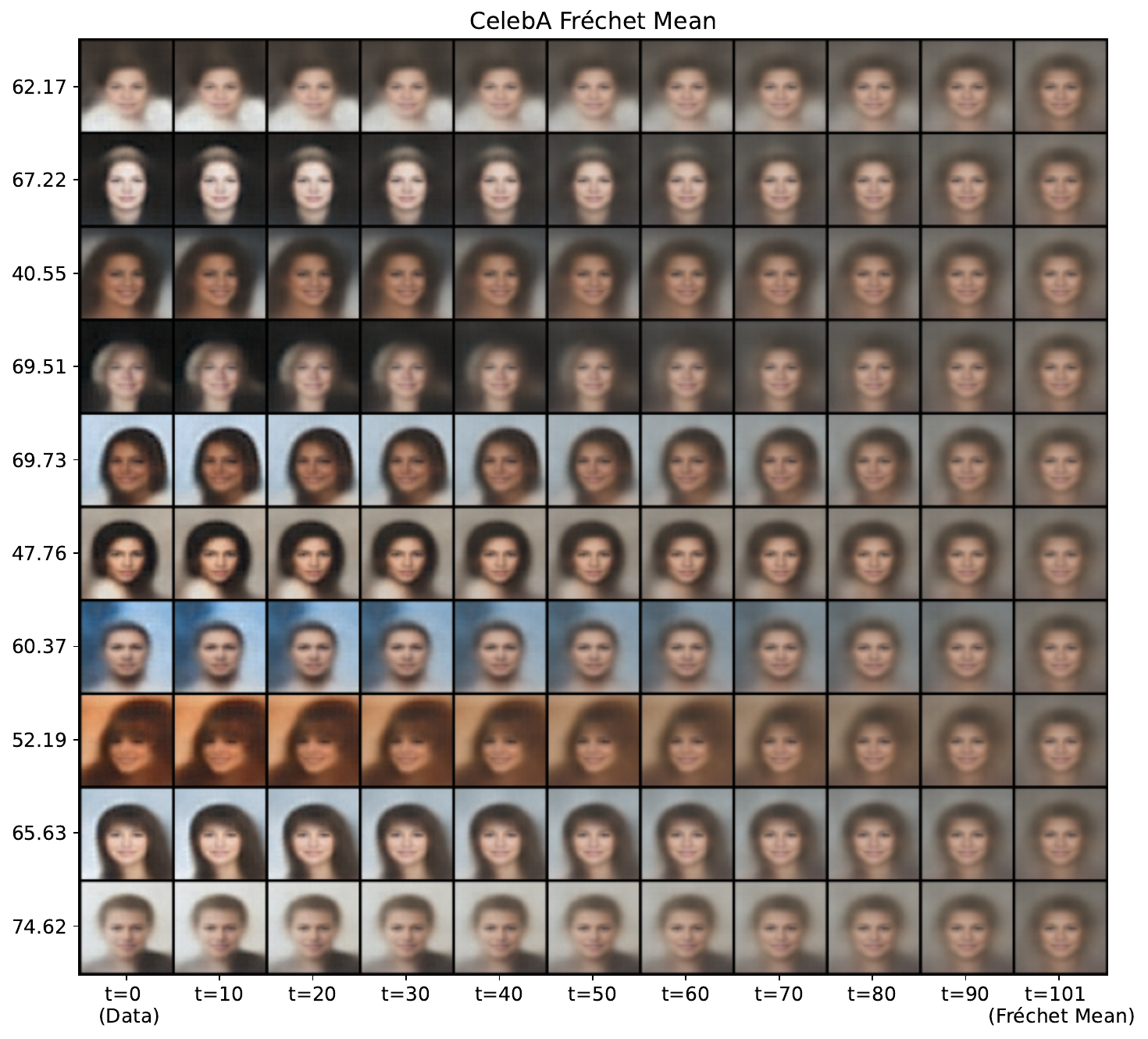}
    \caption{The application of \textit{GEORCE-FM} to manifold learning for a VAE estimating the Fr\'echet mean for $10$ images of the CelebA dataset reconstructed using a VAE \citep{liu2015faceattributes}. The learned manifold is equipped with the generic force field in eq,~\ref{eq:generic_forcefield} such that it is a Finsler manifold as described in Appendix~\ref{ap:manifold_description}. row shows $10$ points on the geodesic, while the left most image is the data and the right most image is the estimated Fr\'echet mean using \textit{GEORCE-FM}. The length of the estimated geodesic using \textit{GEORCE-FM} is shown to the left, while the grid point number is shown at the bottom.}
    \label{fig:celeba_finsler_32}
    \vspace{-1.5em}
\end{figure}

\subsubsection{Additional Runtime Estimates Riemannian Manifolds}

\begin{sidewaystable}
    \centering
    \scriptsize
    \begin{tabular*}{\textheight}{@{\extracolsep\fill}lcccccc}
        \toprule%
        & \multicolumn{2}{c}{\textbf{SGD (T=100)}} & \multicolumn{2}{c}{\textbf{RMSprop  (T=100)}} & \multicolumn{2}{c}{\textbf{Adamax  (T=100)}} \\\cmidrule{2-3}\cmidrule{4-5}\cmidrule{6-7}%
        Riemannian Manifold & Length & Runtime & Length & Runtime & Length & Runtime \\
        \midrule
        $\mathbb{S}^{2}$ & $\pmb{205.94}$ & $2.1838 \pm 0.0006$ & $224.29$ & $1.6724 \pm 0.0014$ & $227.18$ & $\pmb{1.6636} \pm \pmb{ 0.0019 }$ \\ 
        $\mathbb{S}^{3}$ & $\pmb{151.07}$ & $1.6920 \pm 0.0011$ & $183.56$ & $\pmb{1.6814} \pm \pmb{ 0.0011 }$ & $188.34$ & $1.7186 \pm 0.0010$ \\ 
        $\mathbb{S}^{5}$ & $194.97$ & $\pmb{1.7908} \pm \pmb{ 0.0011 }$ & $\pmb{179.87}$ & $1.8230 \pm 0.0010$ & $193.20$ & $1.8785 \pm 0.0017$ \\ 
        $\mathbb{S}^{10}$ & $203.93$ & $\pmb{2.8978} \pm \pmb{ 0.0014 }$ & $\pmb{95.59}$ & $2.9154 \pm 0.0012$ & $118.48$ & $2.9460 \pm 0.0017$ \\ 
        $\mathbb{S}^{20}$ & $143.37$ & $\pmb{5.4790} \pm \pmb{ 0.0013 }$ & $\pmb{48.80}$ & $5.5168 \pm 0.0012$ & $69.48$ & $5.5421 \pm 0.0013$ \\ 
        $\mathbb{S}^{50}$ & $64.22$ & $23.7887 \pm 0.0018$ & $\pmb{19.21}$ & $\pmb{22.8096} \pm \pmb{ 0.3671 }$ & $28.00$ & $23.9063 \pm 0.0006$ \\ 
        $\mathbb{S}^{100}$ & $32.36$ & $97.5666 \pm 4.6575$ & $\pmb{10.58}$ & $\pmb{84.6670} \pm \pmb{ 0.0665 }$ & $13.50$ & $84.8641 \pm 0.0702$ \\ 
        \hline
        $\mathrm{E}\left( 2 \right)$ & $137.83$ & $\pmb{1.6446} \pm \pmb{ 0.0015 }$ & $129.43$ & $1.6579 \pm 0.0016$ & $\pmb{124.54}$ & $1.6696 \pm 0.0009$ \\ 
        $\mathrm{E}\left( 3 \right)$ & $\pmb{117.87}$ & $\pmb{1.6822} \pm \pmb{ 0.0013 }$ & $120.71$ & $1.7087 \pm 0.0008$ & $120.92$ & $1.7110 \pm 0.0009$ \\ 
        $\mathrm{E}\left( 5 \right)$ & $\pmb{111.72}$ & $\pmb{1.8051} \pm \pmb{ 0.0012 }$ & $116.89$ & $1.8220 \pm 0.0014$ & $125.84$ & $1.8564 \pm 0.0018$ \\ 
        $\mathrm{E}\left( 10 \right)$ & $108.97$ & $2.9128 \pm 0.0011$ & $\pmb{68.20}$ & $\pmb{2.9124} \pm \pmb{ 0.0013 }$ & $75.54$ & $2.9482 \pm 0.0018$ \\ 
        $\mathrm{E}\left( 20 \right)$ & $82.70$ & $\pmb{5.5250} \pm \pmb{ 0.0013 }$ & $\pmb{37.48}$ & $5.5792 \pm 0.0009$ & $44.89$ & $5.5879 \pm 0.0011$ \\ 
        $\mathrm{E}\left( 50 \right)$ & $41.87$ & $\pmb{23.5703} \pm \pmb{ 0.0012 }$ & $\pmb{14.97}$ & $23.6562 \pm 0.0020$ & $18.62$ & $23.7257 \pm 0.0012$ \\ 
        $\mathrm{E}\left( 100 \right)$ & $22.31$ & $\pmb{84.1223} \pm \pmb{ 0.0829 }$ & $\pmb{7.95}$ & $85.0741 \pm 0.5396$ & $9.21$ & $84.4970 \pm 0.1002$ \\ 
        \hline
        $\mathbb{T}^{2}$ & $-$ & $\pmb{0.0219} \pm \pmb{ 0.0001 }$ & $1250.43$ & $1.6474 \pm 0.0014$ & $\pmb{1246.28}$ & $1.6429 \pm 0.0018$ \\ 
        \hline
        $\mathbb{H}^{2}$ & $\pmb{174.60}$ & $\pmb{0.0225} \pm \pmb{ 0.0001 }$ & $174.62$ & $1.6579 \pm 0.0013$ & $174.61$ & $1.6705 \pm 0.0014$ \\ 
        \hline
        Paraboloid & $-/-$ & $-$ & $112.73$ & $1.6066 \pm 0.0009$ & $\pmb{104.19}$ & $1.6068 \pm 0.0016$ \\ 
        Hyperbolic Paraboloid & $-/-$ & $-$ & $114.56$ & $1.6121 \pm 0.0010$ & $\pmb{114.40}$ & $1.6126 \pm 0.0014$ \\ 
        \hline
        Gaussian Distribution & $243.03$ & $0.2013 \pm 0.0008$ & $156.82$ & $\pmb{0.1893} \pm \pmb{ 0.0014 }$ & $\pmb{154.78}$ & $0.2141 \pm 0.0153$ \\ 
        Fr\'echet Distribution & $60189340.00$ & $\pmb{0.0039} \pm \pmb{ 0.0001 }$ & $32.89$ & $0.1971 \pm 0.0081$ & $\pmb{31.51}$ & $0.1932 \pm 0.0004$ \\ 
        Cauchy Distribution & $77.17$ & $\pmb{0.2181} \pm \pmb{ 0.0005 }$ & $67.21$ & $0.2299 \pm 0.0014$ & $\pmb{64.97}$ & $0.2357 \pm 0.0035$ \\ 
        Pareto Distribution & $25.48$ & $\pmb{0.1966} \pm \pmb{ 0.0009 }$ & $24.93$ & $0.1967 \pm 0.0073$ & $\pmb{24.31}$ & $0.2019 \pm 0.0001$ \\ 
        \hline
        VAE MNIST & $5687.01$ & $\pmb{0.3928} \pm \pmb{ 0.0030 }$ & $897.52$ & $39.1152 \pm 0.0014$ & $\pmb{894.32}$ & $39.1217 \pm 0.0033$ \\ 
        VAE CelebA & $25601932831752192.00$ & $\pmb{5.1663} \pm \pmb{ 0.0007 }$ & $9483.42$ & $857.7762 \pm 0.0208$ & $\pmb{9465.98}$ & $862.3471 \pm 0.0082$ \\ 
        \hline
    \end{tabular*}
    \caption{The table shows the sum of squared lengths of the estimated Fr\'echet mean on a GPU. When the computational time was longer than $24$ hours, the value is set to $-$.}
    \label{tab:riemmannian_comparison_table_extra1}
    \vspace{-2.5em}
\end{sidewaystable}

\begin{sidewaystable}
    \centering
    \scriptsize
    \begin{tabular*}{\textheight}{@{\extracolsep\fill}lcc}
        \toprule%
        & \multicolumn{2}{c}{\textbf{Adagrad (T=100)}} \\\cmidrule{2-3}%
        Riemannian Manifold & Length & Runtime \\
        \midrule
        \textbf{Manifold} & Length & Runtime \\
        $\mathbb{S}^{2}$ & $\pmb{226.54}$ & $\pmb{1.6704} \pm \pmb{ 0.0012 }$ \\ 
        $\mathbb{S}^{3}$ & $\pmb{184.13}$ & $\pmb{1.6923} \pm \pmb{ 0.0006 }$ \\ 
        $\mathbb{S}^{5}$ & $\pmb{178.95}$ & $\pmb{1.8459} \pm \pmb{ 0.0005 }$ \\ 
        $\mathbb{S}^{10}$ & $\pmb{106.87}$ & $\pmb{2.9221} \pm \pmb{ 0.0014 }$ \\ 
        $\mathbb{S}^{20}$ & $\pmb{64.50}$ & $\pmb{5.5231} \pm \pmb{ 0.0024 }$ \\ 
        $\mathbb{S}^{50}$ & $\pmb{26.57}$ & $\pmb{23.8409} \pm \pmb{ 0.0004 }$ \\ 
        $\mathbb{S}^{100}$ & $\pmb{12.87}$ & $\pmb{104.3970} \pm \pmb{ 4.9584 }$ \\ 
        \hline
        $\mathrm{E}\left( 2 \right)$ & $\pmb{124.39}$ & $\pmb{1.6623} \pm \pmb{ 0.0012 }$ \\ 
        $\mathrm{E}\left( 3 \right)$ & $\pmb{120.60}$ & $\pmb{1.6860} \pm \pmb{ 0.0012 }$ \\ 
        $\mathrm{E}\left( 5 \right)$ & $\pmb{120.13}$ & $\pmb{1.8351} \pm \pmb{ 0.0011 }$ \\ 
        $\mathrm{E}\left( 10 \right)$ & $\pmb{69.30}$ & $\pmb{2.9252} \pm \pmb{ 0.0013 }$ \\ 
        $\mathrm{E}\left( 20 \right)$ & $\pmb{39.57}$ & $\pmb{5.5755} \pm \pmb{ 0.0011 }$ \\ 
        $\mathrm{E}\left( 50 \right)$ & $\pmb{16.46}$ & $\pmb{23.6970} \pm \pmb{ 0.0009 }$ \\ 
        $\mathrm{E}\left( 100 \right)$ & $\pmb{7.99}$ & $\pmb{85.1028} \pm \pmb{ 0.4041 }$ \\ 
        \hline
        $\mathbb{T}^{2}$ & $\pmb{1248.16}$ & $\pmb{1.6418} \pm \pmb{ 0.0015 }$ \\ 
        \hline
        $\mathbb{H}^{2}$ & $\pmb{174.61}$ & $\pmb{1.6710} \pm \pmb{ 0.0017 }$ \\ 
        \hline
        Paraboloid & $\pmb{104.02}$ & $\pmb{1.6051} \pm \pmb{ 0.0013 }$ \\ 
        Hyperbolic Paraboloid & $\pmb{114.40}$ & $\pmb{1.6111} \pm \pmb{ 0.0022 }$ \\ 
        \hline
        Gaussian Distribution & $\pmb{155.03}$ & $\pmb{0.2229} \pm \pmb{ 0.0007 }$ \\ 
        Fr\'echet Distribution & $\pmb{31.44}$ & $\pmb{0.2235} \pm \pmb{ 0.0005 }$ \\ 
        Cauchy Distribution & $\pmb{64.36}$ & $\pmb{0.2480} \pm \pmb{ 0.0011 }$ \\ 
        Pareto Distribution & $\pmb{24.27}$ & $\pmb{0.1998} \pm \pmb{ 0.0003 }$ \\ 
        \hline
        VAE MNIST & $\pmb{884.30}$ & $\pmb{39.0508} \pm \pmb{ 0.0012 }$ \\ 
        VAE CelebA & $\pmb{9452.99}$ & $\pmb{835.0145} \pm \pmb{ 0.1027 }$ \\ 
        \hline
    \end{tabular*}
    \caption{The table shows the sum of squared lengths of the estimated Fr\'echet mean on a GPU. When the computational time was longer than $24$ hours, the value is set to $-$.}
    \label{tab:riemmannian_comparison_table_extra2}
    \vspace{-2.5em}
\end{sidewaystable}

\paragraph{Additional Runtime Estimates Finsler Manifolds}

\begin{sidewaystable}
    \centering
    \scriptsize
    \begin{tabular*}{\textheight}{@{\extracolsep\fill}lcccccc}
        \toprule%
        & \multicolumn{2}{c}{\textbf{SGD (T=100)}} & \multicolumn{2}{c}{\textbf{RMSprop  (T=100)}} & \multicolumn{2}{c}{\textbf{Adamax  (T=100)}} \\\cmidrule{2-3}\cmidrule{4-5}\cmidrule{6-7}%
        Finsler Manifold & Length & Runtime & Length & Runtime & Length & Runtime \\
        \midrule
        $\mathbb{S}^{2}$ & $86.83$ & $\pmb{2.3036} \pm \pmb{ 0.0018 }$ & $83.81$ & $2.3118 \pm 0.0014$ & $\pmb{83.68}$ & $2.3185 \pm 0.0015$ \\ 
        $\mathbb{S}^{3}$ & $79.70$ & $2.3912 \pm 0.0003$ & $79.16$ & $\pmb{2.3890} \pm \pmb{ 0.0012 }$ & $\pmb{78.49}$ & $2.4210 \pm 0.0014$ \\ 
        $\mathbb{S}^{5}$ & $96.69$ & $\pmb{3.1103} \pm \pmb{ 0.0013 }$ & $84.52$ & $3.1244 \pm 0.0010$ & $\pmb{84.43}$ & $3.1469 \pm 0.0014$ \\ 
        $\mathbb{S}^{10}$ & $76.13$ & $\pmb{5.4475} \pm \pmb{ 0.0019 }$ & $\pmb{50.06}$ & $5.4784 \pm 0.0018$ & $52.81$ & $5.4966 \pm 0.0032$ \\ 
        $\mathbb{S}^{20}$ & $52.70$ & $\pmb{12.3483} \pm \pmb{ 0.0015 }$ & $\pmb{26.17}$ & $12.3857 \pm 0.0019$ & $30.54$ & $12.4081 \pm 0.0013$ \\ 
        $\mathbb{S}^{50}$ & $23.45$ & $\pmb{54.9592} \pm \pmb{ 0.0016 }$ & $\pmb{10.91}$ & $55.0525 \pm 0.0108$ & $12.65$ & $55.1001 \pm 0.0013$ \\ 
        $\mathbb{S}^{100}$ & $11.78$ & $\pmb{179.9665} \pm \pmb{ 0.1075 }$ & $\pmb{6.16}$ & $180.1367 \pm 0.0069$ & $6.36$ & $180.2978 \pm 0.2581$ \\ 
        \hline
        $\mathrm{E}\left( 2 \right)$ & $52.09$ & $\pmb{2.3014} \pm \pmb{ 0.0023 }$ & $49.72$ & $2.3081 \pm 0.0005$ & $\pmb{46.92}$ & $2.3178 \pm 0.0034$ \\ 
        $\mathrm{E}\left( 3 \right)$ & $51.24$ & $\pmb{2.3818} \pm \pmb{ 0.0010 }$ & $50.88$ & $2.3839 \pm 0.0015$ & $\pmb{49.31}$ & $2.4363 \pm 0.0017$ \\ 
        $\mathrm{E}\left( 5 \right)$ & $56.06$ & $\pmb{3.0950} \pm \pmb{ 0.0020 }$ & $\pmb{53.31}$ & $3.0960 \pm 0.0021$ & $54.20$ & $3.1361 \pm 0.0014$ \\ 
        $\mathrm{E}\left( 10 \right)$ & $42.15$ & $\pmb{5.0520} \pm \pmb{ 0.0022 }$ & $\pmb{32.50}$ & $5.0801 \pm 0.0020$ & $32.98$ & $5.0919 \pm 0.0018$ \\ 
        $\mathrm{E}\left( 20 \right)$ & $30.65$ & $\pmb{12.4187} \pm \pmb{ 0.0018 }$ & $\pmb{18.66}$ & $12.4817 \pm 0.0010$ & $19.87$ & $12.4936 \pm 0.0014$ \\ 
        $\mathrm{E}\left( 50 \right)$ & $15.20$ & $\pmb{54.6004} \pm \pmb{ 0.0015 }$ & $\pmb{8.21}$ & $54.6795 \pm 0.0113$ & $8.76$ & $54.7518 \pm 0.0011$ \\ 
        $\mathrm{E}\left( 100 \right)$ & $8.09$ & $\pmb{176.7958} \pm \pmb{ 0.0389 }$ & $4.99$ & $176.9338 \pm 0.0083$ & $\pmb{4.84}$ & $177.0561 \pm 0.0270$ \\ 
        \hline
        $\mathbb{T}^{2}$ & $-/-$ & $-$ & $467.77$ & $2.4788 \pm 0.0014$ & $\pmb{467.67}$ & $2.4945 \pm 0.0010$ \\ 
        \hline
        $\mathbb{H}^{2}$ & $74.76$ & $\pmb{0.0296} \pm \pmb{ 0.0014 }$ & $74.79$ & $2.5035 \pm 0.0019$ & $\pmb{74.75}$ & $2.5172 \pm 0.0019$ \\ 
        \hline
        Paraboloid & $inf$ & $\pmb{0.0186} \pm \pmb{ 0.0001 }$ & $\pmb{61.26}$ & $2.2251 \pm 0.0007$ & $65.78$ & $2.2251 \pm 0.0008$ \\ 
        Hyperbolic Paraboloid & $-$ & $\pmb{0.0206} \pm \pmb{ 0.0001 }$ & $\pmb{52.62}$ & $2.2064 \pm 0.0004$ & $62.22$ & $2.2156 \pm 0.0007$ \\ 
        \hline
        Gaussian Distribution & $76.16$ & $0.7477 \pm 0.0019$ & $68.64$ & $\pmb{0.7352} \pm \pmb{ 0.0003 }$ & $\pmb{64.75}$ & $0.7513 \pm 0.0016$ \\ 
        Fr\'echet Distribution & $35.58$ & $\pmb{0.6141} \pm \pmb{ 0.0031 }$ & $\pmb{11.74}$ & $0.6541 \pm 0.0019$ & $15.06$ & $0.6393 \pm 0.0007$ \\ 
        Cauchy Distribution & $28.10$ & $0.7585 \pm 0.0044$ & $26.92$ & $\pmb{0.7365} \pm \pmb{ 0.0026 }$ & $\pmb{25.73}$ & $0.7628 \pm 0.0020$ \\ 
        Pareto Distribution & $27.73$ & $0.6621 \pm 0.0011$ & $\pmb{11.36}$ & $\pmb{0.6529} \pm \pmb{ 0.0072 }$ & $15.52$ & $0.6633 \pm 0.0035$ \\ 
        \hline
        VAE MNIST & $-/-$ & $-$ & $316.46$ & $463.9254 \pm 0.0712$ & $\pmb{313.52}$ & $462.4614 \pm 0.2502$ \\ 
        VAE CelebA & $-$ & $\pmb{9.6821} \pm \pmb{ 0.0034 }$ & $3213.31$ & $1602.4119 \pm 0.1665$ & $\pmb{3202.72}$ & $1597.8483 \pm 0.3720$ \\ 
        \hline
    \end{tabular*}
    \caption{The table shows the sum of squared lengths of the estimated Fr\'echet mean on a GPU. When the computational time was longer than $24$ hours, the value is set to $-$.}
    \label{tab:finsler_comparison_table_extra1}
    \vspace{-2.5em}
\end{sidewaystable}

\begin{sidewaystable}
    \centering
    \scriptsize
    \begin{tabular*}{\textheight}{@{\extracolsep\fill}lcc}
        \toprule%
        & \multicolumn{2}{c}{\textbf{Adagrad (T=100)}} \\\cmidrule{2-3}%
        Finsler Manifold & Length & Runtime \\
        \midrule
        $\mathbb{S}^{2}$ & $\pmb{83.67}$ & $\pmb{2.3119} \pm \pmb{ 0.0006 }$ \\ 
        $\mathbb{S}^{3}$ & $\pmb{78.23}$ & $\pmb{2.4114} \pm \pmb{ 0.0014 }$ \\ 
        $\mathbb{S}^{5}$ & $\pmb{81.83}$ & $\pmb{3.1352} \pm \pmb{ 0.0013 }$ \\ 
        $\mathbb{S}^{10}$ & $\pmb{49.19}$ & $\pmb{5.4872} \pm \pmb{ 0.0015 }$ \\ 
        $\mathbb{S}^{20}$ & $\pmb{28.47}$ & $\pmb{12.3839} \pm \pmb{ 0.0012 }$ \\ 
        $\mathbb{S}^{50}$ & $\pmb{12.06}$ & $\pmb{55.1034} \pm \pmb{ 0.0018 }$ \\ 
        $\mathbb{S}^{100}$ & $\pmb{6.14}$ & $\pmb{180.1266} \pm \pmb{ 0.0070 }$ \\ 
        \hline
        $\mathrm{E}\left( 2 \right)$ & $\pmb{46.76}$ & $\pmb{2.3096} \pm \pmb{ 0.0007 }$ \\ 
        $\mathrm{E}\left( 3 \right)$ & $\pmb{49.15}$ & $\pmb{2.3999} \pm \pmb{ 0.0022 }$ \\ 
        $\mathrm{E}\left( 5 \right)$ & $\pmb{53.28}$ & $\pmb{3.1178} \pm \pmb{ 0.0022 }$ \\ 
        $\mathrm{E}\left( 10 \right)$ & $\pmb{31.17}$ & $\pmb{5.0837} \pm \pmb{ 0.0031 }$ \\ 
        $\mathrm{E}\left( 20 \right)$ & $\pmb{17.89}$ & $\pmb{12.4912} \pm \pmb{ 0.0016 }$ \\ 
        $\mathrm{E}\left( 50 \right)$ & $\pmb{7.90}$ & $\pmb{54.7486} \pm \pmb{ 0.0008 }$ \\ 
        $\mathrm{E}\left( 100 \right)$ & $\pmb{5.22}$ & $\pmb{176.9808} \pm \pmb{ 0.0056 }$ \\ 
        \hline
        $\mathbb{T}^{2}$ & $\pmb{467.59}$ & $\pmb{2.4961} \pm \pmb{ 0.0019 }$ \\ 
        \hline
        $\mathbb{H}^{2}$ & $\pmb{74.73}$ & $\pmb{2.5105} \pm \pmb{ 0.0021 }$ \\ 
        \hline
        Paraboloid & $\pmb{65.71}$ & $\pmb{2.2253} \pm \pmb{ 0.0005 }$ \\ 
        Hyperbolic Paraboloid & $\pmb{61.45}$ & $\pmb{2.2111} \pm \pmb{ 0.0006 }$ \\ 
        \hline
        Gaussian Distribution & $\pmb{65.12}$ & $\pmb{0.7496} \pm \pmb{ 0.0003 }$ \\ 
        Fr\'echet Distribution & $\pmb{14.97}$ & $\pmb{0.6548} \pm \pmb{ 0.0014 }$ \\ 
        Cauchy Distribution & $\pmb{25.70}$ & $\pmb{0.7510} \pm \pmb{ 0.0018 }$ \\ 
        Pareto Distribution & $\pmb{16.07}$ & $\pmb{0.6515} \pm \pmb{ 0.0016 }$ \\ 
        \hline
        VAE MNIST & $\pmb{310.68}$ & $\pmb{462.2765} \pm \pmb{ 0.0565 }$ \\ 
        VAE CelebA & $\pmb{3194.40}$ & $\pmb{1644.1156} \pm \pmb{ 0.0458 }$ \\ 
        \hline
    \end{tabular*}
    \caption{The table shows the sum of squared lengths of the estimated Fr\'echet mean on a GPU. When the computational time was longer than $24$ hours, the value is set to $-$.}
    \label{tab:finsler_comparison_table_extra2}
    \vspace{-2.5em}
\end{sidewaystable}